\documentclass{article}

\usepackage[main, final]{neurips_2026}

\usepackage[utf8]{inputenc} 
\usepackage[T1]{fontenc}    
\usepackage{hyperref}       
\hypersetup{
  colorlinks,
  linkcolor={red!50!black},
  citecolor={blue!50!black},
  urlcolor={blue!80!black}
}
\usepackage{url}            
\usepackage{booktabs}       
\usepackage{amsfonts}       
\usepackage{nicefrac}       
\usepackage{microtype}      
\usepackage{xcolor}         
\usepackage{macros}

\title{SGD in Multiclass Logistic Regression: \\ Sequential Learning and Scaling Laws}

\author{%
  Konstantinos Christopher Tsiolis \\
  University of Toronto, Vector Institute \\
  \texttt{kc.tsiolis@mail.utoronto.ca} \\
  \And
  Denny Wu \\
  New York University, Flatiron Institute \\
  \texttt{dennywu@nyu.edu}
  \And
  Christos Thrampoulidis \\
  University of British Columbia \\
  \texttt{cthrampo@ece.ubc.ca} \\
  \And
  Murat A. Erdogdu \\
  University of Toronto, Vector Institute, Numurho \\
  \texttt{erdogdu@cs.toronto.edu} \\
}

\begin{document}

\maketitle
\begin{abstract}
  We study the training dynamics of multiclass logistic regression on high-dimensional mixture models with a large number of classes and establish precise scaling laws governing the cross-entropy risk under gradient-based optimization. We show that learning proceeds sequentially across classes, from most to least frequent. When the class priors follow a power law distribution, the risk dynamics decompose into three phases: an initial plateau until the first class is learned, a power-law decay regime during which sequential learning occurs, and a final convergence regime. 

We then analyze how model capacity interacts with optimization under a fixed compute budget. When the effective dimension is restricted via projection onto leading principal components, the risk decomposes into a capacity term (a power law in the retained dimension) and an optimization term (a power law in training time). Optimizing this tradeoff yields a compute-optimal scaling law for logistic regression, with explicit prescriptions for model size and training time as functions of compute. These results extend theoretical scaling laws from linear regression to multiclass classification, while connecting to empirical scaling laws observed in large-scale neural networks.
\end{abstract}

\section{Introduction}\label{section:intro}

Large-scale neural networks exhibit remarkably predictable loss curves: as the amount of compute, data, and model size increases, the test loss often decreases as a power law over several orders of magnitude \citep{hestness2017deep,kaplan2020scaling,hoffmann2022training}. These empirical neural scaling laws have become a central tool for forecasting the outcome of large training runs and for allocating a fixed compute budget between model size and training time. A growing line of work has begun to develop a theoretical understanding of scaling law behaviors in stylized models. Existing results typically consider gradient-based learning of either sketched linear regression models, where the effective model size is controlled by a random feature map \citep{maloney2022solvable,bordelon2024dynamical,paquette2024phases,lin2024scaling}, or shallow neural networks with power-law features (i.e., additive model) \citep{nam2024exactly,ren2026emergence,ben2026learning,defilippis2026optimal}. For such model classes, the squared loss admits a decomposition in the form of: 
$
    \mathrm{Risk}(m,t) \asymp m^{-\gamma} + t^{-\beta},
$
where $m$ denotes an effective model size and $t$ denotes sample size or optimization time. Together, these results provide a mathematical picture for scaling laws in regression 
problems.

However, there is a noticeable mismatch between existing scaling law theory focusing on the squared loss and modern language-model pretraining, which is typically formulated as next-token prediction with the cross-entropy loss. On the other hand, next-token prediction is not a binary or fixed-class problem: the number of possible tokens is large, and token frequencies are highly non-uniform. This suggests that a natural theoretical abstraction should be a multiclass classification problem with a large number of classes and a heavy-tailed class prior. We therefore ask the following question:

\begin{center}
\emph{
What is a tractable model to derive scaling laws
for gradient-based optimization under cross-entropy loss, in a
multiclass classification problem with divergingly many classes?}
\end{center}

\subsection{Our Contributions}

\begin{wrapfigure}{r}{0.46\textwidth}  
\vspace{-5mm} 
\centering
\centering 
{\includegraphics[height=0.8\linewidth]{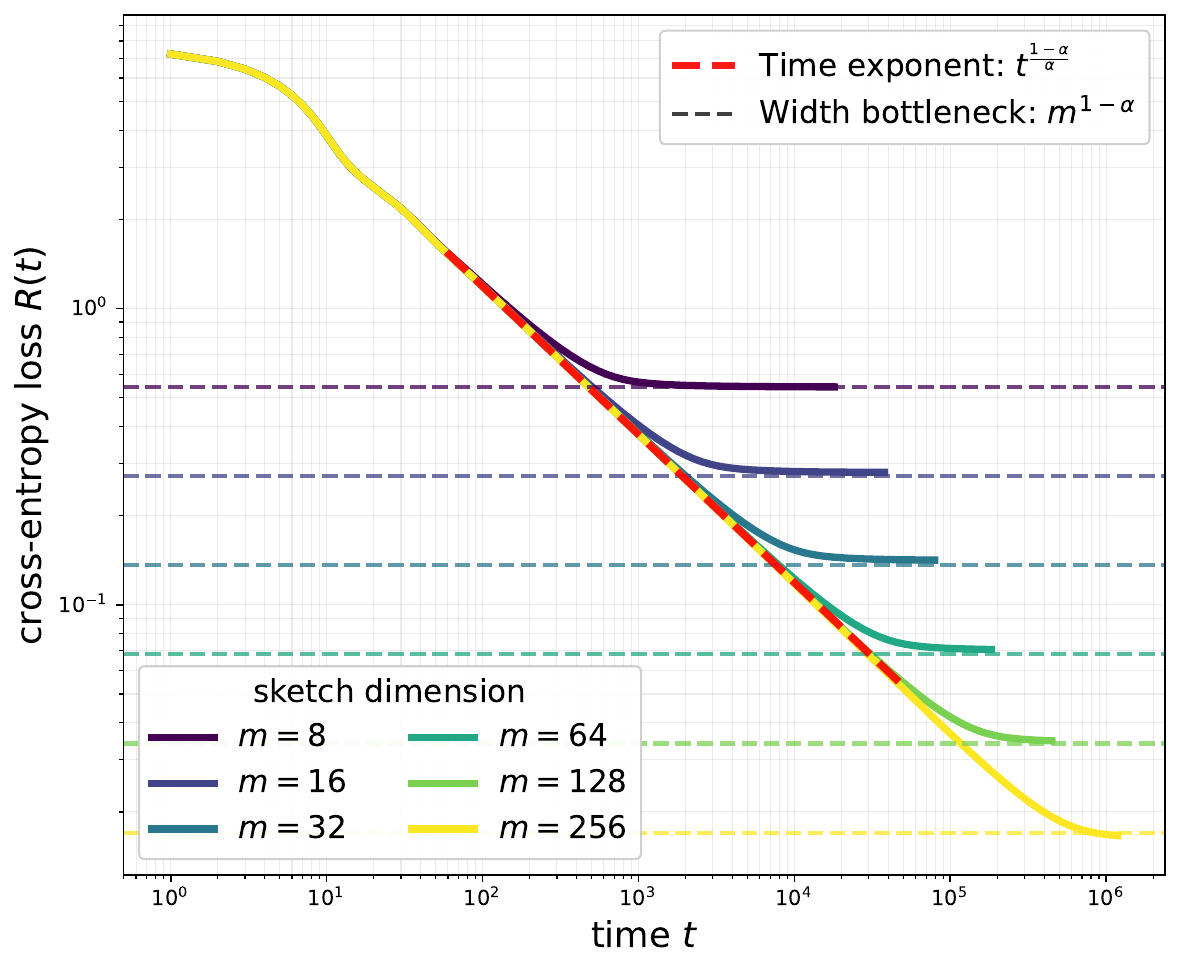}} 
\vspace{-1mm}
\vspace{-1mm} 
\caption{\small 
Gradient flow trajectory and cross-entropy scaling law under PCA sketch. We solve the population ODE \eqref{eq:full_theta_dynamics} using Euler discretization, where expectations are approximated via Monte Carlo samples. We set $K=2048, \sigma=0.1, \alpha=2$ and vary the sketch dimension $m$. The time \& width scaling exponents are predicted by Theorem \ref{thm:noisy_sketch}.  
}  
\label{fig:sketched} 
\end{wrapfigure}

We study this question in a tractable high-dimensional classification model: $K$-class logistic regression on a mixture distribution. The label $y\in[K]$ is sampled from a power-law prior
$\pi_i\asymp i^{-\alpha}$ with $\alpha \geq 0$, and, conditional on $y=i$,
\begin{equation*}
    \xbs = \mubs_i + \xibs_i, \quad \inner{\mubs_i, \mubs_j} = \1\{i=j\},
\end{equation*}
where $\mubs_i$ is the class mean\footnote{The orthgonality assumption can be relaxed to allow for small perturbations.} and $\xibs_i$ is sub-Gaussian noise. The model is trained either with gradient flow or (stochastic) gradient descent on the cross-entropy loss. Though this setting is much simpler than language modelling, it isolates three ingredients that are central to classification scaling laws: a large label space, heterogeneous class frequencies, and the cross-entropy objective.

Our main finding is that gradient-based training learns the classes sequentially in order of decreasing frequency. Aggregating these class-wise learning curves yields a scaling law for the cross-entropy loss trajectory. Under an appropriate notion of dimension reduction, we find the compute-optimal trade-off between the retained dimension and the number of gradient queries. Specifically, we establish the following: 
\begin{enumerate}[leftmargin=*]
    \item \textbf{Sequential learning and loss scaling.}
    In Section \ref{section:pop_gf}, we prove that population gradient flow learns class $i$ at time
    $t_i\asymp (\log K)/\pi_i$. For power-law priors $\pi_i\asymp i^{-\alpha}$ with $\alpha > 1$, this yields a scaling law in the cross-entropy risk:
    \[
        \Rcal(t) \asymp
        \begin{cases}
        \log K, & 0 \le t \lesssim \log K,\\[1mm]
        (\log K)^{2-1/\alpha} \cdot t^{-(1-1/\alpha)},
            & \log K \lesssim t \lesssim K^\alpha \log K,\\[1mm]
        K/t, & K^\alpha \log K \lesssim t.
        \end{cases}
    \]
    Under $\sigma$-sub-Gaussian intra-class noise, this scaling persists up to a finite horizon $T\asymp 1/\sigma$.

    \item \textbf{Compute-optimal scaling laws.}
    In Section \ref{section:compute_optimal_sl}, we introduce a PCA-type projection onto the span of the $m$ most frequent class means to control the effective model capacity. This differs from the random projection commonly used in the regression setting to introduce a capacity bottleneck \citep{bordelon2024dynamical,paquette2024phases,lin2024scaling}, since a Gaussian sketch may not create such a bottleneck in the classification loss. Under the PCA projection, we prove the following power-law risk scaling when $\alpha > 1$ (see Figure~\ref{fig:sketched}):
    \[
        \Rcal_m(t)
        \asymp
        m^{-(\alpha-1)}\log K
        +
        (\log K)^{2-1/\alpha}t^{-(1-1/\alpha)}.
    \]
    Optimizing this tradeoff under compute $C\asymp mKt$ yields
  the compute-optimal rate 
    $$
        \Rcal_{m_\star}(t_\star)
        \asymp
        (\log K)^{2\alpha/(\alpha+1)}
        (K/C)^{(\alpha-1)/(\alpha+1)}.
    $$

    \item \textbf{Online SGD dynamics.}
    Finally, in Section \ref{section:online_sgd}, we study the time discretization and show that for a suitable learning rate, online stochastic gradient descent (SGD) tracks the population gradient flow trajectory over the relevant time horizon. Consequently, the SGD dynamics exhibit similar scaling behavior to gradient flow with effective optimization time $t = \eta n$.
\end{enumerate}


\subsection{Related Work}

\paragraph{Neural scaling laws.}
Empirical studies have shown that neural network losses often obey predictable power laws as a function of model size, dataset size, and compute \citep{hestness2017deep,kaplan2020scaling,hoffmann2022training}. A growing theoretical literature has derived such scaling laws in tractable regression settings, including sketched linear models (random features regression) \citep{maloney2022solvable,bordelon2024dynamical,paquette2024phases,lin2024scaling,lin2025improved,ferbach2025dimension}, shallow neural networks with high-information-exponent activation functions \citep{ren2026emergence,ben2026learning,defilippis2026optimal}, as well as linear bigram \citep{kunstner2025scaling} and associative memory models \citep{cabannes2024scaling,li2026muon,kim2026sharp}. The phenomenology of our problem setting can be compared to the additive model hypothesis  \citep{michaud2023quantization,nam2024exactly}, where the ``global'' smooth power-law risk curve is a result of the sequential learning of a large number of ``atomic'' tasks. We extend this perspective from squared-loss regression to cross-entropy classification, where the individual components are classes with heterogeneous frequencies.


\paragraph{Learning dynamics for classification.}
For multiclass Gaussian mixture data, \cite{thrampoulidis2020theoretical} derive sharp asymptotics for the test error of linear classifiers, while \cite{mignacco2020dynamical} study SGD dynamics for binary Gaussian mixture classification. \cite{ben2025spectral,ben2025local} consider SGD for multiclass logistic regression on Gaussian mixtures, characterizing the emergence of spectral outliers in the Hessian through low-dimensional effective dynamics. Recent works further derive effective dynamics in generalized linear and mixture models \citep{collins2024hitting,collins2025exact}. Another related line studies logistic regression through implicit bias: on separable data, gradient descent converges in direction to max-margin solutions, with extensions to non-separable and multiclass settings \citep{soudry2018implicit,ji2018risk,nacson2019stochastic,wu2023implicit,ravi2024implicit,schliserman2025multiclass,fan2026implicit}. Closest in spirit to our sequential-learning result, \cite{zhao2025loss} show that in an unconstrained feature model for imbalanced classification, standard training learns majority before minority distinctions. In contrast, we analyze multiclass logistic regression with a diverging number of classes and power-law class frequencies, and derive explicit scaling laws for population gradient flow and online SGD.

\subsection{Preliminaries}\label{section:setup}

\paragraph{Notation.} For $k \in \N$, we use $[k]$ to denote the set $\{1,\dots,k\}$. The relations $\lesssim$ and $\gtrsim$ denote upper and lower bounds up to positive constant factors, respectively. We write $a \asymp b$ when $a \lesssim b$ and $a \gtrsim b$. The notations $\inner{\cdot,\cdot}$ and $||\cdot||$ refer respectively to the Euclidean inner product and norm for vectors in $\R^d$ in the absence of a subscript, while for matrices the notation $||\cdot||$ refers to the operator norm.

\paragraph{Problem setup.} We consider the problem of supervised classification on a $K$-component mixture model with orthonormal means and sub-Gaussian noise:
\begin{equation}\label{eq:target}
    \begin{gathered}
    y \sim \mathrm{Categorical}\{\pi_1,\dots,\pi_K\} \quad \text{where}\quad\pi_i = i^{-\alpha}/Z_{K,\alpha} \quad \text{with} \quad Z_{K,\alpha} = \sum_{k=1}^K k^{-\alpha},\quad \alpha > 0,\\
    \E_{\xbs|y=i}[\xbs] = \mubs_i, \quad \E_{\xbs|y=i}[e^{\inner{\ubs, \xbs - \mubs_i}}] \leq e^{\sigma^2 ||u||^2/2} \quad \text{with} \quad \inner{\mubs_i, \mubs_j} = \1\{i=j\}, \quad \sigma^2 \geq 0. \\
    \end{gathered}
\end{equation} 
The orthogonal means assumption can be relaxed to allow for small perturbations. See Appendix \ref{appendix:non_orthogonal} for details. We normalize the class means to isolate the effect of class imbalance, since, in the noiseless case, the class-wise dynamics scale with $\pi_j ||\mubs_j||^2$. 

We are particularly interested in the high-dimensional regime where the input dimension as well as the number of classes are both extremely large, i.e., $1 \ll K \leq d$. While our sequential learning results hold for an arbitrary choice of prior $(\pi_1,\dots,\pi_K)$, we focus on the \emph{power-law prior} where $\pi_i \asymp i^{-\alpha}$ for all $i \in [K]$ and for some $\alpha > 0$.

We fit data from the mixture model via $K$-class logistic regression, which takes input $\xbs \in \R^d$ and predicts an associated mixture component (label) $y \in [K]$. The weights $\Wbs = [\wbs_1,\dots,\wbs_K] \in \R^{d \times K}$ of the model induce a probability distribution over $[K]$ and are trained to minimize the cross-entropy loss
\begin{equation}\label{eq:cross_entropy_loss}
    \Lcal(\Wbs; \xbs,y) = -\sum_{i=1}^K  \1\{y=i\}\log p_i(\xbs) 
    \quad\text{ where }\quad
    p_i(\xbs) := \frac{\exp(\inner{\wbs_i,\xbs})}{\sum_{k=1}^K \exp(\inner{\wbs_k, \xbs})}.
\end{equation}
We are ultimately interested in the trajectory of the \emph{population risk} over the course of training of $\Wbs$, 
\begin{equation}\label{eq:risk_def}
    \Rcal(\Wbs) := \E_{\xbs,y}\left[\Lcal(\Wbs;\xbs,y)\right],
\end{equation}
under gradient flow and online SGD.

\section{Population Gradient Flow}\label{section:pop_gf}
As an important building block towards our later study of online SGD, we consider the learning dynamics under \emph{population gradient flow} (GF):
\begin{equation}\label{eq:pop_gf}
    \dot{\Wbs} = -\nabla_{\Wbs} \Rcal(\Wbs) \quad\text{ with }\quad \Wbs(0) = \zerobs_{d \times K}.
\end{equation}
We track several key quantities of interest, namely:
\begin{equation}
    \theta_{ij} := \inner{\wbs_i, \mubs_j}, \quad p_{i|j} = p_i(\mubs_j), \quad P_{i|j} = \E_{\xbs|y=j}[p_i(\xbs)], \quad i,j \in [K].
\end{equation}
Following the convention in \cite{benarous2021online}, we refer to the $\theta_{ij}$ as \emph{summary statistics}. A short calculation (see Appendix \ref{appendix:proofs_pop_grad_flow}) shows that the summary statistics satisfy the ODE
\begin{equation}\label{eq:full_theta_dynamics}
    \dot{\theta}_{ij} = \pi_j (\1\{i=j\} - P_{i|j}) - \sum_{k=1}^K \pi_k \inner{\E_{\xibs_k}[p_i(\mubs_k + \xibs_k)\xibs_k], \mubs_j}.
\end{equation}


\subsection{The Noiseless Case}\label{section:noiseless_case}

For a first result, we consider the noiseless setting $\sigma^2 = 0$, where the problem reduces to associating each of $K$ points with the correct label (i.e., associate $\mubs_1$ with the label 1, $\mubs_2$ with the label 2, and so on). This coincides with a special case of associative memory models studied in \cite{cabannes2024scaling,li2026muon}, though our focus here is on sharply characterizing the risk and the underlying summary statistics along the GF dynamics.

Here, the problem reduces to tracking the quantities $p_{i|i} := p_i(\mubs_i)$ for $i \in [K]$, which we show in Appendix \ref{appendix:noiseless_case} each satisfy the autonomous ODE
\begin{equation}\label{eq:idealized_p_ii_ode}
    \dot{p}_{i|i} = \frac{K}{K-1}\pi_i p_{i|i}(1-p_{i|i})^2 \quad\text{ with }\quad p_{i|i}(0) = \frac{1}{K}.
\end{equation}
Hence, the $p_{i|i}$ are increasing for all time, and the prior $\pi_i$ controls the speed of learning. Tracking the early-time dynamics of the above ODE yields the following lemma, confirming the intuition that more frequently observed classes are learned faster.

\begin{lemma}[Noiseless Sequential Learning]\label{lem:noiseless_sequential_learning}
    Given the target \eqref{eq:target}, suppose that $\sigma^2 = 0$. Fix $c \in (1/K,1)$ and define the hitting time $t_i(c) := \inf \{t: p_{i|i}(t) \geq c\}$ under population gradient flow \eqref{eq:pop_gf} with the cross-entropy loss. Then
    \begin{equation}
        t_i(c) = \frac{K-1}{K} \frac{1}{\pi_i} \left[\log \frac{c}{1-c} + \frac{1}{1-c} + \log(K-1) - \frac{K}{K-1}\right] \asymp \frac{\log K}{\pi_i}.
    \end{equation}
\end{lemma}

\begin{remark}
    Lemma \ref{lem:noiseless_sequential_learning} holds for \emph{any} choice of $\pi_1,\dots,\pi_K$, not just power law.
\end{remark}

The proof in Appendix \ref{appendix:noiseless_case} relies on the conservation law $\sum_{i=1}^K \theta_{ij}(t) = 0$ for all $j \in [K]$ and $t \geq 0$, which follows immediately from the sum-to-one constraint of the softmax and the initialization $\Wbs(0) = \zerobs$. This law holds for $\sigma^2 > 0$, but, in the noiseless case, we can take this even further to obtain $\theta_{ij}(t) = -\frac{1}{K-1}\theta_{jj}(t)$ for all $i \neq j$, leading directly to the decoupled dynamics \eqref{eq:idealized_p_ii_ode}. 

We write the population risk as a function of time, $\Rcal(t) = -\sum_{i=1}^K \pi_i \log p_{i|i}(t)$. For a given $t$, we decompose $\Rcal$ into a \emph{head} term covering the classes for which $p_{i|i} \geq 1/2$ and a \emph{tail} term for the rest, leveraging Lemma~\ref{lem:noiseless_sequential_learning}. We treat head classes as being in the late phase of their dynamics, where $\dot{p}_{i|i} \asymp \pi_i(1-p_{i|i})^2$ and tail classes as being in the early phase, where $\dot{p}_{i|i} \asymp \pi_i p_{i|i}$. When we have the power law assumption on $\pi_i$ with $\alpha > 1$, we obtain power-law scaling of the risk.
\begin{proposition}[Noiseless Scaling Law]\label{prop:noiseless_scaling_law}
    Given the target \eqref{eq:target}, suppose $\alpha > 0$ and $\sigma^2 = 0$. Then, for the population gradient flow on cross-entropy loss \eqref{eq:pop_gf}, there exist constants\footnote{Many choices of constants are possible. For any $K \geq 4$, a clean choice is $c_1 = 2^{-(\alpha+2)}$, $c_2 = 1$.} $c_1 \in (0,1)$, $c_2 > 0$ such that
    \begin{equation}\label{eq:idealized_risk_phases}
        \Rcal(t) \asymp
        \begin{cases}
            \log K, & 0 \leq t \leq t_1 \\
            f_{K,\alpha}(t), & t_1 \leq t \leq c_1 t_K \\
            K/t, &(1+c_2) t_K \leq t, \
        \end{cases}
    \end{equation}
    where
    \begin{equation}
        f_{K,\alpha}(t) = 
        \begin{cases}
            (\log K)\left(1-(t/t_K)^{1/\alpha-1}\right), & 0 < \alpha < 1 \\
            \log\left(K(\log K)^2/t\right), & \alpha = 1\\
            (\log K)^{2-1/\alpha} \cdot t^{-(1-1/\alpha)}, & \alpha > 1. 
        \end{cases}
    \end{equation}
    If $\alpha = 0$, then $t_1 = t_K$ and $\Rcal(t) \asymp \log K$ for $0 \leq t \leq c_1 t_1$ and $\Rcal(t) \asymp K/t$ for $(1+c_2) t_1 \leq t$.
\end{proposition}

The full proof of this result is in Appendix \ref{appendix:noiseless_case}. We identify three phases in the dynamics (see Figure~\ref{fig:unsketched-noiseless}): 
\begin{itemize}[leftmargin=0.15in]
    \item \textbf{The early phase}, where the risk behaves like $\Rcal(t) \asymp \log K$, is the interval during which all classes reside in the tail and have $p_{i|i} \asymp 1/K$.
    \item \textbf{The intermediate phase} is initiated by the recovery of the most frequent class (i.e., at $t_1$) and lasts until a constant fraction of the $K$ classes enter the head. There are three possible forms of the risk curve here, depending on the value of $\alpha$. If $\alpha > 1$, i.e., the prior probabilities remain summable as the number of classes $K$ diverges, the risk scales as a power law in time $t$. This matches the scaling law in the associative memory literature \cite{cabannes2024scaling,li2026muon}. On the other hand, there is no power law during this phase when $0 < \alpha \leq 1$.
    \item \textbf{The convergence phase} begins after all classes enter the head, and the risk follows the $K/t$ from the implicit bias literature (where we have made the dependence on the number of classes $K$ explicit) \cite{soudry2018implicit,nacson2019stochastic,ravi2024implicit}. This regime is separated from the intermediate phase by a transition window that we do not characterize exactly, during which the remaining constant fraction of classes exits the tail.
\end{itemize}

\begin{remark}
    The phenomenology of solving tasks sequentially while following power law scaling in the aggregate is also observed in gradient flow on additive models \cite{nam2024exactly,ren2026emergence,ben2026learning}. In the latter context, the target is a weighted sum of single-index models, where the weights have power law decay. Sharp phase transitions --- in the form of sudden drops and a ``staircase'' behaviour in the risk trajectory --- are observed in early time as the components with the largest weights are recovered, while the power law regime emerges later. By contrast, logistic regression does not exhibit these sharp transitions and instead the risk follows a smooth power law (when $\alpha > 1$) immediately upon exiting the initial $\Theta(\log K)$ plateau.
\end{remark}

\begin{remark}
    A similar sequential learning and scaling law phenomenology appears under population gradient flow with the squared loss and softmax activation. However, the dynamics progress more slowly, with the early and intermediate phases lasting longer, and the convergence phase exhibits a worse rate of $\Theta(t^{-1/2})$. See Appendix \ref{appendix:proofs_squared_loss} for details.
\end{remark}

\subsection{Noisy Case}\label{section:noisy_case}

In the presence of intra-class $\sigma$-sub-Gaussian noise, the structure of the learning dynamics is preserved: classes are recovered at hitting times $t_i \asymp (\log K)/\pi_i$, and the risk follows the same three-phase trajectory as in Proposition \ref{prop:noiseless_scaling_law} up to the time horizon $T \asymp 1/\sigma$, beyond which our perturbative bounds no longer provide a sharp characterization. We remark that since the class means have unit norm, $T$ can be interpreted as a signal-to-noise ratio.

\begin{theorem}[Mixture Model Scaling Law]\label{thm:mixture_model_scaling_law}
    Given the target \eqref{eq:target} with $\alpha > 0$ and $0 < \sigma^2 \leq 1$, there exists a time horizon $T = c_*/\sigma$ for some $c_* > 0$ and constants $0 < c_1 < c_2$ such that the cross-entropy risk along population gradient flow \eqref{eq:pop_gf} satisfies, for $t \leq T$,
    \begin{equation}\label{eq:risk_phases_gmm}
        \Rcal(t) \asymp
        \begin{cases}
            \log K, & 0 \leq t \lesssim Z_{K,\alpha} \log K \\
            f_{K,\alpha}(t), & Z_{K,\alpha} \log K \lesssim t \leq c_1 Z_{K,\alpha} K^{\alpha} \log K \\
            K/t, & c_2Z_{K,\alpha} K^{\alpha} \log K \leq t, \\
        \end{cases}
    \end{equation}
    where
    \begin{equation}
        f_{K,\alpha}(t) = 
        \begin{dcases}
            (\log K)\left(1-\left(\frac{t}{Z_{K,\alpha} K^{\alpha} \log K}\right)^{1/\alpha-1}\right), & 0 < \alpha < 1 \\
            \log\left(K(\log K)^2/t\right), & \alpha = 1\\
            (\log K)^{2-1/\alpha} \cdot t^{-(1-1/\alpha)}, & \alpha > 1. 
        \end{dcases}
    \end{equation}
    If $\alpha = 0$, we have $\Rcal(t) \asymp \log K$ for $0 \leq t \leq c_1 K\log K$ and $\Rcal(t) \asymp K/t$ for $c_2 K\log K \leq t$.
\end{theorem}

We prove this theorem in Appendix \ref{appendix:noisy_case}. Under additive noise, the population dynamics no longer admit an exact decoupling due to the presence of expressions of the form
\begin{equation}\label{eq:annoying_sigma_term}
    \pi_k \inner{\E_{\xibs_k}[p_i(\mubs_k + \xibs_k)\xibs_k], \mubs_j}
\end{equation}
in the summary statistic dynamics \eqref{eq:full_theta_dynamics}, which introduce cross-class interactions and expectations lacking a closed form. To handle this, we conduct a perturbative analysis where we Taylor expand $P_{i|i}$ in \eqref{eq:full_theta_dynamics} around $p_{i|i}$. We control the contribution of \eqref{eq:annoying_sigma_term} and the Taylor remainder by bounding weight norms. We show that $\max_i ||\wbs_i(t)|| \lesssim t$, which consequently means that the perturbative expansion remains controlled up to the time horizon $T \asymp 1/\sigma$. 
Depending on the value of $\sigma$, the dynamics may not enter all three phases from the idealized case. In particular:
\begin{itemize}
    \item If $\sigma \gtrsim 1/(Z_{K,\alpha}\log K)$, the risk dynamics are not guaranteed to exit the early phase.
    \item If $1/(Z_{K,\alpha}K^{\alpha} \log K) \lesssim \sigma \lesssim 1/(Z_{K,\alpha}\log K)$: The intermediate phase is reached, but only the first $(\sigma Z_{K,\alpha} \log K)^{-1/\alpha}$ classes are guaranteed to enter the head regime (see Appendix~\ref{appendix:noisy_case}).
    \item If $\sigma \lesssim 1/(Z_{K,\alpha} K^{\alpha} \log K)$, the dynamics are guaranteed to reach the convergence phase.
\end{itemize}

In Figure~\ref{fig:unsketched-noisy}, we plot the risk curves for a specific choice of $\sigma$ and various choices of $\alpha$. As predicted by Theorem~\ref{thm:mixture_model_scaling_law}, the dynamics follow the noiseless trajectory over a nontrivial time interval. However, the exact point at which the risk curve deviates from this trajectory, as well as its subsequent behaviour, are unexplained by our current theory; we leave this for future work.

\begin{figure}[!t] 
\centering
\begin{subfigure}[t]{0.49\linewidth}
\centering
{\includegraphics[height=0.75\textwidth]{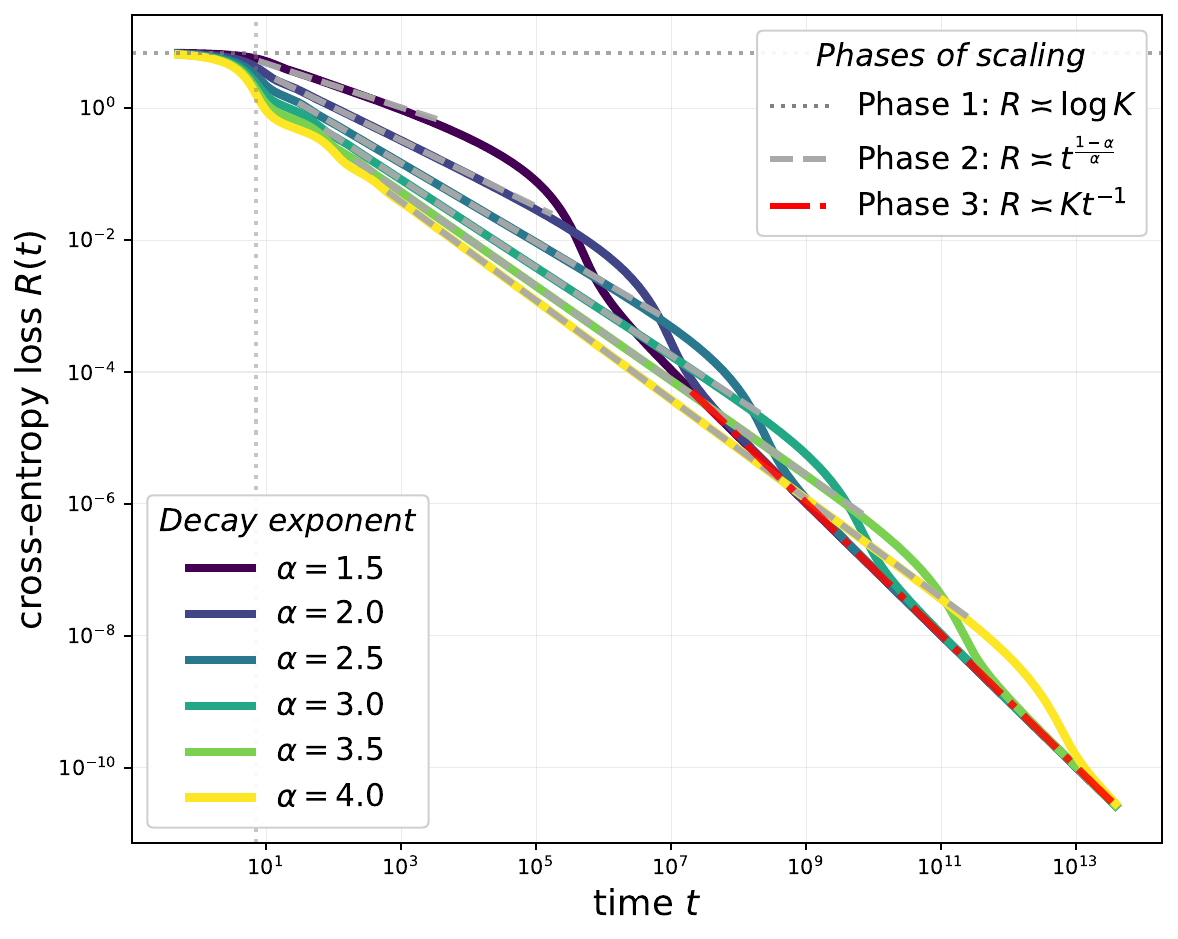}} 
\vspace{-1mm}
\caption{Noiseless setting $(\sigma=0).$}
\label{fig:unsketched-noiseless}
\end{subfigure}%
\begin{subfigure}[t]{0.49\linewidth}
\centering
{\includegraphics[height=0.75\textwidth]{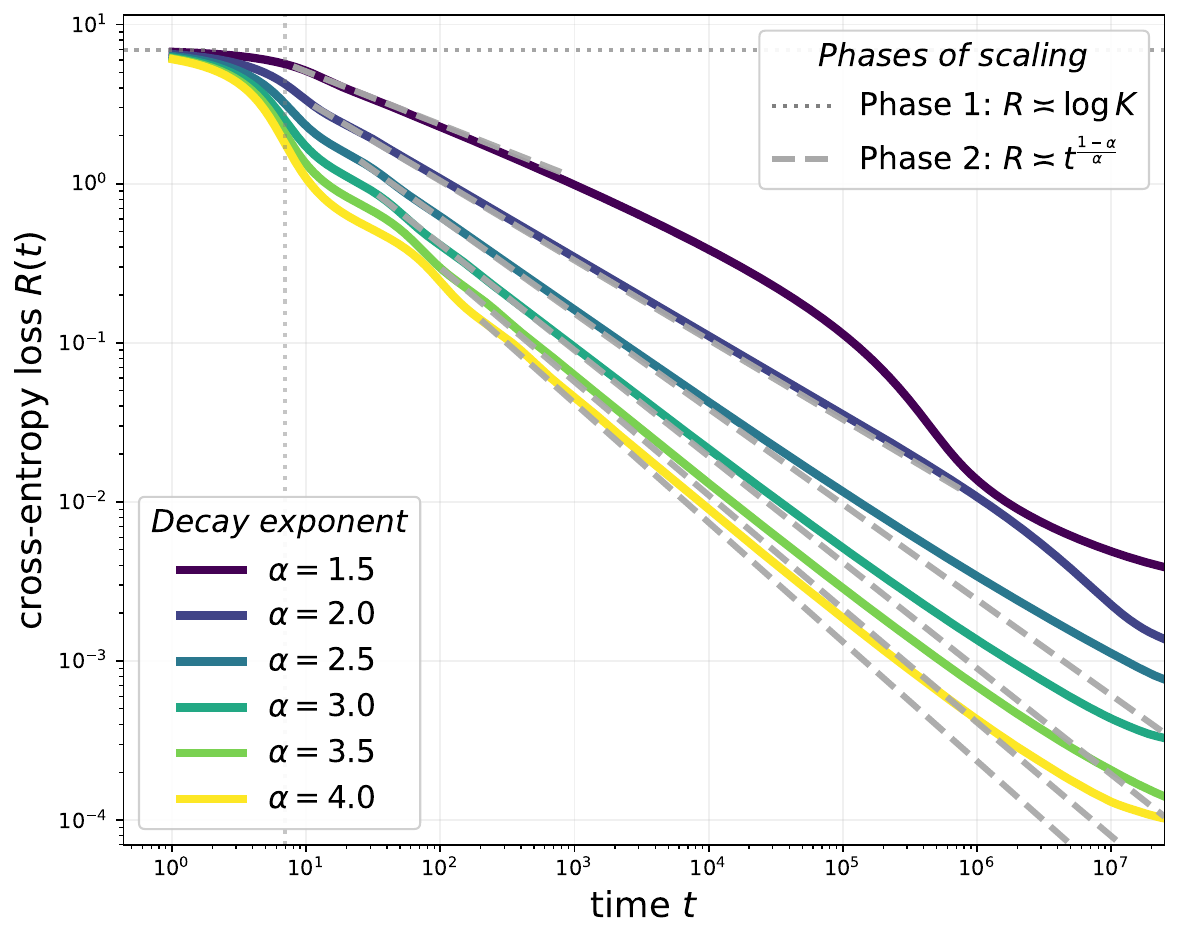}} 
\vspace{-1mm}
\caption{Noisy GMM setting $(\sigma=0.16).$}
\label{fig:unsketched-noisy}
\end{subfigure}%
\vspace{-1mm} 
\caption{\small 
Gradient flow trajectory and cross-entropy scaling law. We solve the population ODE \eqref{eq:full_theta_dynamics} using Euler discretization with adaptive step size, where expectations are approximated using Monte Carlo samples. We set $K=2048$ and vary the power-law exponent $\alpha$. 
\textbf{(a)} $\sigma=0$: The three-phase loss scaling matches the predictions of Proposition \ref{prop:noiseless_scaling_law}: after the initial plateau, the loss follows a power-law decay with $\alpha$-dependent exponent, before entering the local convergence phase with $1/t$ rate.  
\textbf{(b)} $\sigma>0$: The early-phase learning dynamics resemble the noiseless setting, but as $t$ increases, convergence slows down and deviates from theoretical rates due to isotropic Gaussian noise. 
}  
\label{fig:unsketched} 
\end{figure}

To better understand the late-time scaling behavior, we characterize the smallest achievable risk value, i.e., the \emph{irreducible risk}, in Appendix \ref{appendix:irreducible risk}. In the special case of a mixture of isotropic Gaussians target, we show that gradient flow converges to an optimal weight configuration. 

\section{Compute-Optimal Scaling Laws}\label{section:compute_optimal_sl}

Empirical scaling laws \cite{kaplan2020scaling,hoffmann2022training} are closely tied to \emph{compute-optimality}, i.e., finding the optimal tradeoff between the number of model parameters and training set size (or number of online SGD iterations) given a fixed compute budget. Theoretical studies of scaling laws in linear regression introduce a notion of parameter count through a random \emph{sketch} $\Sbs \in \R^{m \times d}$ with i.i.d.\@ $\Ncal(0,1/m)$ applied to the inputs. This induces a random features regression problem in $m$ dimensions \cite{lin2024scaling,paquette2024phases,bordelon2024dynamical}. The result is a scaling law of the form $\mathrm{MSE} \asymp t^{-\beta} + m^{-\gamma}$, which incorporates the model capacity constraint brought on by the sketch and yields the optimal choices of $m,t$ given a compute budget $C \asymp mt$. In our multiclass setting, evaluating all $K$ logits introduces an additional factor of $K$, so we take the compute model $C \asymp mKt$.

In the case of multiclass logistic regression with near-orthogonal class means, however, a Gaussian sketch largely preserves the classification geometry and therefore fails to induce the approximation bottleneck underlying compute-optimal scaling. While dimension reduction in linear regression naturally inhibits the model's flexibility in matching the output, this does not meaningfully reduce expressivity in logistic regression unless the sketch substantially increases overlap between mixture components. Provided $m \gtrsim \log K$, the Johnson-Lindenstrauss Lemma \cite{johnson1984extensions} implies $\norm{\Sbs \mubs_i - \Sbs \mubs_j} \approx \norm{\mubs_i - \mubs_j}$ for all class pairs $i,j \in [K]$.

\paragraph{PCA-based capacity control.} 
Therefore, instead of the Gaussian sketch, we rely on a classical statistical learning principle, dimension reduction through principal components analysis (PCA)~\cite{hastie2009elements},
which induces a capacity bottleneck by mapping onto the span of the means of the $m$ most frequent classes. For this to hold, we introduce an additional assumption to \eqref{eq:target}: the average intra-class covariance matrix is a multiple of the identity, i.e., $\sum_{i=1}^K \pi_i \Sigmabs_i = \tau^2 \Ibs_d$ for some $0 \leq \tau \leq \sigma$. Then, the second moment matrix is
\begin{equation}\label{eq:average_covariance_condition}
    \Gbs :=\E_{\xbs}[\xbs \xbs^{\top}] = \sum_{i=1}^K \pi_i \E_{\xibs_i}\left[(\mubs_i + \xibs_i)(\mubs_i + \xibs_i)^{\top}\right] = \sum_{i=1}^K \pi_i \mubs_i \mubs_i^{\top} + \tau^2 \Ibs_d,
\end{equation}
with top $m$ eigenvalues $(\pi_i + \tau^2)_{i=1}^m$ and associated eigenvectors $\mubs_1,\dots,\mubs_m$. The sketch $\Sbs \in \R^{m \times d}$ is then $\Mbs_{m}^{\top}$, the transpose of the matrix with the first $m$ class means as columns. In particular, the class means $\mubs_{m+1},\dots,\mubs_K$ are collapsed to the origin and rendered indistinguishable under any linear classifier in the sketched space.

In practice, this sketching procedure can be approximated in a data-driven way via PCA on the empirical second moment matrix $\widehat{\Gbs} = \frac{1}{N}\sum_{i=1}^N \xbs_i \xbs_i^{\top}$ given i.i.d\@ samples $\xbs_1,\dots,\xbs_N$ from \eqref{eq:target}. The resulting sketch~is~given~by $\widehat{\Sbs} = \widehat{\Vbs}_m^{\top} \in \R^{m \times d}$, where $\widehat{\Vbs}_m$ is the matrix with the top $m$ eigenvectors $\widehat{\vbs}_1,\dots,\widehat{\vbs}_m$ as columns.

In the noiseless case, the oracle sketch $\Sbs$ mapping to $\spanop\{\mubs_1,\dots,\mubs_m\}$ with $m < K$ introduces a capacity bottleneck that prevents the recovery of all class means, eliminating the convergence phase from \eqref{eq:idealized_risk_phases}. Taking the summary statistics to be $\theta_{ij} := \inner{\wbs_i, \Sbs \mubs_j}$ in this setting, we have $\theta_{ii}(t) = 0$ for all $i > m$ and for all $t \geq 0$. The softmax output $p_{i|i}$ for such classes remains constant at its initialization $1/K$.  Due to the power law prior $\pi_i \asymp i^{-\alpha}$, the capacity bottleneck is a power law in $m$:
\begin{equation}
    \Rcal_{\text{approx}} := -\sum_{i=m+1}^K \pi_i \log(1/K) \asymp m^{1-\alpha} \log K.
\end{equation}
Meanwhile, for classes $i \in [m]$, the same decoupled dynamics from \eqref{eq:idealized_p_ii_ode} are followed, giving the familiar optimization error 
\begin{equation}
    \Rcal_{\text{opt}} := (\log K)^{2-1/\alpha} t^{-(1-1/\alpha)}
\end{equation}
upon exiting the early phase. 

The above logic still holds when moving to an empirical sketch $\widehat{\Sbs}$ and $\sigma^2 > 0$ but for the addition of a time horizon and a sample size requirement for the PCA step.

\begin{theorem}[Scaling Law with Sketching]\label{thm:noisy_sketch}
    Under \eqref{eq:target} with fixed $\sigma^2 > 0$ and $\alpha > 1$, suppose $m \leq K/2$. Define the sketch $\widehat{\Sbs} = \Vbs_m^{\top} \in \R^{m \times d}$, the transpose of the matrix of top $m$ eigenvectors for the second moment matrix $\widehat{\Gbs}$ with $N$ samples. Let $\Rcal_{\widehat{\Sbs}}(t) = \E_{\xbs,y}[\Lcal(\Wbs(t); \widehat{\Sbs} \xbs, y)]$ be the cross-entropy risk of multiclass logistic regression under sketched inputs with weights $\Wbs(t) \in \R^{m \times K}$ following population gradient flow initialized at the origin. Then, for any constant $\delta \in (0,1)$, there exist a time horizon $T = c_*/\sigma$ for some $c_* > 0$ and a threshold $N_0(\sigma^2,m,d,\delta)$ such that for all $N \geq N_0$  and $t \leq T$, the risk satisfies
    \begin{equation}\label{eq:gmm-sketched-scaling-law}
        \Rcal_{\widehat{\Sbs}}(t) \asymp
        \begin{cases}
            \log K & 0 \leq t \lesssim \log K, \\
            m^{-(\alpha-1)} \log K + (\log K)^{2-1/\alpha} \cdot t^{-(1-1/\alpha)} & \log K \lesssim t \leq T,
        \end{cases}
    \end{equation}
    with probability at least $1-\delta$.
\end{theorem}

After exiting the same early phase as in the unsketched case, we arrive at a new capacity-limited phase where the risk decomposes into the sum of the approximation error $\Rcal_{\text{approx}}$ (a power law in $m$) and the optimization error $\Rcal_{\text{opt}}$ (a power law in $t$) --- see Figure \ref{fig:sketched}. This mirrors the scaling law in linear regression with Gaussian sketch under a capacity condition (i.e., where the input distribution has covariance spectrum that follows a power law $\lambda_i \asymp i^{-\alpha}$) \cite[Theorem 4.1]{lin2024scaling}, where the risk decomposes into approximation and optimization terms with the same exponents, though here we additionally have a $(\log K)$-dependence that enters. 

The proof of Theorem \ref{thm:noisy_sketch} is in Appendix \ref{appendix:gf_pca_empirical}. We rely on the same perturbative analysis as in Section \ref{section:noisy_case} to control the noise effects --- once again introducing the time horizon $T \asymp 1/\sigma$ --- and handle the appearance of randomly sketched means in the dynamics \eqref{eq:full_theta_dynamics} by controlling the gap between eigenspaces $\norm{\Mbs_m \Mbs_m^{\top} - \widehat{\Vbs}_m \widehat{\Vbs}_m}$. In the worst case, the latter introduces a dependence on the distance between sample and population second moment matrices $\norm{\widehat{\Gbs} - \Gbs}$ and the reciprocal of the eigengap $\pi_m - \pi_{m+1} \asymp m^{-(\alpha+1)}$. As a result, for the dynamics under empirical sketch to closely track those under population-level PCA (Appendix \ref{appendix:gf_pca_noisy}), we require $N \gtrsim (1+\sigma^2)^2 (d + \log(2/\delta)) m^{2\alpha+2} \max\left\{m^{2\alpha}, \frac{1}{\sigma^2 (\log K)^2}\right\}$ samples.\footnote{We did not attempt to optimize this worst-case Davis-Kahan bound. We emphasize however that these samples are unlabelled and our primary focus is on data and compute usage in SGD, not in pre-processing. The empirical PCA analysis is included to demonstrate that the proposed sketch arises from a finite-sample procedure.}

\paragraph{Compute-optimal scaling.} For a fixed compute budget $C \asymp mKt$, the optimal choices of hyperparameters in \eqref{eq:gmm-sketched-scaling-law} are given as
\begin{equation}\label{eq:compute_optimal_m_t}
    m^* \asymp \left(\frac{C}{K\log K}\right)^{1/(\alpha+1)}, \quad t^* \asymp (\log K)^{1/(\alpha+1)} \left(\frac{C}{K}\right)^{\alpha/(\alpha+1)},
\end{equation}
provided $K\log K \lesssim C \lesssim \min\{K^{\alpha+2} \log K,KT^{(\alpha+1)/\alpha}(\log K)^{-1/\alpha}\}$, which ensures $\log K \lesssim t^* \leq T$, and $m^* < K$. This gives a rate
\begin{equation}
    \Rcal_{\widehat{S}}(C) \asymp (\log K)^{2\alpha/(\alpha+1)} \cdot \left(\frac{K}{C}\right)^{(\alpha-1)/(\alpha+1)}
\end{equation}
in the capacity-limited phase.


\section{Online Stochastic Gradient Descent}\label{section:online_sgd}
We now move from the continuous-time population gradient flow to the discrete finite-sample online SGD, which we specify in Algorithm \ref{alg:online_sgd}. We show that the same scaling law phenomenology from Sections \ref{section:pop_gf} and \ref{section:compute_optimal_sl} holds for online SGD over a fixed time horizon under an appropriate choice of learning rate $\eta > 0$. Specifically, the risk scaling in online SGD with the iteration counter $n$ is analogous to that of gradient flow with the adjusted time scale $t = \eta n$. Importantly, the choice of learning rate depends on the effective dimension of the optimization problem being solved by SGD, which is $d$ in the unsketched case and $m$ in the sketched case. This further informs the compute-optimal trade-off, as more aggressive sketching (i.e., smaller $m$) allows for faster learning and less computation, but at the cost of an approximation error.

\begin{algorithm}[H]
    \caption{Online SGD}\label{alg:online_sgd}
    \textbf{Input:} Learning rate $\eta > 0$, number of iterations $n_{\max}$\\
    \textbf{Initialize} $\Wbs^{(0)} = \zerobs_{d \times K}$ \\
    \For{$n = 0,1,\dots,n_{\max}-1$}{
        Draw i.i.d.\@ sample $(\xbs^{(n)},y^{(n)})$ from the mixture \eqref{eq:target} \\
        Update $\Wbs^{(n+1)} \leftarrow \Wbs^{(n)}- \eta\nabla_{\Wbs} \Lcal(\Wbs^{(n)}; \xbs^{(n)}, y^{(n)})$
    }
    \textbf{Output} $\Wbs^{(n_{\max})}$
\end{algorithm}

We distinguish between the weights under population gradient flow, which we denote by $\Wbs(t)$ and the weights under online SGD, which we denote by $\Wbs^{(n)}$. We first state our result for the unsketched model.

\begin{theorem}[Scaling Law for SGD]\label{thm:sgd_gmm}
    In the same setting as Theorem \ref{thm:mixture_model_scaling_law}, consider online SGD (Algorithm \ref{alg:online_sgd}) with learning rate $\eta > 0$. Fix $\delta \in (0,1)$ and a finite horizon $T \leq T_{\sigma} := c_*/\sigma$. For each iteration $n$, define the effective optimization time $t_n = \eta n$. If, for a sufficiently small constant $c > 0$, 
    \begin{equation}
        \eta \leq c \min\left\{1, \frac{\Rcal(\Wbs(T))}{T(1+T)}, \frac{\delta \min\{\Rcal(\Wbs(T)), \Rcal(\Wbs(T))^2\}}{(1+\sigma^2 d)T}\right\},
    \end{equation}
    then there exist $0 < c_1 < c_2$ such that, with probability at least $1-\delta$, for $n \leq T/\eta$,
    \begin{equation}
        \Rcal(\Wbs^{(n)}) \asymp
        \begin{cases}
            \log K, & 0 \leq t_n \lesssim Z_{K,\alpha}\log K \\
            f_{K,\alpha}(t_n), & Z_{K,\alpha} \log K \lesssim t_n \leq c_1 Z_{K,\alpha} K^{\alpha} \log K \\
            K/t_n, & c_2 Z_{K,\alpha} K^{\alpha} \log K \leq t_n \leq T,
        \end{cases}
    \end{equation}
    where
    \begin{equation}
        f_{K,\alpha}(t) = 
        \begin{dcases}
            (\log K)\left(1- \left(\frac{t}{Z_{K,\alpha} K^{\alpha}\log K}\right)^{1/\alpha-1}\right), & 0 < \alpha < 1, \\
            \log(K (\log K)^2 /t), & \alpha = 1\\
            (\log K)^{2-1/\alpha} \cdot t^{-(1-1/\alpha)}, & \alpha > 1.
        \end{dcases}
    \end{equation}
\end{theorem}

The proof of this result is in Appendix \ref{appendix:proofs_online_sgd} and proceeds in two parts. First, a standard discretization argument (Appendix \ref{appendix:gf_to_gd}) controls the gap between population gradient flow and population gradient descent, relying on convexity and $L$-smoothness of $\Rcal$ with respect to the weights $\Wbs$, where $L \lesssim 1$. Second, a martingale concentration argument controls the fluctuations of online SGD relative to population GD (Appendix \ref{appendix:gd_to_sgd}), incurring the $d$-dependence in the choice of learning rate through the second moment bound $\E||\xbs||^2 \leq 1 + \sigma^2 d$.

\begin{remark}
    See Corollary \ref{cor:lr_condition_online_sgd_unsketched} for a breakdown of the sufficient learning rate conditions in each of three regimes for the time horizon $T = c_*/\sigma$ in population gradient flow.
\end{remark}


Next, we have the result for the sketched case. The proof of this result in Appendix \ref{appendix:online_sgd_empirical_sketch_case} follows the same argument as for the previous theorem, with the only change being to the martingale bound, which now scales with $m$ instead of $d$ due to the reduced problem dimension.

\begin{theorem}[Scaling Law for SGD, Sketched Case]\label{thm:sgd_sketched}
    In the same setting as Theorem \ref{thm:noisy_sketch}, consider online SGD (Algorithm \ref{alg:online_sgd}) with learning rate $\eta > 0$. Fix $\delta \in (0,1)$ and $T \leq T_{\sigma} := c_*/\sigma$. Assume the sample used to construct $\hat{\Sbs}$ is independent of the samples used in online SGD. Assume the number of samples $N$ for empirical PCA exceeds the threshold $N_0$ in Theorem \ref{thm:noisy_sketch}. If
    \begin{equation}
        \eta \lesssim \min\left\{1, \frac{\Rcal_{\widehat{\Sbs}}(\Wbs(T))}{T(1+T)}, \frac{\delta \min\{\Rcal_{\widehat{\Sbs}}(\Wbs(T)), \Rcal_{\widehat{\Sbs}}(\Wbs(T))^2\}}{(1+\sigma^2 m)T}\right\},
    \end{equation}
    then, with probability at least $1-\delta$,
    \begin{equation}
        \Rcal_{\widehat{\Sbs}}(\Wbs^{(n)}) \asymp 
        \begin{cases}
            \log K, & 0 \leq t_n \lesssim \log K \\
            m^{1-\alpha}\log K + (\log K)^{2-1/\alpha} \cdot t_n^{-(1-1/\alpha)}, & \log K \lesssim t_n \leq T
        \end{cases}
    \end{equation}
    uniformly for all $n$ satisfying $n\eta \leq T$.
\end{theorem}

Just as we did in Section \ref{section:compute_optimal_sl}, we can derive the compute optimal trade-off between the sketch dimension $m$ and the number of samples $n$. Setting a fixed level of compute $C \asymp mKn$ and balancing the approximation and optimization errors in this regime gives the optimal hyperparameters, under the conditions 
\begin{equation}
    m^* \asymp \left(\frac{\eta C}{K \log K}\right)^{1/(\alpha+1)}, \quad n^* \asymp \left(\frac{\log K}{\eta}\right)^{1/(\alpha+1)}\left(\frac{C}{K}\right)^{\alpha/(\alpha+1)},
\end{equation}
provided $(K\log K)/\eta \lesssim C \lesssim (K^{\alpha+2}\log K)/\eta$ and $\eta n^* \leq T$. This gives a compute-optimal scaling law
\begin{equation}
    \Rcal_{\widehat{\Sbs}}(C) \asymp (\log K)^{2\alpha/(\alpha+1)} \left(\frac{K}{\eta C}\right)^{(\alpha-1)/(\alpha+1)}.
\end{equation}

\section{Conclusion}\label{section:conclusion}




In this work, we studied gradient-based learning in multiclass logistic regression under cross-entropy loss, in a regime where the number of classes is large and the class frequencies follow a power law. We showed that the training dynamics proceed sequentially, with more frequent classes learned earlier; aggregating these class-wise learning curves yields a three-phase scaling law for the cross-entropy risk. We further showed that a PCA-type dimension reduction induces a natural capacity bottleneck for classification, leading to a decomposition
of the risk into approximation and optimization errors. Trading these errors off determines the compute-optimal choices of parameter count and training time. Finally, we studied the time-discretized dynamics and presented scaling laws for online SGD.

Important directions remain open. First, under intra-class sub-Gaussian noise, our current proof controls the dynamics perturbatively up to a noise-limited time horizon; a refined analysis beyond this perturbative approach could lead to sharper power-law rates and additional scaling regimes. Second, our SGD analysis focuses on one-pass online training. It would be interesting to study how the scaling laws change under different optimizers \citep{ferbach2025dimension,kim2026scaling}, and under data reuse, where multiple epochs may modify the compute-optimal tradeoff between model size, sample size, and optimization time \citep{pillaud2018statistical,lin2025improved,yan2025larger}.

\bigskip

\begin{ack}
    The authors thank Noah Marshall for helpful discussions. Resources used in preparing this research were provided, in part, by the Province of Ontario, the Government of Canada through CIFAR, and companies sponsoring the Vector Institute. KCT acknowledges support from NSERC through the PGS-D program. CT acknowledges support from an NSERC Discovery Grant, the Alliance Grant ALLRP 581098-22, and a CIFAR AI Catalyst grant. MAE acknowledges support from the NSERC Grant [2019-06167], the CIFAR AI Chairs program, the CIFAR AI Catalyst grant, and the Ontario Early Researcher Award.
\end{ack}

{
\small
\bibliographystyle{alpha}
\bibliography{ref}
}

\newpage
\appendix
\tableofcontents

\newpage

\section{Proofs for Unsketched Population Gradient Flow}\label{appendix:proofs_pop_grad_flow}
This appendix contains the proofs of the main results in Section \ref{section:pop_gf}, covering the behaviour of population gradient flow in multi-class logistic regression with the teacher \eqref{eq:target}:
\begin{equation}
    \dot{\Wbs}(t) = -\nabla_{\Wbs} \Rcal(\Wbs), \quad \Wbs(0) = \zerobs_{d \times K}
\end{equation}
where
\begin{equation}\label{eq:unregularized_risk}
    \begin{split}
        \Rcal(\Wbs) &:= \E_{\xbs,y}[\Lcal(\Wbs;\xbs,y)] \\
        &= -\sum_{i=1}^K \pi_i \E_{\xbs|i} \bigg[\log \frac{e^{\inner{\xbs,\wbs_i}}}{\sum_{k=1}^K e^{\inner{\xbs,\wbs_k}}}\bigg] \\
        &= - \sum_{i=1}^K \pi_i \inner{\mubs_i,\wbs_i} + \sum_{i=1}^K \pi_i \E_{\xbs|i}\bigg[\log\bigg(\sum_{k=1}^K e^{\inner{\xbs,\wbs_k}}\bigg)\bigg].
    \end{split}
\end{equation}

Now, the difficulty lies in the lack of a closed form for the expectation in the last line of \eqref{eq:unregularized_risk}. We track the key quantities
\begin{equation}
    \theta_{ij} := \inner{\wbs_i, \mubs_j}, \quad p_{i|j} = p_i(\mubs_j), \quad P_{i|j} = \E_{\xbs|y=j}[p_i(\xbs)], \quad i,j \in [K].
\end{equation}

The gradient of the risk with respect to a given weight $\wbs_i$, $i \in [K]$, is
\begin{equation}
    \begin{split}
    \nabla_{\wbs_i} \Rcal(\Wbs) &= -\pi_i \mubs_i + \E_{\xbs} \bigg[\frac{e^{\inner{\xbs,\wbs_i}}}{\sum_{k=1}^K e^{\inner{\xbs,\wbs_k}}}\xbs\bigg] \\ 
    &= -\pi_i\mubs_i + \E_{\xbs}[p_i(\xbs)\xbs] \\
    &= -\pi_i \mubs_i + \sum_{k=1}^K \pi_k P_{i|k} \mubs_k + \sum_{k=1}^K \pi_k \E_{\xibs_k}[p_i(\mubs_k + \xibs_k)\xibs_k].
    \end{split}
\end{equation}
Hence, the weight dynamics are
\begin{equation}\label{eq:full_unsketched_w_dynamics}
    \dot{\wbs}_i = \sum_{k=1}^K \pi_k (\1\{i=k\} - P_{i|k})\mubs_k - \sum_{k=1}^K \pi_k \E_{\xibs_k}[p_i(\mubs_k + \xibs_k)\xibs_k],
\end{equation}
and the summary statistic dynamics are
\begin{equation}
    \dot{\theta}_{ij} = \pi_j (\1\{i=j\} - P_{i|j}) - \sum_{k=1}^K \pi_k \inner{\E_{\xibs_k}[p_i(\mubs_k + \xibs_k)\xibs_k], \mubs_j}.
\end{equation}

\begin{remark}[The GMM Case $\xibs_j \sim \Ncal(\zerobs,\sigma^2 \Ibs_d)$]
In the mixture model where the additive noise $\xibs_j$ for each class $j \in [K]$ is isotropic Gaussian, the dynamics simplify via
\begin{equation}\label{eq:gaussian_integral_expansion}
    \begin{split}
        \E_{\xbs}[p_i(\xbs)\xbs] &= \sum_{k=1}^K  \pi_k \E_{\xbs|k}[p_i(\xbs)\xbs] \\
        &= \sum_{k=1}^K \pi_k \E_{\xbs|k}[p_i(\xbs)]\mubs_k + \sigma^2 \sum_{k=1}^K \pi_k \E_{\xbs|k}[\nabla_{\xbs} p_i(\xbs)] \\
        &= \sum_{k=1}^K \pi_k P_{i|k} \mubs_k + \sigma^2 \E_{\xbs}[\nabla_{\xbs}p_i(\xbs)],
    \end{split}
\end{equation}
by Stein's lemma. So, we have the gradient flow trajectory
\begin{equation}
    \dot{\wbs}_i = \sum_{k=1}^K \pi_k (\1\{i=k\} - P_{i|k}) \mubs_k - \sigma^2 \E_{\xbs}[\nabla_{\xbs} p_i(\xbs)].
\end{equation}

Focusing on the last term, we expand
\begin{equation}
    \nabla_{\xbs}p_i(\xbs) = \sum_{k=1}^K p_i(\xbs) \big(\1\{i=k\} - p_k(\xbs)\big)\wbs_k.
\end{equation}
With this, we can express the gradient flow as
\begin{equation}
    \dot{\wbs}_i = \sum_{k=1}^K \pi_k (\1\{i=k\}-P_{i|k})\mubs_k - \sigma^2 \sum_{\ell=1}^K \E_{\xbs}\big[p_i(\xbs)(\1\{i=\ell\}-p_{\ell}(\xbs))\big]\wbs_{\ell}.
\end{equation}
Hence, by the orthonormality assumption in \eqref{eq:target} the dynamics of the summary statistics are
\begin{equation}
    \dot{\theta}_{ij} = \pi_j(\1\{i=j\}-P_{i|j}) - \sigma^2 \sum_{\ell=1}^K \E_{\xbs}\big[p_i(\xbs)(\1\{i=\ell\} - p_{\ell}(\xbs))\big] \theta_{\ell j},
\end{equation}
So, in this case, we have an ODE system in terms of the summary statistics $(\theta_{ij})_{i,j \in [K]}$.
\end{remark}

\begin{remark}\label{rem:risk_is_convex}
Note that $\Rcal$ is convex. A simple calculation yields
\begin{equation}\label{eq:risk_gradient}
    \nabla_{\wbs_i} \Rcal(\Wbs) = -\E_{\xbs,y}\left[(\1\{y=i\} - p_i(\xbs))\xbs\right]
\end{equation}
and
\begin{equation}
    \nabla^2_{\wbs_i, \wbs_j} \Rcal(\Wbs) = \E_{\xbs}\left[p_i(\xbs) \left(\1\{i=j\} - p_j(\xbs)\right)\xbs \xbs^{\top}\right].
\end{equation}
Then, given $\Vbs \in \R^{d \times K}$, we have
\begin{equation}
    \begin{split}
        \nabla^2 \Rcal(\Wbs)[\Vbs, \Vbs] &= \sum_{i,j=1}^K \vbs_i^{\top} \nabla^2_{\wbs_i, \wbs_j} \Rcal(\Wbs) \vbs_j \\
        &= \E_{\xbs}\left[\sum_{i,j=1}^K p_i(\xbs)(\1\{i=j\} - p_j(\xbs)) \inner{\vbs_i, \xbs} \inner{\vbs_j, \xbs}\right] \\
        &= \E_{\xbs} \left[\sum_{i=1}^K p_i(\xbs) \inner{\vbs_i,\xbs}^2 - \left(\sum_{i=1}^K p_i(\xbs) \inner{\vbs_i,\xbs}\right)^2\right] \\
        &= \E_{\xbs}\left[\Var_{J \sim p(\cdot|\xbs)}\left(\inner{\vbs_J, \xbs}\right)\right].
    \end{split}
\end{equation}
The above is nonnegative and therefore $\Rcal$ is convex.
\end{remark}

\subsection{Noiseless Case}\label{appendix:noiseless_case}

We proceed with a more refined analysis to exactly characterize the risk curve over the course of training. We begin with the noiseless case $\sigma^2 = 0$, where the dynamics for each class cleanly decouple.

\subsubsection{Class-wise Learning Dynamics}

A key property of \eqref{eq:full_theta_dynamics} is the following conservation law, which gives the exact decoupling.

\begin{lemma}[Conservation Law]\label{lem:conservation_law}
    For all $t \geq 0$, $j \in [K]$, we have
    \begin{equation}
        \sum_{i=1}^K \theta_{ij}(t) = 0.
    \end{equation}
    Moreover, when $\sigma^2 = 0$,
    \begin{equation}\label{eq:conservation_law}
        \theta_{ij}(t) = -\frac{1}{K-1} \theta_{jj}(t) \quad \text{for all $i \neq j$.}
    \end{equation}
\end{lemma}
\begin{proof}
    We have
    \begin{equation}
        \sum_{i=1}^K \dot{\theta}_{ij} = \pi_j \left(1 - \sum_{i=1}^K P_{i|j}\right) - \sum_{i=1}^K \sum_{k=1}^K \pi_k \inner{\E_{\xibs_k}[p_i(\mubs_k + \xibs_k)\xibs_k], \mubs_j} = 0
    \end{equation}
    because probabilities sum to one. Since $\theta_{ij}(0) = 0$ for all $i,j$, it immediately follows that $\sum_{i=1}^K \theta_{ij}(t) = 0$ for all $t \geq 0$. When $\sigma^2 = 0$, the dynamics of the $\theta_{ij}$ do not depend on $\pi_i$ but rather only on $\pi_j$ and hence are identical for all $i \neq j$, yielding \eqref{eq:conservation_law}.
\end{proof}

\begin{remark}
    \eqref{eq:conservation_law} does \emph{not} hold when there is additive intra-class noise and the $\pi_i$ are not all equal.
\end{remark}

In the noiseless case, the summary statistic dynamics simplify to
\begin{equation}\label{eq:theta_ode_idealized}
    \dot{\theta}_{ij} = \pi_j\1\{i=j\} - \pi_j p_{i|j}.
\end{equation}

We write $p_{i|i} = e^{\theta_{ii}}/Z_i$, where, by Lemma \ref{lem:conservation_law}, $Z_i = e^{\theta_{ii}} + (K-1)e^{-\frac{1}{K-1}\theta_{ii}}$. The quotient rule gives
\begin{equation}
    \begin{split}
        \dot{p}_{i|i} &= p_{i|i} \dot{\theta}_{ii} - \frac{e^{\theta_{ii}}(e^{\theta_{ii}}\dot{\theta}_{ii} - e^{-\frac{1}{K-1}\theta_{ii}}\dot{\theta}_{ii})}{Z_i^2} \\
        &= \pi_i p_{i|i} (1-p_{i|i}) \bigg(1 - \bigg(p_{i|i} - \frac{1}{K-1}(1-p_{i|i})\bigg)\bigg) \\
        &= \frac{K}{K-1} \pi_i p_{i|i}(1-p_{i|i})^2,
    \end{split}
\end{equation}
giving us \eqref{eq:idealized_p_ii_ode} from the main text. 

\begin{proof}[Proof of Lemma \ref{lem:noiseless_sequential_learning}]
It is immediate from the above that $\dot{p}_{i|i}(t) \geq 0$ and $1/K \leq p_{i|i}(t) \leq 1$ for all $t \geq 0$. Separation of variables gives
\begin{equation}
    \frac{1}{p_{i|i}(1-p_{i|i})^2} \dee p_{i|i} = \frac{K}{K-1}\pi_i \dee t.
\end{equation}
A partial fraction decomposition then yields
\begin{equation}
    \int \frac{1}{p_{i|i}} + \frac{1}{1-p_{i|i}} + \frac{1}{(1-p_{i|i})^2} \dee p = \int \frac{K}{K-1} \pi_i \dee t,
\end{equation}
from which we obtain 
\begin{equation}\label{eq:implicit_idealized_ode_solution}
    \log \frac{p_{i|i}(t)}{1-p_{i|i}(t)} + \frac{1}{1-p_{i|i}(t)} = -\log (K-1) + \frac{K}{K-1} + \frac{K}{K-1} \pi_i t. 
\end{equation}
We can use this to solve for $t_i(c)$ by plugging in $p_{i|i}(t_i(c)) = c$, obtaining
\begin{equation}\label{eq:hitting_time_exact}
    t_i(c) = \frac{K-1}{K} \frac{1}{\pi_i} \left[\log \frac{c}{1-c} + \frac{1}{1-c} + \log(K-1) - \frac{K}{K-1}\right] \asymp \frac{\log K}{\pi_i}.
\end{equation}
\end{proof}

In particular, we obtain
\begin{equation}\label{eq:hitting_time_1/2}
    t_i(1/2) = \frac{K-1}{K} \frac{1}{\pi_i} \left[2 + \log (K-1) - \frac{K}{K-1}\right].
\end{equation}
Henceforth, we use $t_i$ without its argument to denote $t_i(1/2)$.

\subsubsection{Risk Decomposition}

The pieces are now in place to combine the dynamics of all classes. Fix a time $t \geq 0$ and define
\begin{equation}
    k^*(t) := \left|\{i \in [K]: t_i \leq t\}\right|.
\end{equation}
Solving $t_i \leq t$ for $i$ yields
\begin{equation}
    i \leq \left(\frac{K}{K-1}\right)^{1/\alpha} \left(\frac{t}{Z_{K,\alpha} (2+ \log(K-1) - \frac{K}{K-1})}\right)^{1/\alpha}.
\end{equation}
Hence, $k^*(t)$ is the floor of the expression on the RHS above. For simplicity, we omit constants and write
\begin{equation}
    k^*(t) \asymp \left(\frac{t}{Z_{K,\alpha} \log K}\right)^{1/\alpha}.
\end{equation}
We decompose the cross-entropy risk into class-wise losses:
\begin{equation}
    \Rcal(t) = -\sum_{i=1}^K \pi_i \log p_{i|i}.
\end{equation}
We split this into the loss on \emph{head} classes --- those for which $p_{i|i}(t) > 1/2$, or equivalently $i \leq k^*(t)$ --- and the loss on the remaining \emph{tail} classes:
\begin{equation}
    \Rcal(t) = \underbrace{-\sum_{i=1}^{k^*(t)} \pi_i \log p_{i|i}(t)}_{\Rcal_{\text{head}}(t)} - \underbrace{\sum_{i=k^*(t)+1}^K \pi_i\log p_{i|i}(t)}_{\Rcal_{\text{tail}}(t)}.
\end{equation}

\subsubsection{Tail Risk}\label{appendix:idealized_tail_risk}

Since $-\log p_{i|i}(t)$ is non-increasing for all $t \geq 0$, we trivially have $\Rcal_{\text{tail}}(t) \leq \Rcal_{\text{tail}}(0) = \log K$. When $k^*(t) \geq 1$, we obtain a sharper upper bound by the integral approximation. For $\alpha \neq 1$,
\begin{equation}\label{eq:tail_risk_upper_bound}
   \frac{\log K}{Z_{K,\alpha}} \sum_{i=k^*+1}^K i^{-\alpha} \leq \frac{\log K}{Z_{K,\alpha}}\int_{k^*}^K x^{-\alpha} \, \dee x = \frac{\log K}{Z_{K,\alpha}}\frac{x^{1-\alpha}}{1-\alpha} \bigg|_{k^*}^K = \frac{\log K}{(1-\alpha)Z_{K,\alpha}}\left(K^{1-\alpha} - (k^*)^{1-\alpha}\right).
\end{equation}
To lower bound the tail risk, define
\begin{equation}
    k'(t) := \left|\{i \in [K]: t_i(K^{-1/2}) \leq t\}\right|.
\end{equation}
Plugging into \eqref{eq:hitting_time_exact} gives
\begin{equation}
    \begin{split}
    t_i(K^{-1/2}) &= \frac{K-1}{K} \frac{1}{\pi_i} \left[\log(K-1) - \frac{1}{2}\log K - \log(1-K^{-1/2}) + \frac{1}{1-K^{-1/2}} - \frac{K}{K-1}\right] \\
    &\asymp \frac{\log K}{\pi_i}.
    \end{split}
\end{equation}
Hence, $(k'(t)+1) \asymp (k^*(t)+1)$. Then, the tail error satisfies, for $\alpha \neq 1$
\begin{equation}\label{eq:tail_risk_lower_bound}
    \begin{split}
    \Rcal_{\text{tail}}(t) &\geq \frac{1}{2}\sum_{i=k'(t)+1}^K \pi_i \log K \\
    &\geq \frac{\log K}{2Z_{K,\alpha}} \int_{k'(t)+1}^{K} x^{-\alpha} \dee x = \frac{\log K}{2(1-\alpha)Z_{K,\alpha}}\left(K^{1-\alpha} - (k'(t)+1)^{1-\alpha}\right).
    \end{split}
\end{equation}

\subsubsection{Head Risk}\label{appendix:idealized_head_risk}

For the head risk, rearranging \eqref{eq:hitting_time_1/2} gives
\begin{equation}
    \frac{K}{K-1}\pi_i(t-t_i) = \log \frac{p_{i|i}(t)}{1-p_{i|i}(t)} + \frac{1}{1-p_{i|i}(t)} - 2.
\end{equation}
When $p_{i|i} \geq 1/2$, we have $\log\frac{p_{i|i}}{1-p_{i|i}} \geq 0$. Hence,
\begin{equation}
    1-p_{i|i}(t) \geq \frac{1}{2(1+\frac{K}{K-1}\pi_i (t-t_i))}.
\end{equation}
On the other hand,
\begin{equation}
    \log \frac{p_{i|i}(t)}{1-p_{i|i}(t)} \leq \log \frac{1}{1-p_{i|i}(t)} \leq \frac{1}{1-p_{i|i}(t)} - 1,
\end{equation}
so
\begin{equation}
    \frac{K}{K-1}\pi_i (t-t_i) \leq \frac{2}{1-p_{i|i}(t)} - 3.
\end{equation}
With this, we have the upper bound
\begin{equation}
    1-p_{i|i}(t) \leq \frac{2}{3+\frac{K}{K-1}\pi_i (t-t_i)} \leq \frac{2}{1+\frac{K}{K-1}\pi_i (t-t_i)}.
\end{equation}
Moreover, when $p_{i|i} \geq 1/2$ we have
\begin{equation}\label{eq:log_p_inequality}
    1-p_{i|i}(t) \leq -\log p_{i|i}(t) \leq 2(1-p_{i|i}(t)).
\end{equation}
Putting these together, we can tightly control the head risk as
\begin{equation}
    \Rcal_{\text{head}}(t) \asymp -\sum_{i=1}^{k^*(t)} \pi_i\log p_{i|i}(t) \asymp \sum_{i=1}^{k^*(t)} \frac{\pi_i}{1+\frac{K}{K-1}\pi_i(t-t_i)} \asymp \sum_{i=1}^{k^*(t)} \frac{1}{t-t_i + \pi_i^{-1}}.
\end{equation}
Using the relation $\pi_i^{-1} \asymp t_i/(\log K)$, we can lower bound the head risk:
\begin{equation}
    \Rcal_{\text{head}}(t) \gtrsim \frac{k^*(t)}{t} \asymp \left(\frac{1}{Z_{K,\alpha} \log K}\right)^{1/\alpha} \cdot t^{(1-\alpha)/\alpha}
\end{equation}
As for the upper bound, we split the head classes into those which have been recovered recently and those which have not. For $i \leq k^*(t)/2$, we see that
\begin{equation}
    \frac{t_i}{t} \leq \frac{t_i}{t_{k^*}} = \left(\frac{i}{k^*}\right)^{\alpha} \leq \frac{1}{2^{\alpha}}.
\end{equation}
In particular, we can write $t - t_i \gtrsim t$ and obtain
\begin{equation}
    \sum_{i=1}^{\lceil k^*(t)/2 \rceil} \frac{1}{t-t_i + \pi_i^{-1}} \lesssim \frac{k^*(t)}{t}. 
\end{equation}
Now, for the case $k^*(t)/2 \leq i \leq k^*(t)$, we write
\begin{equation}
    t -t_i +\pi_i^{-1} \gtrsim t \left(1 - \left(\frac{i}{k^*(t)}\right)^{\alpha} + \frac{1}{\log K}\right) \asymp t\left(\frac{k^*(t)-i}{k^*(t)} + \frac{1}{\log K}\right),
\end{equation}
where the final comparison follows because for fixed $\alpha > 0$, $(1-x^{\alpha}) \asymp (1-x)$ holds uniformly over $x \in [1/2,1]$. The contribution of these classes is of order at most
\begin{equation}
    \begin{split}
    \sum_{i=\lceil k^*(t)/2 \rceil}^{k^*(t)} \frac{1}{t(\frac{k^*(t)-i}{k^*(t)} + \frac{1}{\log K})} &= \frac{k^*(t)}{t} \sum_{i=\lceil k^*(t)/2 \rceil}^{k^*(t)} \frac{1}{k^*(t)-i + \frac{k^*(t)}{\log K}} \\
    &= \frac{k^*(t)}{t} \left[\frac{\log K}{k^*(t)} + \sum_{i=1}^{\lceil k^*(t)/2 \rceil} \frac{1}{i + \frac{k^*(t)}{\log K}}\right],
    \end{split}
\end{equation}
We control the last term with the integral approximation
\begin{equation}
    \sum_{i=1}^{k^*(t)/2} \frac{1}{i + \frac{k^*(t)}{\log K}} \leq \int_0^{k^*(t)/2} \frac{1}{x+\frac{k^*(t)}{\log K}} \dee x = \log\left(1 + \frac{\log K}{2}\right).
\end{equation}
Hence, we can conclude
\begin{equation}\label{eq:head_risk_bounds}
    \frac{k^*(t)}{t} \lesssim \Rcal_{\text{head}}(t) \lesssim \frac{k^*(t)}{t} \left(\log(1+\log K)\right) + \frac{\log K}{t}.
\end{equation}
Note that when $t - t_K \gtrsim t$, the upper bound collapses to $K/t$ and matches the lower bound.

\subsubsection{The Case \texorpdfstring{$\alpha > 1$}{alpha > 1}}\label{appendix:noiseless_case_alpha>1}

When $\alpha > 1$, the normalization factor is constant order: $Z_{K,\alpha} \asymp 1$ and therefore $k^*(t) \asymp (t/\log K)^{1/\alpha}$. We consider three regimes for $t$.

\textbf{Early phase.} First, when $t < t_1$, we have $k^*(t) = 0$, which gives $\Rcal_{\text{head}}(t) = 0$. The upper bound $\Rcal_{\text{tail}}(t) \lesssim \log K$ is clear. The lower bound $\Rcal_{\text{tail}}(t) \gtrsim \log K$ follows from the fact $k'(t) \asymp k^*(t)$ and thus, in \eqref{eq:tail_risk_lower_bound}, we have 
\begin{equation}
    \frac{\log K}{2(\alpha-1)} \left((k'(t)+1)^{1-\alpha} - K^{1-\alpha}\right) \asymp \log K.
\end{equation}

\textbf{Intermediate phase.} Consider times when at most a constant proportion of classes have been recovered, i.e., $t_1 \leq t < c_1 t_K$ for some $c_1 \in (0,1)$. More specifically, we choose $c_1$ so that our upper and lower bounds for $\Rcal_{\text{tail}}(t)$ match up to constants. For $K \geq 4$, this is done by enforcing $k^*(t) \leq k'(t) \leq K/2$ (though any constant order denominator works). Recall the hitting time
\begin{equation}
    t_i(K^{-1/2}) = \frac{K-1}{K} \frac{1}{\pi_i} \left[\log \frac{K^{-1/2}}{1-K^{-1/2}} + \frac{1}{1-K^{-1/2}} + \log (K-1) - \frac{K}{K-1}\right].
\end{equation}
Let $c_1 \in (0,1)$ and suppose $t \leq c_1 t_K$. Then, if a class index $i$ satisfies $i \leq k'(t)$, this implies $t_i(K^{-1/2}) \leq t \leq c_1 t_K$ and
\begin{equation}
    \begin{split}
    &\frac{K-1}{K} \frac{1}{\pi_i}\left[\log \frac{K^{-1/2}}{1-K^{-1/2}} + \frac{1}{1-K^{-1/2}} + \log(K-1) - \frac{K}{K-1}\right] \\
    &\leq c_1 \frac{K-1}{K} \frac{1}{\pi_K} \left[2 + \log(K-1) - \frac{K}{K-1}\right].
    \end{split}
\end{equation}
Rearranging yields
\begin{equation}\label{eq:exact_c1_equation}
    \left(\frac{i}{K}\right)^{\alpha} \leq c_1 \frac{2 + \log(K-1) - \frac{K}{K-1}}{\log \frac{K^{-1/2}}{1-K^{-1/2}} + \frac{1}{1-K^{-1/2}} + \log(K-1) - \frac{K}{K-1}}.
\end{equation}
To simplify and bound the ratio on the RHS, let $r = K^{1/2}$. Then,
\begin{equation}
    \begin{split}
        &\log \frac{K^{-1/2}}{1-K^{-1/2}} + \frac{1}{1-K^{-1/2}} + \log(K-1) - \frac{K}{K-1} \\
        &= \log \frac{1}{r(1-r^{-1})} + \frac{1}{1-r^{-1}} + \log(r^2 - 1) - \frac{r^2}{r^2-1} \\
        &= \log(r+1) + \frac{r}{r^2-1}
    \end{split}
\end{equation}
and
\begin{equation}
    2 + \log(K-1) - \frac{K}{K-1} = \log(r^2 - 1) + 1 - \frac{1}{r^2-1}.
\end{equation}
The ratio of these quantities is
\begin{equation}
    \frac{\log(r^2-1)+1 - \frac{1}{r^2-1}}{\log(r+1)+\frac{r}{r^2-1}} \leq \frac{2\log(r+1)+1}{\log(r+1)} = 2 + \frac{1}{\log(r+1)} < 3.
\end{equation}
Returning to \eqref{eq:exact_c1_equation}, we see that $c_1 = 2^{-(\alpha + 2)}$ is a valid choice to obtain $i/K \leq 1/2$.

Then, the tail upper and lower bounds match (up to constants):
\begin{equation}
    \Rcal_{\text{tail}}(t) \asymp (\log K) (k^*(t))^{1-\alpha} \asymp (\log K)^{2-1/\alpha} \cdot t^{(1-\alpha)/\alpha}.
\end{equation}
Meanwhile, the head risk satisfies
\begin{equation}
    \Rcal_{\text{head}}(t) \lesssim (\log K)^{-1/\alpha} \log(1+\log K))\cdot t^{(1-\alpha)/\alpha} + \log K \cdot t^{-1},
\end{equation}
which is of lower order than the tail risk. Hence, $\Rcal(t) \asymp (\log K)^{2-1/\alpha} \cdot t^{-(1-1/\alpha)}$ in this intermediate phase.

\textbf{Convergence phase.} When $t > (1+\delta)t_K$ for $\delta > 0$, we may write $t - t_K \gtrsim t$ and the head risk collapses to $\Theta(K/t)$. Since $\Rcal_{\text{tail}}(t) = 0$ in this regime, we have $\Rcal(t) \asymp K/t$.

\subsubsection{The Case \texorpdfstring{$\alpha = 0$}{alpha = 0}}

When $\alpha = 0$, $\pi_i = 1/K$ for all $i \in [K]$. From \eqref{eq:full_theta_dynamics}, we immediately see that the dynamics along each class are identical. There is \emph{no intermediate phase}, since $t_1 = \dots = t_K \asymp K\log K$. Either all classes are in the head, or all classes are in the tail. Instead the risk transitions from the early search phase, where $\Rcal(t) \asymp \log K$, to the convergence phase.

\textbf{Early phase.} When $t < c_1t_1$, all classes are in the tail, and we can use the observation $k'(t) \asymp K$ in a similar fashion to the previous subsection to argue $\Rcal(t) \gtrsim \log K$, where the constant in the bound depends on $c_1$.

\textbf{Convergence phase.} When $t > (1+\delta)t_K = (1+\delta)t_1$ for $\delta > 0$, we may write $t - t_i \gtrsim t$ for all $i \in [K]$ and the head risk collapses to $\Theta(K/t)$. Since $\Rcal_{\text{tail}}(t) = 0$ in this regime, we have $\Rcal(t) \asymp K/t$.

\subsubsection{The Case \texorpdfstring{$0 < \alpha < 1$}{0 < alpha < 1}}

When $0 < \alpha < 1$, the normalization factor
\begin{equation}
    Z_{K,\alpha} = \sum_{i=1}^K i^{-\alpha} \asymp K^{1-\alpha},
\end{equation}
is not constant. As a result, we have 
\begin{equation}
    t_i \asymp \frac{\log K}{\pi_i} \asymp K^{1-\alpha} i^{\alpha} \log K, \quad k^*(t) \asymp \left(\frac{t}{K^{1-\alpha}\log K}\right)^{1/\alpha}.
\end{equation}

The three-phase decomposition follows in a similar fashion to the previous subsection.

\textbf{Early phase.} When $t \leq t_1$, we have the tail lower bound
\begin{equation}
    \frac{1}{Z_{K,\alpha}}\frac{\log K}{2(1-\alpha)} (K^{1-\alpha} - (k')^{1-\alpha}) \asymp \log K. 
\end{equation}
Hence, $\Rcal(t) \asymp \log K$ in the early phase.

\textbf{Intermediate phase.} When $t_1 \leq t < c_1 t_K$ with $c_1 = 2^{-(\alpha+2)}$ as in the previous subsection, the tail risk is of order
\begin{equation}
    \begin{split}
    \frac{\log K}{Z_{K,\alpha}} (K^{1-\alpha} - (k^*(t))^{1-\alpha}) &\asymp (\log K) \left(1- \left(\frac{k^*(t)}{K}\right)^{1-\alpha}\right) \\
    &\asymp (\log K)\left(1 - \left(\frac{t}{t_K}\right)^{1/\alpha-1}\right) \\
    &\asymp \log K.
    \end{split}
\end{equation}
Meanwhile, recalling the head risk upper bound \eqref{eq:head_risk_bounds}, we have
\begin{equation}
    \Rcal_{\text{head}}(t) \lesssim \frac{\log(1+\log K)}{\log K} \left(\frac{t}{t_K}\right)^{1/\alpha-1} + \frac{\log K}{t}.
\end{equation}
For $t \geq t_1 \asymp K^{1-\alpha}\log K$, we have $(\log K)/t \lesssim K^{-(1-\alpha)}$. Moreover, $t \leq c_1t_K$ implies $(t/(K\log K))^{1/\alpha-1} = O(1)$. So, the tail risk dominates the head risk in the regime under consideration and $\Rcal(t) \asymp \log K$.

\textbf{Convergence phase.} The argument is the exact same as in the previous subsection, and we conclude $\Rcal(t) \asymp K/t$ when $t \geq (1+\delta)t_K$ for any constant $\delta > 0$.

\subsubsection{The Case \texorpdfstring{$\alpha = 1$}{alpha = 1}}

In the special case $\alpha = 1$, we have
\begin{equation}
    Z_{K,1} \asymp \sum_{i=1}^K i^{-1} \asymp \log K, \quad t_i \asymp (\log K)^2 i, \quad k^*(t) \asymp \frac{t}{(\log K)^2}.
\end{equation}
The tail lower bound becomes
\begin{equation}
    \Rcal_{\text{tail}}(t) \geq \frac{\log K}{2Z_{K,\alpha}} \int_{k'(t)+1}^K \frac{1}{x} \, \dee x \asymp \left(\log K - \log k'(t)\right) = \log \frac{K}{k'(t)},
\end{equation}
and the upper bound is similar, with $k^*(t)$ replacing $k'(t)$.

\textbf{Early phase.} When $t < t_1$, we have $\Rcal(t) \asymp \log K$ from the tail bounds.

\textbf{Intermediate phase.} When $t_1 \leq t \leq c_1t_K$, with $c_1 = 2^{-(\alpha+2)}$ as in the previous subsections, we have 
\begin{equation}
    \Rcal_{\text{tail}}(t) \asymp \log \left(K(\log K)^2/t\right). 
\end{equation}
Meanwhile, the head upper bound gives
\begin{equation}
    \Rcal_{\text{head}}(t) \lesssim \frac{\log \log K}{(\log K)^2} + \frac{\log K}{t},
\end{equation}
which is of lower order than $\Rcal_{\text{tail}}$ in the regime under consideration.

\textbf{Convergence phase.} The same argument from the previous subsections to give $\Rcal(t) \asymp K/t$ when $t \geq (1+\delta)t_K$ for any constant $\delta > 0$ still goes through.

We have now covered all cases to prove Proposition \ref{prop:noiseless_scaling_law}.

\subsection{The Noisy Case}\label{appendix:noisy_case}

We now consider the role of sub-Gaussian noise $\xibs_j$, $j \in [K]$, and prove Theorem \ref{thm:mixture_model_scaling_law}. Recall the definitions
\begin{equation}
    p_i(\xbs) := \frac{\exp(\inner{\xbs, \wbs_i})}{\sum_{k=1}^K \exp(\inner{\xbs, \wbs_k})}, \quad p_{i|j} := p_i(\mubs_j), \quad P_{i|j} := \E_{\xbs|j}[p_i(\xbs)], \quad \theta_{ij} = \inner{\wbs_i, \mubs_j},
\end{equation}
the summary statistic dynamics
\begin{equation}
    \dot{\theta}_{ij} = \pi_j (\1\{i=j\} - P_{i|j}) - \sum_{k=1}^K \pi_k \inner{\E_{\xibs_k}[p_i(\mubs_k + \xibs_k)\xibs_k], \mubs_j},
\end{equation}
and the $\sigma$-sub-Gaussianity condition
\begin{equation}\label{eq:sub_gaussianity}
    \E_{\xibs_j}[e^{\inner{\ubs, \xibs_j}}] \leq e^{\sigma^2 ||u||^2/2}, \quad j \in [K].  
\end{equation}

These dynamics do not admit a decoupling in the same way that they did in the idealized noiseless case. We handle this by treating non-zero noise as an additive perturbation. 

\subsubsection{Perturbation Lemmas}\label{appendix:perturbation_lemmas}

We begin by bounding the evolution of the weight norms, which in turn gives bounds on the summary statistics themselves by Cauchy-Schwarz. Define 
\begin{equation}
    M(t) := \max_{j \in [K]} ||\wbs_j(t)||,
\end{equation}
i.e., the largest weight norm at time $t$. The growth of $M$ induces a time horizon $T$ over which the noisy dynamics remain close to the idealized noiseless case.

\begin{lemma}[Weight Norm Bound]\label{lem:weight_norm_bound}
    For every $t \geq 0$ and $\sigma^2 \geq 0$, $M(t) \leq t$. Consequently, for any constant $c_* > 0$, if we define $T := c_*/\sigma$, $0 < \sigma \leq 1$, we have, uniformly for $0 \leq t \leq T$,
    \begin{equation}\label{eq:time_horizon_bounds}
        M(t) \leq t, \quad \sigma M(t) \leq c_*, \quad \sigma^2 \int_0^t M(s) \dee s \leq \frac{c_*^2}{2}.
    \end{equation}
\end{lemma}
\begin{proof}
    We have
    \begin{equation}
        \frac{\dee}{\dee t} \left[\frac{1}{2}\norm{\nabla \Rcal(\Wbs(t))}^2\right] = -\inner{\nabla \Rcal(\Wbs(t)), \nabla^2 \Rcal(\Wbs(t)) \nabla \Rcal(\Wbs(t))} \leq 0,
    \end{equation}
    by convexity of $\Rcal$ with respect to $\Wbs$, which we established in Remark \ref{rem:risk_is_convex}. Moreover, 
    \begin{equation}
        \nabla_{\wbs_i} \Rcal(\zerobs) = -\pi_i \mubs_i + \frac{1}{K} \sum_{k=1}^K \pi_k \mubs_k,
    \end{equation}
    and therefore $\norm{\nabla \Rcal(\Wbs(0))}_F \leq 1$. With this in hand, we can bound
    \begin{equation}
        M(t) \leq \norm{\Wbs(t)}_F \leq \int_0^t \norm{\nabla \Rcal(\Wbs(s))}_F \dee t \leq t.
    \end{equation}
\end{proof}

\begin{lemma}[Noise Remainder]\label{lem:noise_remainder}
    For every $i,j,k \in [K]$ and $t \geq 0$,
    \begin{equation}
        \left|\E_{\xibs_j}[p_k(\mubs_j + \xibs_j)\inner{\mubs_i, \xibs_j}]\right| \leq \frac{\sigma^2}{2} M(t) \leq \frac{\sigma^2 t}{2}.
    \end{equation}
\end{lemma}
\begin{proof}
    Define the scalar function
    \begin{equation}
        f(u) = \E\left[\log\left(\sum_{\ell=1}^K p_{\ell|j} \exp\left(\inner{\wbs_{\ell} + u\1\{\ell=k\} \mubs_i, \xibs_j}\right)\right)\right].
    \end{equation}
    Then, the quantity of interest is $f'(0) = \E[p_k(\mubs_j + \xibs_j) \inner{\mubs_i, \xibs_j}]$, and moreover, letting
    \begin{equation}
        q = \frac{p_{k|j} \exp(\inner{\wbs_k + u\mubs_i, \xibs_j})}{\sum_{\ell=1}^K p_{\ell|j} \exp(\inner{\wbs_{\ell} + u\1\{\ell =k \}\mubs_i, \xibs_j})},
    \end{equation}
    we have
    \begin{equation}
        f''(u) = \E\left[q(1-q)\inner{\mubs_i, \xibs_j}^2\right] \leq \frac{1}{4}\E[\inner{\mubs_i, \xibs_j}^2] \leq \frac{\sigma^2}{4}.
    \end{equation}
    The weighted AM-GM inequality gives
    \begin{equation}
        \begin{split}
        f(u) &\geq \E\left[\log\left(\prod_{\ell=1}^K \exp\left(p_{\ell|j} \inner{\wbs_{\ell} + u\1\{\ell=k\} \mubs_i, \xibs_j}\right)\right)\right] \\
        &\geq \E\left[\sum_{\ell=1}^K p_{\ell|j} \inner{\wbs_{\ell},\xibs_j} + up_{k|j} \inner{\mubs_i, \xibs_j}\right] \\
        &= 0.
        \end{split}
    \end{equation}
    We combine this with the Taylor expansion
    \begin{equation}
        0 \leq f(0) + uf'(0) + \frac{\sigma^2 u^2}{8}.
    \end{equation}
    Minimizing the RHS with respect to $u$ gives the choice $u = -4f'(0)/\sigma^2$, from which we derive
    \begin{equation}
        |f'(0)|^2 \leq \frac{\sigma^2}{2}f(0).
    \end{equation}
    By Jensen's inequality and sub-Gaussianity \eqref{eq:sub_gaussianity}, we have
    \begin{equation}
        f(0) \leq \log\left(\sum_{\ell=1}^K p_{\ell|j} e^{\sigma^2 ||\wbs_{\ell}||^2/2}\right) \leq \frac{\sigma^2 M(t)^2}{2},
    \end{equation}
    which we plug into the previous equation to yield the result.
\end{proof}

Next, we relate the model probabilities assigned to the class means, $p_{i|j} = p_i(\mubs_j)$, to the conditional expectations $P_{i|j} = \E_{\xbs|y=j}[p_i(\xbs)]$ that appear in the noisy dynamics. 
\begin{lemma}[Perturbation of the Softmax]\label{lem:softmax_perturbation}
    There exists $C > 0$ such that, for all $i,j \in [K]$ and all $t \geq 0$, 
    \begin{equation}
        \absolute{P_{i|j}(t) - p_{i|j}(t)} \leq C\sigma^2 e^{C\sigma^2 M(t)^2} p_{i|j}(1-p_{i|j}) M(t)^2.
    \end{equation}
    In particular, for $t \in [0,T]$, where $T$ is as in Lemma \ref{lem:weight_norm_bound}, 
    \begin{equation}
        \absolute{P_{i|j}(t) - p_{i|j}(t)} \lesssim \sigma^2 p_{i|j}(1-p_{i|j}) M(t)^2.
    \end{equation}
\end{lemma}
\begin{proof}
    Recall $P_{i|j} = \E_{\xibs}[p_i(\mubs_j + \xibs_j)]$. We Taylor-expand the inside of the expectation about $\xibs_j = 0$. Using the integral form of the Taylor remainder, this yields
    \begin{equation}\label{eq:taylor_remainder}
        P_{i|j} - p_{i|j} = \int_0^1 (1-s) \E_{\xibs_j}\left[\nabla^2 p_i(\mubs_j + s \xibs_j)[\xibs_j,\xibs_j]\right] \dee s.
    \end{equation}
    The first-order term vanishes as $\xibs_j$ has mean zero. The Hessian takes the form 
    \begin{equation}
        \nabla_{\xbs}^2 p_i(\xbs) = p_i(\xbs)\left[(\wbs_i - \bar{\wbs})(\wbs_i - \bar{\wbs})^{\top} - \sum_{k=1}^K p_k(\xbs) (\wbs_k-\bar{\wbs})(\wbs_k - \bar{\wbs})^{\top}\right].
    \end{equation}
    Specifically, given a direction $\hbs \in \R^d$,
    \begin{equation}
        \begin{split}
            \absolute{\nabla^2 p_i(\xbs)[\hbs,\hbs]} &= p_i(\xbs)\absolute*{\inner{\wbs_i - \bar{\wbs}(\xbs), \hbs}^2 - \sum_{k=1}^K p_k(\xbs)\inner{\wbs_k - \bar{\wbs}(\xbs), \hbs}^2} \\
            &\leq p_i(\xbs) \left[\inner{\wbs_i - \bar{\wbs}(\xbs),\hbs}^2 + \sum_{k=1}^K p_k(\xbs) \inner{\wbs_k - \bar{\wbs}(\xbs),\hbs}^2\right] \\
            &= p_i(\xbs) \sum_{k=1}^K p_k(\xbs) \inner{\wbs_k-\wbs_i,\hbs}^2.
        \end{split}
    \end{equation}
    We can write
    \begin{equation}
        p_i(\mubs_j + s\xibs_j) = \frac{p_{i|j} e^{s \inner{\wbs_i, \xibs_j}}}{\sum_{k=1}^K p_{k|j} e^{s \inner{\wbs_k, \xibs_j}}}.
    \end{equation}
    The weighted AM-GM inequality gives
    \begin{equation}
        \sum_{k=1}^K p_{k|j} e^{s \inner{\wbs_k, \xibs_j}} \geq \prod_{k=1}^K (e^{s\inner{\wbs_k,\xibs_j}})^{p_{k|j}} = \prod_{k=1}^K e^{s p_{k|j}\inner{\wbs_k,\xibs_j}} = e^{s\inner{\bar{\wbs}^{(j)},\xibs_j}},
    \end{equation}
    where $\bar{\wbs}^{(j)} = \sum_{k=1}^K p_{k|j}\wbs_k$. Subsequently, for $k \neq i$, we may write
    \begin{equation}
        p_i(\mubs_j + s \xibs_j)p_k(\mubs_j + s \xibs_j) \leq p_{i|j} p_{k|j} \exp(s \inner{\wbs_i + \wbs_k - 2\bar{\wbs}^{(j)}, \xibs_j}).
    \end{equation}
    Then,
    \begin{equation}
        \begin{split}
            \E_{\xibs_j}\absolute*{\nabla^2 p_i(\mubs_j + s \xibs_j)[\xibs_j,\xibs_j]} &\leq p_{i|j} \sum_{k \neq i} p_{k|j} \E_{\xibs_j}\left[\exp(s \inner{\wbs_i + \wbs_k - 2\bar{\wbs}^{(j)},\xibs_j})\inner{\wbs_k - \wbs_i, \xibs_j}^2\right] \\
            &\leq C\sigma^2  p_{i|j} (1-p_{i|j})\exp(C\sigma^2 M(t)^2) M(t)^2
        \end{split}
    \end{equation}
    for a sufficiently large constant $C > 0$, by Cauchy-Schwarz and sub-Gaussianity of $\xibs_j$. The result immediately follows by using the above to bound the Taylor remainder \eqref{eq:taylor_remainder}.
\end{proof}

Next, we compare the expression for the risk in the idealized case to the one in the noisy mixture model case.
\begin{lemma}\label{lem:log_p_inequality}
    For every $i \in [K]$, 
    \begin{equation}
        -\log p_{i|i} \leq \E_{\xbs|y=i}[-\log p_i(\xbs)] \leq e^{2\sigma^2 M(t)^2} (-\log p_{i|i}).
    \end{equation}
    Hence, uniformly on $t \in [0,T]$,
    \begin{equation}
        \Rcal(t) \asymp -\sum_{i=1}^K \pi_i \log p_{i|i}(t).
    \end{equation}
\end{lemma}
\begin{proof}
    First observe that the risk can be expressed as
    \begin{equation}
        \Rcal(t) = -\sum_{i=1}^K \pi_i \E_{\xbs|y=i}[\log p_i(\xbs)].
    \end{equation}
    We can re-write
    \begin{equation}
        -\log p_i(\mubs_i + \xibs_i) = \log\left(1 + \sum_{k \neq i} \exp(\inner{\wbs_k - \wbs_i, \mubs_i} + \inner{\wbs_k - \wbs_i, \xibs_i})\right).
    \end{equation}
    The lower bound is then immediate from Jensen's inequality. For the upper bound, we leverage concavity of the logarithm and sub-Gaussianity:
    \begin{equation}
        \begin{split}
            \E_{\xibs_i}[-\log p_i(\mubs_i + \xibs_i)] &\leq \log\left(1 + \sum_{k \neq i} \exp(\inner{\wbs_k - \wbs_i, \mubs_i}) \E_{\xibs_i} \left[\exp(\inner{\wbs_k - \wbs_i, \xibs_i})\right]\right) \\
            &\leq \log\left(1 + e^{2\sigma^2 M(t)^2} \sum_{k \neq i} \exp(\inner{\wbs_k - \wbs_i, \mubs_i})\right) \\
            &\leq \log\left(\left[1 + \sum_{k \neq i} \exp(\inner{\wbs_k - \wbs_i, \mubs_i})\right]^{\exp(2\sigma^2 M(t)^2)}\right) \\
            &= e^{2\sigma^2 M(t)^2} (-\log p_{i|i}(t)).
        \end{split}
    \end{equation}
    By \eqref{eq:time_horizon_bounds}, $e^{2\sigma^2 M(t)^2}$ is bounded by a constant over the time horizon $[0,T]$.
\end{proof}

\subsubsection{Class-wise Learning Dynamics}

\begin{lemma}\label{lem:noisy_p_ii_bounds}
    There exist constants $a,b > 0$ such that the following holds for all $i \in [K]$ and all $t \in [0,T]$:
    \begin{equation}
        a \pi_i p_{i|i}(1-p_{i|i})^2 - \sigma^2 M(t) p_{i|i}(1-p_{i|i}) \leq \dot{p}_{i|i} \leq b\pi_ip_{i|i}(1-p_{i|i})^2 + \sigma^2 M(t) p_{i|i} (1-p_{i|i}).
    \end{equation}
\end{lemma}
\begin{proof}
    Recall the summary statistic dynamics
    \begin{equation}
        \dot{\theta}_{ki} = \pi_i(\1\{k=i\} - P_{k|i}) - \sum_{\ell=1}^K \pi_{\ell} \inner{\E_{\xibs_{\ell}}[p_k(\mubs_{\ell} + \xibs_{\ell})\xibs_{\ell}], \mubs_i}.
    \end{equation}
    Then, since $p_{i|i} = e^{\theta_{ii}}/\sum_{k=1}^K e^{\theta_{ki}}$, we have
    \begin{equation}
        \dot{p}_{i|i} = p_{i|i}\left(\dot{\theta}_{ii} - \sum_{k=1}^K p_{k|i} \dot{\theta}_{ki}\right).
    \end{equation}
    We work with the RHS of the ODE above by separating it into the contributions of the first and second terms of the summary statistic dynamics. The first term contributes
    \begin{equation}
        \pi_i p_{i|i} \left(1-P_{i|i} - \sum_{k=1}^K p_{k|i} (\1\{k=i\} - P_{k|i})\right) = \pi_i p_{i|i} \left((1-p_{i|i})(1-P_{i|i}) + \sum_{k \neq i} p_{k|i} P_{k|i}\right).
    \end{equation}
    If we replace $P_{k|i}$ by $p_{k|i}$ for all $k \in [K]$, the expression becomes
    \begin{equation}
        \pi_i p_{i|i} \left((1-p_{i|i})^2 + \sum_{k \neq i} p_{k|i}^2\right).
    \end{equation}
    By Cauchy-Schwarz, we have
    \begin{equation}
        \frac{(1-p_{i|i})^2}{K-1} \leq \sum_{k \neq i} p_{k|i}^2 \leq (1-p_{i|i})^2.
    \end{equation}
    Hence,
    \begin{equation}\label{eq:perturbed_p_ii_sandwich}
        \frac{K}{K-1} \pi_i p_{i|i}(1-p_{i|i})^2 \leq \pi_i p_{i|i} \left((1-p_{i|i})^2 + \sum_{k \neq i} p_{k|i}^2\right) \leq 2\pi_i p_{i|i}(1-p_{i|i})^2.
    \end{equation}
    The error due to replacing $P_{k|i}$ by $p_{k|i}$ is
    \begin{equation}\label{eq:P_p_replacement_error}
        \begin{split}
        &\absolute*{\pi_i p_{i|i}\left((1-p_{i|i})(p_{i|i}-P_{i|i}) + \sum_{k \neq i} p_{k|i}(P_{k|i} - p_{k|i})\right)} \\
        &\leq \pi_i p_{i|i} \left((1-p_{i|i})|p_{i|i}-P_{i|i}| + \sum_{k \neq i} p_{k|i} |P_{k|i} - p_{k|i}|\right).
        \end{split}
    \end{equation}
    By Lemma \ref{lem:softmax_perturbation}, we have
    \begin{equation}
        (1-p_{i|i})\absolute{p_{i|i} - P_{i|i}} \leq C\sigma^2 e^{C\sigma^2 M(t)^2} p_{i|i} (1-p_{i|i})^2 M(t)^2
    \end{equation}
    and
    \begin{equation}
        \sum_{k \neq i} p_{k|i} |P_{k|i} - p_{k|i}| \leq C\sigma^2 e^{C\sigma^2 M(t)^2}M(t)^2 \underbrace{\sum_{k \neq i} p_{k|i}^2(1-p_{k|i})}_{\leq (1-p_{i|i})^2}.
    \end{equation}
    Therefore, the error in \eqref{eq:P_p_replacement_error} is at most
    \begin{equation}
        C \pi_i \sigma^2 e^{C\sigma^2 M(t)^2} p_{i|i} (1-p_{i|i})^2 M(t)^2 \lesssim c_*^2 e^{Cc_*^2} \pi_i p_{i|i}(1-p_{i|i})^2,
    \end{equation}
    where the last inequality follows from \eqref{eq:time_horizon_bounds}. Here, $c_*$ can be chosen sufficiently small so that the error is at most a constant fraction of $\pi_i p_{i|i}(1-p_{i|i})^2$.

    Now, we move to the contribution of the second term in the summary statistic dynamics, which we bound as
    \begin{equation}
        \left|\sum_{\ell=1}^K \pi_{\ell} \E_{\xibs_{\ell}}[p_k(\mubs_{\ell} + \xibs_{\ell})\inner{\xibs_{\ell},\mubs_i}]\right| \leq \frac{\sigma^2 M(t)}{2}
    \end{equation}
    by Lemma \ref{lem:noise_remainder}.

    Hence, the remaining contribution to $\dot{p}_{i|i}$ has absolute value at most
    \begin{equation}
        p_{i|i}\sum_{k=1}^K |\1\{k=i\} - p_{k|i}| \frac{\sigma^2 M(t)}{2} = \sigma^2 M(t) p_{i|i}(1-p_{i|i}).
    \end{equation}
\end{proof}

For every $i \in [K]$, define the hitting times as we did in the idealized case:
\begin{equation}
    t_i := \inf\{t \geq 0: p_{i|i}(t) \geq 1/2\}.
\end{equation}
The perturbative analysis of the learning dynamics above gives us a noisy analogue to our earlier sequential recovery finding (Lemma \ref{lem:noiseless_sequential_learning}).
\begin{lemma}[Sequential Recovery]\label{lem:noisy_sequential_recovery}
    There exist constants $0 < c < C$ such that 
    \begin{enumerate}
        \item If $t_i \leq T$, then $t_i \geq c(\log K)/\pi_i$.
        \item Whenever $C(\log K)/\pi_i \leq T$, we have $t_i \leq C(\log K)/\pi_i$, i.e., class $i$ is guaranteed to be recovered within the horizon $[0,T]$.
    \end{enumerate}
\end{lemma}
\begin{proof}
    Suppose $t < t_i$. Then, $1/2 \leq 1-p_{i|i} \leq 1$. From Lemma \ref{lem:noisy_p_ii_bounds}, there exist constants $a_0, b_0, C_0 > 0$ (which do not depend on $i$) such that
    \begin{equation}
        a_0\pi_i - C_0\sigma^2 M(t) \leq \frac{\dee}{\dee t} \log p_{i|i}(t) \leq b_0 \pi_i + C\sigma^2 M(t).
    \end{equation}
    Assume further that $t \in [0,T]$. Integrating from zero and using $p_{i|i}(0) = 1/K$ gives, for some constant $C > 0$
    \begin{equation}
        -\log K + a_0 \pi_i t - C c_*^2 \leq \log p_{i|i}(t) \leq - \log K + b_0 \pi_i t + C c_*^2,
    \end{equation}
    where we have used \eqref{eq:time_horizon_bounds}. If $t_i \leq T$, the upper bound shows that $p_{i|i}$ cannot reach $1/2$ prior to time being a constant multiple of $(\log K)/\pi_i$. This proves the first point.

    For a sufficiently large constant multiple of $(\log K)/\pi_i$, the lower bound exceeds $-\log 2$, proving the second point.
\end{proof}

The learning curves after recovery also exhibit similar behaviour to the idealized case.
\begin{lemma}[Post-Recovery Decay]\label{lem:post_recovery_decay}
    Let $i \in [K]$ and suppose $t_i \leq T$. Then $p_{i|i}(t) \geq 1/2$ for all $t \in [t_i,T]$, and, uniformly for $t_i \leq t \leq T$,
    \begin{equation}
        1-p_{i|i}(t) \asymp \frac{1}{1+\pi_i(t-t_i)}, \quad -\pi_i \log p_{i|i}(t) \asymp \frac{1}{t-t_i + \pi_i^{-1}}.
    \end{equation}
\end{lemma}
\begin{proof}
    From Lemma \ref{lem:noisy_sequential_recovery}, we have that $\pi_i t_i \gtrsim \log K$. Since $t_i \leq T$, it follows that $\pi_i \gtrsim (\log K)/T \asymp \sigma \log K$. On the other hand, \eqref{eq:time_horizon_bounds} gives $\sigma^2 M(t) \lesssim c_*\sigma$ for $t \leq T$. Hence, $c_*$ can be taken sufficiently small so that
    \begin{equation}
        \frac{\dee}{\dee t} \log p_{i|i}(t) \geq a_0 \pi_i - C_0\sigma^2 M(t) > 0
    \end{equation}
    when $p_{i|i}(t) = 1/2$. This proves that $p_{i|i}(t) \in [1/2,1]$ for $t \in [t_i,T]$. Combining this with Lemma \ref{lem:noisy_p_ii_bounds}, we have 
    \begin{equation}
        \frac{a}{2}\pi_i (1-p_{i|i}(t))^2 - \sigma^2 M(t)(1-p_{i|i}(t)) \leq \dot{p}_{i|i} \leq b \pi_i(1-p_{i|i})^2 + \sigma^2 M(t) (1-p_{i|i}).
    \end{equation}
    Then,
    \begin{equation}
        \frac{a}{2}\pi_i - \frac{\sigma^2 M(t)}{1-p_{i|i}(t)} \leq \frac{\dee}{\dee t} \left[\frac{1}{1-p_{i|i}(t)}\right] \leq b\pi_i + \frac{\sigma^2 M(t)}{1-p_{i|i}(t)}.
    \end{equation}
    We multiply by the integrating factor $\exp(-\sigma \int_{t_i}^t M(r) \dee r)$ to obtain, 
    \begin{equation}
        \frac{\dee}{\dee t}\left[\exp\left(-\sigma^2 \int_{t_i}^t M(r) \dee r\right) \frac{1}{1-p_{i|i}(t)}\right] \leq b\pi_i \exp\left(-\sigma^2 \int_{t_i}^t M(r) \dee r\right).
    \end{equation}
    Integrating from $t_i$ to $t$ gives
    \begin{equation}
        \exp\left(-\sigma^2 \int_{t_i}^t M(r) \dee r\right) \frac{1}{1-p_{i|i}(t)} \leq \frac{1}{1-p_{i|i}(t_i)} + b\pi_i \int_{t_i}^t \exp\left(-\sigma^2 \int_{t_i}^s M(r) \dee r\right) \dee s.
    \end{equation}
    Thus,
    \begin{equation}
        \frac{1}{1-p_{i|i}(t)} \leq 2\exp\left(\sigma^2 \int_{t_i}^t M(r) \dee r\right) + b\pi_i \int_{t_i}^t \exp\left(\sigma^2 \int_s^t M(r) \dee r\right) \dee s.
    \end{equation}
    Since $\sigma^2 \int_{t_i}^T M(r) \dee r \leq \sigma^2 \int_0^T M(r) \dee r \leq c_*^2/2$ by \eqref{eq:time_horizon_bounds}, we have the bound
    \begin{equation}
        \frac{1}{1-p_{i|i}(t)} \lesssim 1 + \pi_i (t-t_i).
    \end{equation}
    The same integrating factor argument with $\exp(\sigma^2 \int_{t_i}^t M(r) \dee r)$ gives the matching lower bound (up to constants). Hence,
    \begin{equation}
        1-p_{i|i}(t) \asymp \frac{1}{1+\pi_i(t-t_i)}.
    \end{equation}
    Since $p_{i|i}(t) \geq 1/2$, \eqref{eq:log_p_inequality} gives
    \begin{equation}
        -\pi_i \log p_{i|i}(t) \asymp \pi_i (1-p_{i|i}(t)) \asymp \frac{1}{t-t_i + \pi_i^{-1}}.
    \end{equation}
\end{proof}

Similarly to Appendix \ref{appendix:noiseless_case}, we define the head and tail risks
\begin{equation}
    \Rcal_{\text{head}}(t) := \sum_{i: t_i \leq t} \pi_i \E_{\xbs|y=i}[-\log p_i(\xbs)], \quad \Rcal_{\text{tail}}(t) := \sum_{i: t_i > t} \pi_i \E_{\xbs|y=i}[-\log p_i(\xbs)].
\end{equation}

\subsubsection{Tail Risk}\label{appendix:gmm_tail_risk}

Lemma \ref{lem:noisy_p_ii_bounds} implies the crude bound
\begin{equation}
    \frac{\dee}{\dee t} \log p_{i|i}(t) \geq -\sigma^2 M(t) (1-p_{i|i}(t)) \geq -\sigma^2 M(t).
\end{equation}
Integrating this and multiplying by $-1$ gives $-\log p_{i|i}(t) \lesssim \log K$. By Lemma \ref{lem:log_p_inequality}, we have
\begin{equation}\label{eq:risk_logK_upper_bound}
    \E_{\xbs|y=i}[-\log p_i(\xbs)] \lesssim \log K, \quad \forall \, i \in [K], t \in [0,T].
\end{equation}
On the other hand, Lemma \ref{lem:noisy_p_ii_bounds} also gives
\begin{equation}
    \log p_{i|i}(t) \leq -\log K + b_0 \pi_i t + C'c_*^2, \quad \forall \, t \in [0,T]
\end{equation}
for some constant $C' > 0$. Let $c,C$ be the same constants as in Lemma \ref{lem:noisy_sequential_recovery}. Then, taking $c_*$ sufficiently small (and possibly decreasing $c$) yields 
\begin{equation}
    -\log p_{i|i}(t) \gtrsim \log K, \quad \text{whenever } \pi_it \leq (c/2)\log K.
\end{equation}
Lemma \ref{lem:noisy_sequential_recovery} then allows us to write
\begin{equation}
    (\log K) \sum_{i: \pi_it \leq (c/2)\log K} \pi_i \lesssim \Rcal_{\text{tail}}(t) \lesssim (\log K)\sum_{i:\pi_i t < C \log K} \pi_i.
\end{equation}
Define
\begin{equation}\label{eq:k_star_minus_plus_def}
    k_-^*(t) := \left(\frac{t}{CZ_{K,\alpha}\log K}\right)^{1/\alpha}, \quad k_+^*(t) := \left(\frac{t}{cZ_{K,\alpha}\log K}\right)^{1/\alpha}.
\end{equation}
Using this notation, we re-write the tail risk bounds as
\begin{equation}
    \frac{\log K}{Z_{K,\alpha}} \sum_{i \geq k_+^*(2t)} i^{-\alpha} \lesssim \Rcal_{\text{tail}}(t) \lesssim \frac{\log K}{Z_{K,\alpha}} \sum_{i > k_-^*(t)} i^{-\alpha}.
\end{equation}

\subsubsection{Head Risk}\label{appendix:gmm_head_risk}

If $t_i \leq t/2$, then
\begin{equation}
    \frac{1}{t-t_i + \pi_i^{-1}} \leq \frac{2}{t}.
\end{equation}
Moreover, there are at most $k_+^*(t/2)$ classes with recovery times $t_i \leq t/2$. Hence, the head risk contribution of classes recovered by time $t/2$ is at most of order $(1/t) \cdot k_+^*(t/2)$.

On the other hand, suppose that for a given index $i \in [K]$, the recovery time satisfies $t/2 \leq t_i \leq t$. Then,
\begin{equation}
    \frac{1}{t-t_i + \pi_i^{-1}} \leq \pi_i.
\end{equation}
By Lemma \ref{lem:noisy_sequential_recovery}, the condition on recovery time implies $k_-^*(t/2) \leq i \leq k_+^*(t)$. Therefore,
\begin{equation}
    \Rcal_{\text{head}}(t) \lesssim \frac{1}{t}\left(\frac{t}{Z_{K,\alpha} \log K}\right)^{1/\alpha} + \frac{1}{Z_{K,\alpha}} \sum_{k_-^*(t/2) \leq i \leq k_+^*(t)} i^{-\alpha}.
\end{equation}
The sum in the second term can be handled by an integral approximation
\begin{equation}
    \sum_{k_-^*(t/2) \leq i \leq k_+^*(t)} i^{-\alpha} \leq \left(k_-^*(t/2)\right)^{-\alpha} + \int_{k_-^*(t/2)}^{k_+^*(t)} x^{-\alpha} \dee x.
\end{equation}
For $\alpha \neq 1$, the integral is
\begin{equation}
    \begin{split}
        &\frac{1}{1-\alpha} \left[\left(\frac{t}{cZ_{K,\alpha} \log K}\right)^{1/\alpha-1} - \left(\frac{t}{2CZ_{K,\alpha} \log K}\right)^{1/\alpha-1}\right] \\
        &= \left(\frac{t}{Z_{K,\alpha} \log K}\right)^{1/\alpha -1} \frac{c^{1-1/\alpha} - (2C)^{1-1/\alpha}}{1-\alpha} \\
        &\asymp \left(\frac{t}{Z_{K,\alpha} \log K}\right)^{1/\alpha-1}.
    \end{split}
\end{equation}
For $\alpha = 1$, the integral is
\begin{equation}
    \log\left(\frac{t/(cZ_{K,1}\log K)}{t/(2CZ_{K,1}\log K)}\right) = \log \frac{2C}{c}.
\end{equation}
So, for all $\alpha > 0$, we have
\begin{equation}
    \frac{1}{Z_{K,\alpha}} \sum_{k_-^*(t/2) \leq i \leq k_+^*(t)} i^{-\alpha} \lesssim \frac{1}{Z_{K,\alpha}}\left(\frac{t}{Z_{K,\alpha}\log K}\right)^{1/\alpha-1}.
\end{equation}
Next, we lower bound the head error. Notice that if $i \leq k_-^*(t/2)$, then $t_i \leq t/2$ and $\pi_i^{-1} \lesssim t$. So 
\begin{equation}
    \frac{1}{t-t_i + \pi_i^{-1}} \gtrsim \frac{1}{t}
\end{equation}
and
\begin{equation}
    \Rcal_{\text{head}}(t) \gtrsim \frac{1}{t} \min\left\{K, \left(\frac{t}{2CZ_{K,\alpha}\log K}\right)^{1/\alpha}\right\}.
\end{equation}

\subsubsection{Decomposition into Risk Phases}

We are now ready to assess the aggregate risk.

\textbf{Early phase.} Suppose $0 \leq t \leq 2C Z_{K,\alpha} \log K$. Then, if $i \geq (4C/c)^{1/\alpha}$, we have $\pi_i t = i^{-\alpha} t / Z_{K,\alpha} = (c/2) \log K$. As a result,
\begin{equation}
    \Rcal_{\text{tail}}(t) \gtrsim (\log K) \sum_{i \geq (4C/c)^{1/\alpha}} \pi_i \asymp \log K.
\end{equation}
\eqref{eq:risk_logK_upper_bound} gives a matching upper bound, and therefore $\Rcal(t) \asymp \log K$ in this phase.

\textbf{Intermediate phase.} Suppose $2C Z_{K,\alpha} \log K \leq t \leq c_1 Z_{K,\alpha} K^{\alpha} \log K$, where $c_1 = c \cdot 2^{-(\alpha+2)}$. 

\underline{Case $\alpha > 1$.} The same integrals \eqref{eq:tail_risk_upper_bound} and \eqref{eq:tail_risk_lower_bound} from the idealized case give 
\begin{equation}
    \sum_{i \geq k_+^*(2t)} i^{-\alpha} \asymp \sum_{i \geq k_-^*(t)} i^{-\alpha} \asymp \left(\frac{t}{Z_{K,\alpha} \log K}\right)^{1/\alpha-1}.
\end{equation}
Hence,
\begin{equation}
    \Rcal_{\text{tail}}(t) \asymp \frac{\log K}{Z_{K,\alpha}} \left(\frac{t}{Z_{K,\alpha} \log K}\right)^{1/\alpha-1} \asymp (\log K)^{2-1/\alpha} \cdot t^{-(1-1/\alpha)}.
\end{equation}
Meanwhile,
\begin{equation}
    \Rcal_{\text{head}}(t) \lesssim \frac{1}{t} \min\left\{K, \left(\frac{t}{\log K}\right)^{1/\alpha}\right\} + \left(\frac{t}{\log K}\right)^{1/\alpha-1}.
\end{equation}
The first term is of order $\Rcal_{\text{tail}}(t)/(\log K)^2$, and the second term is of order $\Rcal_{\text{tail}}(t)/(\log K)$. Hence, $\Rcal(t) \asymp (\log K)^{2-1/\alpha} \cdot t^{-(1-1/\alpha)}$ in this regime.

\underline{Case $\alpha = 0$.} All recovery times $t_i$, $i \in [K]$ have the same order $K \log K$. There is thus no intermediate phase.

\underline{Case $0 < \alpha < 1$.} Recall that in this regime, we have $Z_{K,\alpha} \asymp K^{1-\alpha}$. We obtain
\begin{equation}
    \Rcal_{\text{tail}}(t) \asymp (\log K) \left(1-\left(\frac{t}{K\log K}\right)^{1/\alpha-1}\right),
\end{equation}
using the same integral bounds as in Appendix \ref{appendix:noiseless_case}. The choice of $c_1$ ensures that the expression in parentheses is bounded below by a positive constant. As for the head risk, the recently recovered classes contribute 
\begin{equation}
    \frac{1}{Z_{K,\alpha}} \left(\frac{t}{Z_{K,\alpha} \log K}\right)^{1/\alpha -1} \lesssim \frac{K^{1-\alpha}}{Z_{K,\alpha}} \lesssim 1.
\end{equation}
while the contribution of the other head classes is of order at most $1/\log K$, using a similar argument as in the $\alpha >1$ case. This confirms that the tail risk dominates in this regime as well.

\underline{Case $\alpha = 1$.} We know that $Z_{K,1} \asymp \log K$. Once again, arguing via integral approximation as in Appendix \ref{appendix:noiseless_case} yields
\begin{equation}
    \Rcal_{\text{tail}}(t) \asymp \log\left(\frac{K(\log K)^2}{t}\right).
\end{equation}
Meanwhile, $\Rcal_{\text{head}}(t) \lesssim 1/(\log K)$, and so the tail risk dominates.

\textbf{Convergence phase.} Suppose $2C Z_{K,\alpha} K^{\alpha} \log K \leq t \leq T$. Then,
\begin{equation}
    \frac{C \log K}{\pi_i} \leq C Z_{K,\alpha} K^{\alpha} \log K \leq \frac{t}{2},
\end{equation}
i.e., all classes have been recovered. Hence,
\begin{equation}
    \Rcal(t) = \Rcal_{\text{head}}(t) \asymp K/t.
\end{equation}

\subsection{Non-Orthogonal Means}\label{appendix:non_orthogonal}

In \eqref{eq:target}, we assume that the class means $\{\mubs_1,\dots,\mubs_K\}$ are orthogonal. This assumption ensures an exact decoupling of the dynamics for $(p_{i|i})_{i=1}^K$ in the noiseless case. In this subsection, we show that an approximate decoupling still holds under controlled violations of the orthogonality assumption (though means are still taken to lie on the unit sphere). As a result, we identify a regime with non-orthogonal means where sequential learning and power-law scaling --- as we have described it in the previous subsections --- manifests.

Define
\begin{equation}
    \rho_i := \sum_{j \neq i} \pi_j |\inner{\mubs_i, \mubs_j}|, \quad i \in [K].
\end{equation}
Under \eqref{eq:target}, all $\rho_i$ are zero. In the following, we show that even when then the $\rho_i$ are positive (but sufficiently small), a version of Theorem \ref{thm:mixture_model_scaling_law} still holds.

\begin{theorem}\label{thm:non_orthogonal_scaling_law}
    There exist sufficiently small constants $c_0, c_* > 0$ such that the following hold. Suppose
    \begin{equation}\label{eq:rho_i_sigma_conditions}
        T \max_{i \in [K]} \rho_i \leq c_0 \log K, \quad \sigma T \leq c_*.
    \end{equation}
    Moreover, there exists $c_1 > 0$ such that, for $t \leq T$,
    \begin{equation}
        \Rcal(t) \asymp 
        \begin{cases}
            \log K, & 0 \leq t \leq Z_{K,\alpha} \log K, \\
            f_{K,\alpha}(t), & Z_{K,\alpha} \log K \leq t \leq c_1 Z_{K,\alpha} K^{\alpha} \log K,
        \end{cases}
    \end{equation}
    where
    \begin{equation}
        f_{K,\alpha}(t) = 
        \begin{dcases}
            (\log K)\left(1-\left(\frac{t}{Z_{K,\alpha}K^{\alpha} \log K}\right)^{1/\alpha-1}\right), & 0 < \alpha < 1 \\
            \log\left(\frac{K(\log K)^2}{t}\right), & \alpha = 1\\
            (\log K)^{2-1/\alpha} \cdot t^{-(1-1/\alpha)}, & \alpha > 1. 
        \end{dcases}
    \end{equation}
\end{theorem}

The proof of Theorem \ref{thm:non_orthogonal_scaling_law} proceeds in several steps and leverages several intermediate lemmas that we derive in a similar fashion to those in Appendix \ref{appendix:noisy_case}.

Recall the weight dynamics \eqref{eq:full_unsketched_w_dynamics}, which are derived without assuming orthogonality:
\begin{equation}
    \dot{\wbs}_i = \sum_{k=1}^K \pi_k (\1\{i=k\} - P_{i|k})\mubs_k - \sum_{k=1}^K \pi_k \E_{\xibs_k}[p_i(\mubs_k + \xibs_k)\xibs_k].
\end{equation}

The proof of the weight norm bound in Lemma \ref{lem:weight_norm_bound} does not rely on orthogonality, and therefore we can leverage the fact $M(t) \leq t$ for all $t \geq 0$. For the same reason, Lemmas \ref{lem:softmax_perturbation} and \ref{lem:log_p_inequality} hold in this setting as well. 

The change is to the class-wise learning dynamics.
\begin{lemma}\label{lem:class_wise_dynamics_without_orthogonality}
    There exist universal constants $a,b > 0$ such that, for every $i \in [K]$ and $t \leq T$,
    \begin{equation}
        \begin{split}
            \dot{p}_{i|i} &\geq p_{i|i}(1-p_{i|i})\left[a\pi_i (1-p_{i|i}) - 2\rho_i - \sigma^2 M(t)\right] \\
            \dot{p}_{i|i} &\leq p_{i|i} (1-p_{i|i}) \left[b\pi_i(1-p_{i|i}) + 2\rho_i + \sigma^2 M(t)\right].
        \end{split}
    \end{equation}
\end{lemma}
\begin{proof}
    From the expression 
    \begin{equation}
        p_{i|i} = \frac{e^{\theta_{ii}}}{\sum_{k=1}^K e^{\theta_{ki}}},
    \end{equation}
    we have
    \begin{equation}
        \dot{p}_{i|i} = p_{i|i} \sum_{k=1}^K \left(\1\{k=i\} - p_{k|i}\right) \dot{\theta}_{ki}.
    \end{equation}
    Starting from \eqref{eq:full_unsketched_w_dynamics}, the summary statistic dynamics are
    \begin{equation}
        \dot{\theta}_{ki} = \sum_{j=1}^K \pi_j (\1\{k=j\} - P_{k|j})\inner{\mubs_i, \mubs_j} - \sum_{\ell=1}^K \pi_{\ell}\inner{\E_{\xibs_{\ell}}[p_k(\mubs_{\ell} + \xibs_{\ell})\xibs_{\ell}], \mubs_i}.
    \end{equation}
    Comparing this to \eqref{eq:full_theta_dynamics}, the additional contribution due to non-orthogonality in the ODE for $p_{i|i}$ is
    \begin{equation}\label{eq:non_orthogonality_contribution_to_p_ii}
        p_{i|i} \sum_{j \neq i} \pi_j \inner{\mubs_i, \mubs_j} \sum_{k=1}^K (\1\{k=i\} - p_{k|i}) (\1\{k=j\} - P_{k|j}).
    \end{equation}
    Since $|\1\{k=j\} - P_{k|j}| \leq 1$ and $\sum_{k=1}^K |\1\{k=i\} - p_{k|i}| = 2(1-p_{i|i})$, \eqref{eq:non_orthogonality_contribution_to_p_ii} is bounded in absolute value by $2\rho_i p_{i|i}(1-p_{i|i})$. The rest of the proof then proceeds in the same way as in the proof of Lemma \ref{lem:noisy_p_ii_bounds}.
\end{proof}

\begin{lemma}[Sequential Recovery, Non-Orthogonal Case]\label{lem:noisy_sequential_recovery_without_orthogonality}
    Assume that \eqref{eq:rho_i_sigma_conditions} holds. There exist constants $0 < c < C$ such that, for every $i \in [K]$,
    \begin{enumerate}
        \item If $t_i \leq T$, then $t_i \geq c(\log K)/\pi_i$; 
        \item If $C(\log K)/\pi_i \leq T$, then $t_i \leq C(\log K)/\pi_i$.
    \end{enumerate}
\end{lemma}
\begin{proof}
    The proof is analogous to that of Lemma \ref{lem:noisy_sequential_recovery}.

    Suppose $t < t_i$. Then, $1-p_{i|i} \in [1/2,1]$. From Lemma \ref{lem:class_wise_dynamics_without_orthogonality} and the fact $M(t) \leq t$, there exist constants $a, b > 0$ (which do not depend on $i$) such that
    \begin{equation}\label{eq:log_p_ii_dynamics_without_orthogonality}
        a\pi_i - 2\rho_i - \sigma^2 t\leq \frac{\dee}{\dee t} \log p_{i|i}(t) \leq b \pi_i + 2 \rho_i + \sigma^2 t.
    \end{equation}
    Integrating from zero and using $p_{i|i}(0) = 1/K$ gives, for some constant $C > 0$
    \begin{equation}
        -\log K + a \pi_i t -2\rho_i t - \frac{\sigma^2 t^2}{2} \leq \log p_{i|i}(t) \leq - \log K + b \pi_i t + 2\rho_i t + \frac{\sigma^2 t^2}{2}.
    \end{equation}
    Under the condition \eqref{eq:rho_i_sigma_conditions}, if $t_i \leq T$, the above upper bound shows that $p_{i|i}$ cannot reach $1/2$ prior to time being a constant multiple of $(\log K)/\pi_i$. This proves the first point.
    
    For a sufficiently large constant multiple of $(\log K)/\pi_i$, the lower bound exceeds $-\log 2$, proving the second point.
\end{proof}

\begin{lemma}[Post Recovery Decay, Non-Orthogonal Case]\label{lem:post_recovery_decay_without_orthogonality}
    If $t_i \leq T$, then $p_{i|i}(t) \geq 1/2$ for all $t \in [t_i, T]$. Moreover,
    \begin{equation}
        1 - p_{i|i}(t) \lesssim \frac{1}{1+\pi_i (t-t_i)} + \frac{\rho_i}{\pi_i}
    \end{equation}
    and
    \begin{equation}
        -\pi_i \log p_{i|i}(t) \lesssim \frac{1}{t-t_i + \pi_i^{-1}} + \rho_i.
    \end{equation}
\end{lemma}
\begin{proof}
    We proceed in similar fashion to the proof of Lemma \ref{lem:post_recovery_decay}.

    From the sequential recovery result, we $t_i \leq T$ implies $\pi_i \geq c(\log K)/T$. Moreover, for $c_0 > 0$ sufficiently small in \eqref{eq:rho_i_sigma_conditions}, we have $\rho_i \leq \frac{c}{2}(\log K)/T$. On the other hand, \eqref{eq:time_horizon_bounds} gives $\sigma^2 M(t) \leq c_*\sigma$ for $t \leq T$. Hence, $c_*$ can be taken sufficiently small so that
    \begin{equation}
        \frac{\dee}{\dee t} \log p_{i|i}(t) \geq a_0 \pi_i - 2\rho_i - \sigma^2 M(t) > 0
    \end{equation}
    when $p_{i|i}(t) = 1/2$. This proves that $p_{i|i}(t) \in [1/2,1]$ for $t \in [t_i,T]$. Combining this with Lemma \ref{lem:class_wise_dynamics_without_orthogonality}, we have 
    \begin{equation}
        \frac{a}{2}\pi_i - \frac{2\rho_i + \sigma^2 M(t)}{1-p_{i|i}(t)} \leq \frac{\dee}{\dee t} \left[\frac{1}{1-p_{i|i}(t)}\right] \leq b\pi_i + \frac{2\rho_i + \sigma^2 M(t)}{1-p_{i|i}(t)}.
    \end{equation}
    We multiply by the integrating factor $\exp(\int_{t_i}^t 2\rho_i +  \sigma^2 M(r) \dee r)$ to obtain
    \begin{equation}
        \frac{\dee}{\dee t}\left[\exp\left(2\rho_i(t-t_i) + \sigma^2 \int_{t_i}^t M(r) \dee r\right) \frac{1}{1-p_{i|i}(t)}\right] \geq \frac{a}{2}\pi_i \exp\left(2\rho_i (t-t_i) + \sigma^2 \int_{t_i}^t M(r) \dee r\right).
    \end{equation}
    Integrating from $t_i$ to $t$ and using $p_{i|i}(t_i) = 1/2$ gives
    \begin{equation}
        \begin{split}
            \frac{1}{1-p_{i|i}(t)} &\geq \exp\left(-2\rho_i (t-t_i) - \sigma^2 \int_{t_i}^t M(r) \dee r\right) \\
            &\quad \cdot \left[2 + \frac{a}{2}\pi_i \int_{t_i}^t \exp\left(2\rho_i (s-t_i) + \sigma^2 \int_{t_i}^s M(r) \dee r\right) \dee s\right] \\
            &\geq 2e^{-c_*^2/2} e^{-2\rho_i(t-t_i)} + \frac{a}{2} \pi_i e^{-c_*^2/2} \int_{t_i}^t e^{2\rho_i (s-t)} \dee s \\
            &\geq e^{-c_*^2/2} \left[2e^{-2\rho_i(t-t_i)} + \frac{a\pi_i}{4\rho_i} (1-e^{-2\rho_i(t-t_i)})\right],
        \end{split}
    \end{equation}
    where we have used $\sigma^2 \int_{t_i}^T M(r) \dee r \leq \sigma^2 \int_0^T M(r) \dee r \leq c_*^2/2$. Rearranging, we have
    \begin{equation}
        \begin{split}
            1-p_{i|i}(t) &\leq \frac{e^{c_*^2/2}}{2e^{-2\rho_i(t-t_i)} + \frac{a\pi_i}{4\rho_i} (1-e^{-2\rho_i(t-t_i)})} \\
            &= e^{c_*^2/2} \left[\frac{1}{2+\frac{a\pi_i}{4\rho_i}(e^{2\rho_i(t-t_i)}-1)} + \frac{e^{2\rho_i(t-t_i)}-1}{2+\frac{a\pi_i}{4\rho_i}(e^{2\rho_i(t-t_i)}-1)}\right] \\
            &\lesssim \frac{1}{1+\pi_i(t-t_i)} + \frac{\rho_i}{\pi_i},
        \end{split}
    \end{equation}
    where we have used $e^{2\rho_i(t-t_i)} - 1 \geq 2\rho_i(t-t_i)$. Since $p_{i|i}(t) \geq 1/2$, \eqref{eq:log_p_inequality} gives
    \begin{equation}\label{eq:head_class_risk}
        -\pi_i \log p_{i|i}(t) \asymp \pi_i (1-p_{i|i}(t)) \lesssim \frac{1}{t-t_i + \pi_i^{-1}} + \rho_i, \quad t_i \leq t \leq T.
    \end{equation}
\end{proof}

The pieces are now in place to prove Theorem \ref{thm:non_orthogonal_scaling_law}.

\begin{proof}[Proof of Theorem \ref{thm:non_orthogonal_scaling_law}]

From \eqref{eq:log_p_ii_dynamics_without_orthogonality}, we have
\begin{equation}
    -2\rho_i - \sigma^2 t \leq \frac{\dee}{\dee t} \log p_{i|i}(t) \leq b \pi_i + 2\rho_i + \sigma^2 t.
\end{equation}
Consequently, if $\pi_i t \leq c \log K$,
\begin{align}
    -\log p_{i|i}(t) &\leq \log K + 2\rho_i t + \sigma^2 t^2/2 \lesssim \log K, \\
    -\log p_{i|i}(t) &\geq \log K - b\pi_i t - 2\rho_i t - \sigma^2 t^2/2 \gtrsim \log K,
\end{align}
by \eqref{eq:rho_i_sigma_conditions}. In fact, by Lemma \ref{lem:log_p_inequality}, the same inequalities hold for $\E_{\xbs|y=i}[-\log p_i(\xbs)]$. Hence, if $t \leq c_1 Z_{K,\alpha} \log K$ for a sufficiently small $c_1 > 0$, we have $\Rcal(t) \asymp \log K$. 

For the intermediate phase, we take a head-tail decomposition as in the previous subsections. As in \eqref{eq:k_star_minus_plus_def}, define the thresholds
\begin{equation}
    k_-^*(t) := \left(\frac{t}{CZ_{K,\alpha}\log K}\right)^{1/\alpha}, \quad k_+^*(t) := \left(\frac{t}{cZ_{K,\alpha}\log K}\right)^{1/\alpha}.
\end{equation}
We have the same tail risk bound
\begin{equation}
    \frac{\log K}{Z_{K,\alpha}} \sum_{i \geq k_+^*(t)} i^{-\alpha} \lesssim \Rcal_{\mathrm{tail}}(t) \lesssim \frac{\log K}{Z_{K,\alpha}} \sum_{i \geq k_-^*(t/2)} i^{-\alpha},
\end{equation}
which is of order $f_{K,\alpha}$. 

Indices $i \leq k_-^*(t/2)$ satisfy $t_i \leq t/2$, and therefore their contribution to the risk is at most of order $1/t + \rho_i$ by Lemma \ref{lem:post_recovery_decay_without_orthogonality}. Every remaining recovered class satisfies $k_-^*(t/2) < i \leq k_+^*(t)$. Since $p_{i|i}(t) \geq 1/2$, Lemma \ref{lem:log_p_inequality} gives
\begin{equation}
    \pi_i \E_{\xbs|y=i}[-\log p_i(\xbs)] \leq e^{2c_*^2} \pi_i \log 2 \lesssim \pi_i.
\end{equation}
Hence,
\begin{equation}
    \Rcal_{\mathrm{head}}(t) \lesssim \left(\frac{1}{t} + \max_{i \in [K]} \rho_i\right) \left(\frac{t}{Z_{K,\alpha}\log K}\right)^{1/\alpha} + \frac{1}{Z_{K,\alpha}} \sum_{k_-^*(t/2) \leq i \leq k_+^*(t)} i^{-\alpha}.
\end{equation}
Hence, the only difference from Appendix \ref{appendix:noisy_case} is the additional head risk contribution of order at most
\begin{equation}
    \left(\max_{i \in [K]} \rho_i\right) \left(\frac{t}{Z_{K,\alpha} \log K}\right)^{1/\alpha} \lesssim \left(\frac{\log K}{T}\right)\left(\frac{t}{Z_{K,\alpha} \log K}\right)^{1/\alpha} \leq (\log K)^{1-1/\alpha} Z_{K,\alpha}^{-1/\alpha} t^{-(1-1/\alpha)},
\end{equation}
which is dominated by $f_{K,\alpha}(t)$.
\end{proof}

We conclude this section by establishing two useful corollaries to Theorem \ref{thm:non_orthogonal_scaling_law}.

\begin{corollary}\label{cor:first_m_means_recovery_guarantee}
    Let $m \in [K]$. There exists $\eps > 0$ such that if $\max_{i \in [K]} \rho_i \leq \eps \pi_m$, then Theorem \ref{thm:non_orthogonal_scaling_law} applies with
    \begin{equation}
        T = \min\left\{\frac{C \log K}{\pi_m}, \frac{c_*}{\sigma}\right\},
    \end{equation}
    where $C$ is as in Lemma \ref{lem:noisy_sequential_recovery_without_orthogonality}. Moreover, all of the top $m$ classes are recovered, i.e., $t_i \leq T$ for $i \in [m]$, if
    \begin{equation}\label{eq:sigma_non_orthogonal_condition}
        \sigma \leq \frac{c_* \pi_m}{C\log K}.
    \end{equation}
\end{corollary}
\begin{proof}
    Choose $0 < \eps \leq \min\{1/2, c_0/C\}$. By the definition of $T$,
    \begin{equation}
        T \max_i \rho_i \leq \frac{C\log K}{\pi_m} \eps \pi_m \leq c_0 \log K, \quad \sigma T \leq c_*.
    \end{equation}
    Thus Theorem \ref{thm:non_orthogonal_scaling_law} applies. If \eqref{eq:sigma_non_orthogonal_condition} holds, then $c_*/\sigma \geq C (\log K)/\pi_m$, so $T = C(\log K)/\pi_m$. Hence, for every $i \in [m]$, $C(\log K)/\pi_i \leq T$. Lemma \ref{lem:noisy_sequential_recovery_without_orthogonality} therefore gives
    \begin{equation}
        t_i \leq \frac{C \log K}{\pi_i} \leq T.
    \end{equation}
\end{proof}

\begin{corollary}
    Suppose that $\mubs_1,\dots,\mubs_K \overset{\mathrm{iid}}{\sim} \mathrm{Unif}(\Sbb^{d-1})$. For any $\delta \in (0,1)$, the condition in Corollary \ref{cor:first_m_means_recovery_guarantee} holds with probability at least $1-\delta$ whenever
    \begin{equation}
        d \gtrsim \frac{\log(K/\delta)}{\pi_m^2} = Z_{K,\alpha}^2 m^{2\alpha} \log (K/\delta).
    \end{equation}
\end{corollary}
\begin{proof}
    The uniform distribution over $\Sbb^{d-1}$ is sub-Gaussian with variance proxy $1/d$ \cite[Theorem 3.4.6]{vershynin2018high}. Hence, for fixed indices $i \neq j$ and $u > 0$,
    \begin{equation}
        \Prob\left(|\inner{\mubs_i, \mubs_j}| > u\right) \leq 2e^{-cdu^2}
    \end{equation}
    for some $c > 0$. By a union bound,
    \begin{equation}
        \Prob\left(\max_{i < j} |\inner{\mubs_i, \mubs_j}| > u\right) \leq K(K-1) e^{-cdu^2}.
    \end{equation}
    We set the RHS to be at most $\delta$ and take $u = \eps \pi_m$. Then, solving for $d$ gives
    \begin{equation}
        d \geq \frac{1}{c(\eps \pi_m)^2} \log\left(\frac{K(K-1)}{\delta}\right).
    \end{equation}
\end{proof}

\section{Proofs for Sketched Population Gradient Flow}\label{appendix:proofs_sketching}
This section builds up to and contains the proof of the main result in Section \ref{section:compute_optimal_sl}. We focus solely on the case $\alpha > 1$, where the risk curve will obey power laws in both the sketch dimension and training time. The logic flow is similar to that of Appendix \ref{appendix:proofs_pop_grad_flow}. We adapt the notation in this appendix to account for the sketching matrix $\Sbs \in \R^{m \times d}$. Namely,
\begin{equation}
    \theta_{ij} := \inner{\wbs_i, \Sbs \mubs_j}, \quad p_{i|j} := p_i(\Sbs \mubs_j), \quad P_{i|j} := \E_{\xbs|y=j}[p_i(\Sbs \xbs)].
\end{equation}
The risk takes the form
\begin{equation}
    \Rcal(\Wbs) = -\sum_{i=1}^K \pi_i \E_{\xbs|y=i}[\log p_i(\Sbs \xbs)].
\end{equation}

\subsection{Noiseless Case}\label{appendix:gf_pca_noiseless}

In the noiseless case, we can separate the dynamics of the top-$m$ classes, which remain as they were in the unsketched case, from the dynamics of the bottom $K-m$ classes, which are mapped to zero and are thus indistinguishable.

Let the sketch matrix $\Sbs \in \R^{m \times d}$ be the projection of $\R^d$ onto $\mathrm{span}\{\mubs_1,\dots,\mubs_m\}$. The dynamics for the weights and summary statistics respectively take the form
\begin{equation}
    \dot{\wbs}_k = \pi_k \Sbs \mubs_k - \sum_{j=1}^K \pi_j p_{k|j} \Sbs \mubs_j,
\end{equation}
\begin{equation}
    \dot{\theta}_{ki} = \pi_k \inner{\Sbs \mubs_k, \Sbs \mubs_i} - \sum_{j=1}^K \pi_j p_{k|j} \inner{\Sbs \mubs_j, \Sbs \mubs_i}.
\end{equation}
As a result, for $k \neq i$,
\begin{equation}
    \dot{\theta}_{ii} =
    \begin{cases}
        \pi_i(1-p_{i|i}), & i \leq m, \\
        0, &i> m,
    \end{cases}
    \quad\quad
    \dot{\theta}_{ki} = 
    \begin{cases}
        -\pi_i p_{k|i}, & i \leq m \\
        0, & i > m.
    \end{cases}
\end{equation}
From the above, we see that the same conservation law from Lemma \ref{lem:conservation_law} continues to hold in the sketched case. That is, $\sum_{k=1}^K \dot{\theta}_{ki} = 0$ and the dynamics of $\theta_{ki}$, for all $k \neq i$ are identical. Hence, $\theta_{ki} = -\frac{1}{K-1}\theta_{ii}$ and $p_{k|i} = \frac{1}{K-1}(1-p_{i|i})$ for all $k \neq i$.

As a consequence, for $i \leq m$, the softmax probability for class $i$ follows
\begin{equation}
    \dot{p}_{i|i} = p_{i|i} \left(\dot{\theta}_{ii} - \sum_{k=1}^K p_{k|i} \dot{\theta}_{ki}\right) = \frac{K}{K-1}\pi_i p_{i|i}(1-p_{i|i})^2,
\end{equation}
which is identical to the class-wise dynamics in Appendix \ref{appendix:noiseless_case}. In particular, this implies that the recovery times $t_1,\dots,t_m$ are as in \eqref{eq:hitting_time_1/2}.

On the other hand, for $i > m$, $p_{k|i}(t) = 1/K$ for all $k \in [K]$ and $t \geq 0$ since $\Sbs \mubs_i = 0$ and thus all inner products in the softmax are zero. Hence, we have an approximation error $\Rcal_{\text{approx}}$ that stays constant over time for fixed $m$.

Hence,
\begin{equation}
    \Rcal(t) = \underbrace{- \sum_{i=1}^m \pi_i \log p_{i|i}(t)}_{\Rcal_{\text{opt}}(t)} + \underbrace{\sum_{i=m+1}^K \pi_i \log K}_{\Rcal_{\text{approx}}}.
\end{equation}
$\Rcal_{\text{opt}}(t)$ can be decomposed into a head and tail risk using the recovery times $t_1,\dots,t_m$, just as in Appendix \ref{appendix:noiseless_case}. The same analysis as in the latter gives
\begin{equation}
    \Rcal_{\text{opt}}(t) \asymp 
    \begin{cases}
        \log K, & 0 \leq t \leq t_1 \\
        (\log K)^{2-1/\alpha} \cdot t^{-(1-1/\alpha)}, & t_1 \leq t \leq c_1t_m \\
        m/t, & (1+c_2)t_m \leq t
    \end{cases}
\end{equation}
for constants $c_1, c_2 \in (0,1)$.

If $m \leq \kappa K$ for some fixed $\kappa \in (0,1)$, i.e., at most a constant proportion of classes are kept by the sketch, then the approximation error takes the form
\begin{equation}
    \Rcal_{\text{approx}}(t) \asymp m^{1-\alpha} \log K,
\end{equation}
which follows from an integral approximation. 

Note that $t_m \asymp m^{\alpha} \log K$ and therefore $\Rcal_{\text{opt}}(c_1t_m) \asymp m^{1-\alpha} \log K$. Since $\Rcal_{\text{opt}}$ is decreasing along population gradient flow, we can conclude 
\begin{equation}
    \Rcal(t) \asymp 
    \begin{cases}
        \log K, & 0 \leq t \leq t_1 \\
        m^{1-\alpha} \log K + (\log K)^{2-1/\alpha} \cdot t^{-(1-1/\alpha)}, & t_1 \leq t.
    \end{cases}
\end{equation}

\subsection{Noisy Case}\label{appendix:gf_pca_noisy}

We impose additional assumptions on the noisy case. In addition to $\sigma$-sub-Gaussianity \eqref{eq:sub_gaussianity}, we also assume that for each covariance matrix $\Sigma_i = \E[\xibs_i \xibs_i^{\top}]$, we have
\begin{equation}
    \sum_{i=1}^K \pi_i \Sigmabs_i = \tau^2 \Ibs_d
\end{equation}
for some $0 \leq \tau \leq \sigma$. That is, $\tau$ is the average covariance scale, whereas $\sigma$ controls the classwise noise tails. Then,
\begin{equation}
    \Gbs = \E[\xbs \xbs^{\top}] = \sum_{i=1}^K \pi_i \mubs_i \mubs_i^{\top} + \tau^2 \Ibs_d.
\end{equation}
As a result, we have the same sketching matrix $\Sbs = \Mbs_m^{\top} \in \R^{m \times d}$ as in the previous section.
\begin{remark}
    This condition does not require the class-wise covariance matrices $\Sigmabs_i$ to be multiples of the identity, nor does it require Gaussianity.
\end{remark}

Note that $\Sbs \Sbs^{\top} = \Ibs_m$. Adapting the dynamics from \eqref{eq:full_theta_dynamics} to the sketched case yields the evolution
\begin{equation}\label{eq:full_sketched_w_dynamics}
    \dot{\wbs}_i = \sum_{k=1}^K \pi_k (\1\{i=k\}-P_{i|k})\Sbs\mubs_k - \sum_{k=1}^K \pi_k \E_{\xibs_k}[p_i(\Sbs \mubs_k + \Sbs\xibs_k)\Sbs\xibs_k]
\end{equation}
The summary statistics $\theta_{ij} = \inner{\wbs_i, \Sbs \mubs_j}$ evolve as 
\begin{equation}\label{eq:theta_ij_dynamics_pop_pca}
    \dot{\theta}_{ij} = \sum_{k=1}^K \pi_k (\1\{i=k\} - P_{i|k})\inner{\Sbs \mubs_j, \Sbs \mubs_k} - \sum_{k=1}^K \pi_k \inner{\E_{\xibs_k}[p_i(\Sbs\mubs_k + \Sbs\xibs_k)\Sbs\xibs_k], \Sbs\mubs_j}.
\end{equation}
We split these based on the index $i$. For $i \leq m$, we have
\begin{equation}
    \dot{\theta}_{ii} = \pi_i(1-P_{i|i}) - \sum_{k=1}^K \pi_k \inner{\E_{\xibs_k}[p_i(\Sbs\mubs_k + \Sbs\xibs_k)\Sbs\xibs_k], \Sbs\mubs_i},
\end{equation}
and
\begin{equation}
    \dot{\theta}_{ki} = -\pi_i P_{k|i} - \sum_{\ell=1}^K \pi_{\ell} \inner{\E_{\xibs_{\ell}}[p_k(\Sbs\mubs_{\ell} + \Sbs\xibs_{\ell})\Sbs\xibs_{\ell}], \Sbs\mubs_i}, \quad k \neq i.
\end{equation}
The same perturbative analysis as in the unsketched case of Appendix \ref{appendix:noisy_case} carries over to the current sketched case, with some modifications. In particular, note that $\sigma$-sub-Gaussianity still holds:
\begin{equation}
    \E[e^{\inner{u,\Sbs \xibs_j}}] \leq e^{\sigma^2 ||\Sbs^{\top}u||^2/2} \leq e^{\sigma^2 ||u||^2 /2}, \quad u \in \R^m.
\end{equation}

Comparing \eqref{eq:full_sketched_w_dynamics} to the unsketched weight dynamics \eqref{eq:full_unsketched_w_dynamics}, the second term only changes in that $p_i$ and $p_{\ell}$ take $\Sbs \xbs$ as argument rather than just $\xbs$. This does not affect the size of the second term or the weight norm bounds that are central to our earlier perturbative analysis. Moreover, the size of the first term does not increase since $||\Sbs \mubs_k|| \leq 1$ for all $k \in [K]$. Hence, Lemma \ref{lem:weight_norm_bound} holds for the sketched dynamics. In particular, we inherit the same time horizon $T := c_*/\sigma$.

Since $\sigma$-sub-Gaussianity still holds, Lemma \ref{lem:softmax_perturbation} holds with the sketched definition of $p_{i|j}$ and $P_{i|j}$. Meanwhile, Lemma \ref{lem:log_p_inequality} continues to hold for the expectation over the sketched conditional distribution.

Lemma \ref{lem:noisy_p_ii_bounds} carries over as-is for $i \leq m$, with the notation $p, P$ now referring to probabilities under sketching. Lemmas \ref{lem:noisy_sequential_recovery} and \ref{lem:post_recovery_decay} follow from this by the exact same argument as in Appendix \ref{appendix:noisy_case}, once again only for $i \leq m$. 

Hence, if we continue to define
\begin{equation}
    \Rcal_{\text{opt}}(t) := -\sum_{i=1}^m \pi_i \E_{\xbs|y=i}[\log p_i(\Sbs \xbs)], \quad \Rcal_{\text{approx}}(t) := -\sum_{i=m+1}^K \pi_i \E_{\xbs|y=i}[\log p_i(\Sbs \xbs)],
\end{equation}
then, the usual head-tail decomposition for $\Rcal_{\text{opt}}$ goes through and there exist $c_1, c_2 > 0$ such that, for $t \in [0,T]$,
\begin{equation}
    \Rcal_{\text{opt}}(t) \asymp
    \begin{cases}
        \log K, & 0 \leq t \lesssim \log K \\
        (\log K)^{2-1/\alpha} \cdot t^{-(1-1/\alpha)}, & \log K \lesssim t \leq c_1 m^{\alpha}\log K \\
        m/t, & c_2 m^{\alpha} \log K \leq t. \\
    \end{cases}
\end{equation}

Meanwhile, for $i > m$, $\dot{\theta}_{ki} = 0$ for all $k \in [K]$ since $\Sbs\mubs_i = \zerobs$. Hence, $p_{i|i} = 1/K$. It remains to bound the loss on noisy perturbations around zero. 

\begin{lemma}[Risk of a Collapsed Class]\label{lem:collapsed_class_risk}
    There exists $C > 0$ such that for every $i > m$ and $0 \leq t \leq T$, 
    \begin{equation}
        \log K \leq \E_{\xbs|y=i}[-\log p_i(\Sbs \xbs)] \leq C \log K.
    \end{equation}
    Consequently if $m \leq \kappa K$ for some $\kappa \in (0,1)$,
    \begin{equation}
        \Rcal_{\text{approx}}(t) \asymp m^{1-\alpha} \log K.
    \end{equation}
\end{lemma}
\begin{proof}
    The lower bound is immediate from Lemma \ref{lem:log_p_inequality}, which we confirmed applies to the sketched case. Moreover, for the upper bound, Lemma \ref{lem:log_p_inequality} gives
    \begin{equation}
        \E_{\xbs|y=i}[-\log p_i(\Sbs \xbs)] \leq e^{2\sigma^2 M(t)^2} (\log K).
    \end{equation}
    The prefactor $e^{2\sigma^2 M(t)^2}$ is uniformly bounded by a constant over the time horizon $[0,T]$ by \eqref{eq:time_horizon_bounds}.
\end{proof}

By the same argument as in Appendix \ref{appendix:gf_pca_noiseless}, the $\Rcal_{\text{approx}}$ is at least of the same order as $\Rcal_{\text{opt}}$ by the time $c_1 m^{\alpha} \log K$. We can therefore conclude that, for $t \in [0,T]$,
\begin{equation}
    \Rcal(t) \asymp 
    \begin{cases}
        \log K, & 0 \leq t \lesssim \log K, \\
        m^{1-\alpha} \log K + (\log K)^{2-1/\alpha} \cdot t^{-(1-1/\alpha)}, & \log K \lesssim t.
    \end{cases}
\end{equation}

\subsection{Empirical PCA}\label{appendix:gf_pca_empirical}

Now, we move to the case of a random sketching matrix that approximates the perfect identification of the top-$m$ class means that occurs at the population level. Note that data distributed as \eqref{eq:target} that satisfies the average covariance condition \eqref{eq:average_covariance_condition} has
\begin{equation}
    \Gbs :=\E_{\xbs}[\xbs \xbs^{\top}] = \sum_{i=1}^K \pi_i \E_{\xibs_i}\big[(\mubs_i + \xibs_i)(\mubs_i + \xibs_i)^{\top}\big] = \sum_{i=1}^K \pi_i \mubs_i \mubs_i^{\top} + \tau^2 \Ibs_d.
\end{equation}
The top-$m$ eigenvectors of this matrix correspond exactly to the top $m$ class means $\mu_1,\dots,\mu_m$, with associated eigenvalues $\lambda_1 := \pi_1+\tau^2,\dots,\lambda_m := \pi_m + \tau^2$. Hence, the discussion in the previous subsection arises naturally from population-level PCA. Now, consider the sample second moment matrix
\begin{equation}
    \widehat{\Gbs} = \frac{1}{N} \sum_{i=1}^N \xbs_i \xbs_i^{\top} \in \R^{d \times d}
\end{equation}
given i.i.d.\@ unlabelled data $\xbs_1,\dots,\xbs_N$ from the target \eqref{eq:target}. This yields eigenvalues $\widehat{\lambda}_1 \geq \dots \geq \widehat{\lambda}_d$ and eigenvectors $\widehat{\vbs}_1,\dots,\widehat{\vbs}_d \in \R^d$. The resulting sketch matrix $\widehat{\Sbs} \in \R^{m \times d}$ is such that $\widehat{\Sbs} \xbs = (\inner{\xbs,\widehat{\vbs}_1},\dots,\inner{\xbs,\widehat{\vbs}_m})$ for all $x \in \R^d$. 

Let $\Pbs_m = \Mbs_m \Mbs_m^{\top} \in \R^{d \times d}$ denote the population top-$m$ projector, i.e., the projection map onto the top-$m$ class means, and $\widehat{\Pbs}_m = \widehat{\Sbs}^{\top} \widehat{\Sbs} \in \R^{d \times d}$ denote the empirical top-$m$ projector. The key quantity that will be used in bounding the gap between the dynamics under population and empirical PCA is
\begin{equation}\label{eq:eps_pca_def}
    \eps_{\text{PCA}} := \norm{\widehat{\Pbs}_m - \Pbs_m}_{\text{op}}.
\end{equation}
In particular, we will leverage the inequality
\begin{equation}\label{eq:empirical_pca_deivation}
    \absolute{\inner{\hat{\Sbs}\mubs_i, \hat{\Sbs}\mubs_j} - \1\{i=j \leq m\}} \leq \eps_{\text{PCA}}, \quad \forall \, i,j \in [K].
\end{equation}
Moreover, if $j > m$, then $\Pbs_m \mubs_j = \zerobs$ and
\begin{equation}\label{eq:bottom_class_deviation}
    \norm{\widehat{\Sbs}\mubs_j} = \norm{\widehat{\Pbs}_m \mubs_j} = \norm{(\widehat{\Pbs}_m - \Pbs_m)\mubs_j} \leq \eps_{\text{PCA}}.
\end{equation}
Abusing the notation from the previous subsection, we track
\begin{equation}
    \theta_{ij} = \inner{\wbs_i, \widehat{\Sbs} \mubs_j}, \quad p_{i|j} := p_i(\widehat{\Sbs}\mubs_j), \quad P_{i|j} := \E_{\xbs|y=j}[p_i(\widehat{\Sbs}\xbs)].
\end{equation}
Note that Lemmas \ref{lem:weight_norm_bound}-\ref{lem:log_p_inequality} hold under the empirical sketch since $\norm{\widehat{\Sbs}\mubs_k} \leq 1$ for all $k \in [K]$ and $\widehat{\Sbs} \widehat{\Sbs}^{\top} = \Ibs_m$. Thus, the same proofs hold by replacing $d$ with $m$. 

As for the dynamics, we must be careful and account for the additional error due to replacing the population sketch by the empirical one. 
\begin{lemma}[Stability under Empirical Sketch]\label{lem:empirical_pca_p_ii_dynamics}
    There exist constants $a,b,C > 0$ such that, for all $i \leq m$ and $t \in [0,T]$,
    \begin{equation}\label{eq:empirical_pca_p_ii_dynamics}
        \begin{split}
            \dot{p}_{i|i} &\geq p_{i|i}(1-p_{i|i}) \left(a \pi_i (1-p_{i|i}) - C\sigma^2 M(t) - C\eps_{\text{PCA}}\right) \\
            \dot{p}_{i|i} &\leq p_{i|i}(1-p_{i|i})\left(b\pi_i (1-p_{i|i}) + C\sigma^2 M(t) + C\eps_{\text{PCA}}\right).
        \end{split}
    \end{equation}
\end{lemma}
\begin{proof}
    The summary statistics follow the same dynamics as in \eqref{eq:theta_ij_dynamics_pop_pca}, where we replace the population sketch matrix $\Sbs$ by the empirical sketch matrix $\widehat{\Sbs}$:
    \begin{equation}\label{eq:summary_statistic_dynamics_emp_pca}
        \dot{\theta}_{ki} = \sum_{j=1}^K \pi_j (\1\{k=j\} - P_{k|j}) \inner{\widehat{\Sbs}\mubs_i, \widehat{\Sbs}\mubs_j} - \sum_{\ell=1}^K \pi_{\ell} \inner{\E_{\xibs_{\ell}}[p_k(\widehat{\Sbs}\mubs_{\ell} + \widehat{\Sbs}\xibs_{\ell})\widehat{\Sbs}\xibs_{\ell}], \widehat{\Sbs}\mubs_i}.
    \end{equation}
    Note that for $i \leq m$, $\inner{\Sbs\mubs_i, \Sbs\mubs_j} = \1\{i=j\}$ under the population sketch $\Sbs$. Hence, to emphasize the deviation brought on by the empirical sketch $\widehat{\Sbs}$, we re-write the first term as
    \begin{equation}
        \pi_i(\1\{k=i\} - P_{k|i}) + \sum_{j=1}^K \pi_j(\1\{k=j\} - P_{k|j}) (\inner{\widehat{\Sbs}\mubs_i, \widehat{\Sbs}\mubs_j} - \1\{i=j\}).
    \end{equation}
    Recalling the identity
    \begin{equation}
        \dot{p}_{i|i} = p_{i|i}\left(\dot{\theta}_{ii} - \sum_{k=1}^K p_{k|i} \dot{\theta}_{ki}\right),
    \end{equation}
    we have that the empirical PCA error $(\inner{\widehat{\Sbs}\mubs_i, \widehat{\Sbs}\mubs_j} - \1\{i=j\})$ contributes
    \begin{equation}
        p_{i|i}\sum_{j=1}^K \pi_j\left(\inner{\widehat{\Sbs}\mubs_i, \widehat{\Sbs}\mubs_j} - \1\{i=j\}\right) \left[\1\{i=j\} - P_{i|j} - \sum_{k=1}^K p_{k|i} (\1\{k=j\} - P_{k|j})\right]
    \end{equation}
    to $\dot{p}_{i|i}$ Focusing on the right-most expression in parentheses, we have
    \begin{equation}
        \begin{split}
            \1\{i=j\} - P_{i|j} - \sum_{k=1}^K p_{k|i} (\1\{k=j\} - P_{k|j}) &= \1\{i=j\} - P_{i|j} - p_{j|i} + \sum_{k=1}^K p_{k|i} P_{k|j} \\
            &= (\1\{i=j\} - p_{j|i}) - \left(P_{i|j} - \sum_{k=1}^K p_{k|i} P_{k|j}\right).
        \end{split}
    \end{equation}
    Since $|\1\{i=j\} - p_{j|i}| \leq 1-p_{i|i}$ and $P_{i|j} - \sum_{k=1}^K p_{k|i} P_{k|j} = (1-p_{i|i}) P_{i|j} - \sum_{k \neq i} p_{k|i} P_{k|j}$, we have
    \begin{equation}
        \begin{split}
        \absolute*{(\1\{i=j\} - p_{j|i}) - \left(P_{i|j} - \sum_{k=1}^K p_{k|i} P_{k|j}\right)} &\leq \absolute{\1\{i=j\} - p_{j|i}} + \absolute*{P_{i|j} - \sum_{k=1}^K p_{k|i} P_{k|j}} \\
        &\leq 3(1-p_{i|i}).
        \end{split}
    \end{equation}
    Hence, the empirical PCA error contributes at most $3\eps_{\text{PCA}} p_{i|i}(1-p_{i|i})$ to $\dot{p}_{i|i}$.

    The result then follows from the same argument as in Lemma \ref{lem:noisy_p_ii_bounds}, since we have argued that Lemmas \ref{lem:weight_norm_bound}-\ref{lem:log_p_inequality} still hold in this sketched setting.
\end{proof}

For sequential recovery, starting from the above lemma and employing the same argument as in the proof of Lemma \ref{lem:noisy_sequential_recovery} yields the same result in this setting so long as $\eps_{\text{PCA}} \leq c_0 \pi_m$ for a sufficiently small constant $c_0$. Indeed, this ensures that the $\pi_i p_{i|i} (1-p_{i|i})^2$ is dominant in the dynamics at least until $p_{i|i} \geq 1/2$.

For the post-recovery decay, we follow a similar integrating-factor approach to the one in the proof of Lemma \ref{lem:post_recovery_decay}. Indeed, within the time horizon, assuming $\eps_{\text{PCA}} \leq c_0 \pi_m$ for a sufficiently small constant $c_0 > 0$ ensures that $p_{i|i}$ does not decrease below $1/2$ after recovery. This yields the lower bound
\begin{equation}
    \frac{\dee}{\dee t} \left[\frac{1}{1-p_{i|i}(t)}\right] \geq \frac{a}{2}\pi_i - \frac{C\sigma^2 M(t)}{1-p_{i|i}(t)} - \frac{C\eps_{\text{PCA}}}{1-p_{i|i}(t)}.
\end{equation}
We introduce the integrating factor $\exp(C \int_{t_i}^t (\sigma^2 M(r) + \eps_{\text{PCA}}) \dee r)$ to obtain
\begin{equation}
    \frac{\dee}{\dee t} \left[\exp\left(C\int_{t_i}^t (\sigma^2 M(r) + \eps_{\text{PCA}} \dee r\right)\frac{1}{1-p_{i|i}(t)}\right] \geq \frac{a}{2}\pi_i \exp\left(C \int_{t_i}^t (\sigma^2 M(r) + \eps_{\text{PCA}}) \dee r\right).
\end{equation}
Integrating from $t_i$ to $t$ gives
\begin{equation}
    \begin{split}
        \frac{1}{1-p_{i|i}(t)} &\geq \exp\left(-C\int_{t_i}^t (\sigma^2 M(r) + \eps_{\text{PCA}}) \dee r \right) \\
        &\quad \cdot \left(2 + \frac{a}{2}\pi_i \int_{t_i}^t \exp\left(C\int_{t_i}^s (\sigma^2 M(r) + \eps_{\text{PCA}}) \dee r \right) \dee s\right) \\
        &= 2\exp\left(-C \int_{t_i}^t (\sigma^2 M(r)  + \eps_{\text{PCA}}) \dee r\right) \\
        &\quad + \frac{a}{2}\pi_i \int_{t_i}^t \exp\left(-C \int_s^t (\sigma^2 M(r) + \eps_{\text{PCA}}) \dee r\right) \dee s.
    \end{split}
\end{equation}
Recall that in the proof of Lemma \ref{lem:post_recovery_decay}, we argued that $\sigma^2 \int_{t_i}^T M(r) \dee r \lesssim c_*^2$. Therefore, $\exp(-C\int_{t_i}^t \sigma^2 M(r) \dee r) \gtrsim 1$. Moreover, we trivially have $\frac{1}{1-p_{i|i}(t)} \geq 1$. Hence, 
\begin{equation}
    \begin{split}
        \frac{1}{1-p_{i|i}(t)} &\gtrsim 1 + \pi_i \int_{t_i}^t \exp(-C \eps_{\text{PCA}}(t-s)) \dee s \\
        &= 1 + \frac{\pi_i}{C\eps_{\text{PCA}}}(1-e^{-C\eps_{\text{PCA}}(t-t_i)}) \\
        &\asymp 1 + \min\left\{\pi_i (t-t_i), \frac{\pi_i}{\eps_{\text{PCA}}}\right\},
    \end{split}
\end{equation}
where the last line follows from the relation $1-e^{-u} \asymp \min\{u,1\}$. Taking the reciprocal gives
\begin{equation}
    1-p_{i|i}(t) \lesssim \frac{1}{1+\min\{\pi_i(t-t_i), \frac{\pi_i}{\eps_{\text{PCA}}}\}} \leq \frac{1}{1+\pi_i(t-t_i)} + \frac{\eps_{\text{PCA}}}{\pi_i}.
\end{equation}
Since $p_{i|i} \geq 1/2$, we have $-\log p_{i|i}(t) \leq 2(1-p_{i|i}(t))$. Then, Lemma \ref{lem:log_p_inequality} implies, over the time horizon $[0,T]$,
\begin{equation}\label{eq:emp_pca_head_risk_term}
    \E_{\xbs|y=i}[-\log p_i(\widehat{\Sbs}\xbs)] \lesssim \frac{1}{1+\pi_i(t-t_i)} + \frac{\eps_{\text{PCA}}}{\pi_i}.
\end{equation}

Note that, under our assumption $\eps_{\text{PCA}} \leq c_0 \pi_m$,
\begin{equation}\label{eq:emp_pca_head_risk_err_term}
    \sum_{i=1}^m \pi_i \frac{\eps_{\text{PCA}}}{\pi_i} = m\eps_{\text{PCA}} \lesssim m \pi_m \asymp m^{1-\alpha}.
\end{equation}
We show that the above contribution is dominated by the approximation error arising from the bottom $K-m$ classes.

\begin{lemma}[Stability of Discarded Class Risk]\label{lem:emp_pca_discarded_class_risk}
    If $\eps_{\text{PCA}} \leq c\log K/T$ for some $c \in (0,1)$, then for $t \in [0,T]$,
    \begin{equation}
        \sum_{i=m+1}^K \pi_i \E_{\xbs|y=i}[-\log p_i(\widehat{\Sbs}\xbs)] \asymp m^{1-\alpha}\log K. 
    \end{equation}
\end{lemma}
\begin{proof}
    Recall that Lemma \ref{lem:log_p_inequality} applied to the sketched case implies
    \begin{equation}
        -\log p_{i|i} \leq \E_{\xbs|y=i}[-\log p_i(\widehat{\Sbs}\xbs)] \leq e^{2\sigma^2 M(t)^2} (-\log p_{i|i}),
    \end{equation}
    where the prefactor in the upper bound is constant order up to time $T$. 

    Let $g(\ubs) = -\log p_i(\ubs)$. Then $\nabla g(u) = \sum_{k=1}^K p_k(\ubs) \wbs_k - \wbs_i$. Hence, $\norm{\nabla g(u)} \leq 2M(t)$. Hence,
    \begin{equation}
        -\log p_{i|i}(t) = \log K + O\left(M(t) \norm{\widehat{\Sbs}\mubs_i}\right) = \log K + O(M(t) \eps_{\text{PCA}}) = \log K + O(t \eps_{\text{PCA}}).
    \end{equation}
    Hence, under the assumption $t \eps_{\text{PCA}} \leq c_1 \log K$, we obtain $-\sum_{i=m+1}^K \pi_i \log p_{i|i} \asymp m^{1-\alpha}\log K$.
\end{proof}

\begin{theorem}[Empirical PCA Scaling Law]\label{thm:emp_pca_scaling_law}
    Assume $\alpha > 1$ and $m \leq K/2$. There exists a constant $c_0 > 0$ such that if 
    \begin{equation}\label{eq:eps_pca_condition}
        \eps_{\text{PCA}} \leq c_0 \min\{\pi_m, \sigma \log K\},
    \end{equation}
    the population risk under empirical sketch satisfies, uniformly for $t \in [0,T]$,
    \begin{equation}
        \Rcal_{\widehat{\Sbs}}(t) \asymp
        \begin{cases}
            \log K, & 0 \leq t \lesssim \log K, \\
            m^{1-\alpha} \log K + (\log K)^{2-1/\alpha} \cdot t^{-(1-1/\alpha)}, &\log K \lesssim t \leq T.
        \end{cases}
    \end{equation}
\end{theorem}
\begin{proof}
    As in the previous subsection, let $\Rcal_{\text{opt}} := -\sum_{i=1}^m \pi_i \E_{\xbs|y=i}[\log p_i(\widehat{\Sbs}\xbs)]$. For each fixed time $t \in [0,T]$, we can split this into head and tail risks using the sequential recovery times from Lemma \ref{lem:noisy_sequential_recovery}, which we argued carry over to the empirical PCA case when $\eps_{\text{PCA}} \leq c_0 \pi_m$. Hence, we obtain the same tail risk as in Appendix \ref{appendix:gmm_tail_risk}. As for the head risk, it is upper-bounded by the same expression as in Appendix \ref{appendix:gmm_head_risk}, plus an additional $m^{1-\alpha}$ term due to \eqref{eq:emp_pca_head_risk_term} and \eqref{eq:emp_pca_head_risk_err_term}. Hence, for some $c_1, c_2 \in (0,1)$ and all $t \in [0,T]$,
    \begin{equation}
        \Rcal_{\text{opt}}(t) \asymp
        \begin{cases}
            \log K, & 0 \leq t \lesssim \log K, \\
            (\log K)^{2-1/\alpha} \cdot t^{-(1-1/\alpha)} + O(m^{1-\alpha}), & \log K \leq t \leq c_1 m^{\alpha} \log K
        \end{cases}
    \end{equation}
    and $\Rcal_{\text{opt}}(t) \lesssim m/t + m^{1-\alpha}$ for $(1+c_2)m^{\alpha} \log K \leq t \leq T$. Combining this with Lemma \ref{lem:emp_pca_discarded_class_risk}, whose condition is satisfied when \eqref{eq:eps_pca_condition} holds, gives the result.
\end{proof}

We now derive a sufficient condition on the sample size $N$ to ensure that \eqref{eq:eps_pca_condition} holds with high-probability. We bound $\eps_{\text{PCA}}$ using the Davis-Kahan Theorem \cite{davis1970rotation}. Specifically, we employ the operator norm version that captures distances and angles between eigenspaces via the distance between projection matrices \cite[Lemma 2.5, Theorem 2.7]{chen2021spectral}:
\begin{equation}\label{eq:davis_kahan_bound}
    \eps_{\text{PCA}} = \norm{\widehat{\Pbs}_m-\Pbs_m} \lesssim \norm{\sin \Theta(\widehat{\Vbs}_m, \Vbs_m)} \lesssim \frac{\norm{\widehat{\Gbs} - \Gbs}}{\pi_m - \pi_{m+1}}.
\end{equation}
To bound the numerator, we relate the second moment matrices $\widehat{\Gbs}$ and $\Gbs$ to the corresponding empirical and population covariance matrices $\widehat{\Sigmabs}$ and $\Sigmabs$, respectively. Letting $\mubs = \E[\xbs]$ and $\widehat{\mubs} = \frac{1}{N}\sum_{n=1}^N \xbs_n$, note that
\begin{equation}
    \widehat{\Gbs} - \Gbs = \left[\frac{1}{N}\sum_{n=1}^N (\xbs_n - \mubs)(\xbs_n - \mubs)^{\top} - \Sigmabs\right] + \mubs(\widehat{\mubs} - \mubs)^{\top} + (\widehat{\mubs} - \mubs)\mubs^{\top}.
\end{equation}
Since $||\mubs|| \leq 1$, we have
\begin{equation}
    \norm{\widehat{\Gbs}-\Gbs} \leq \norm*{\frac{1}{N}\sum_{n=1}^N (\xbs_n - \mubs)(\xbs_n - \mubs)^{\top} - \Sigmabs} + 2\norm{\hat{\mubs}-\mubs}. 
\end{equation}
Now, $\xbs - \mubs$ is a sub-Gaussian vector with sub-Gaussian norm $O(1+\sigma)$. Hence, with probability at least $1-\delta/2$,
\begin{equation}
    \norm*{\frac{1}{N}\sum_{n=1}^N (\xbs_n - \mubs)(\xbs_n - \mubs)^{\top} - \Sigmabs} \lesssim (1+\sigma^2) \left[\sqrt{\frac{d + \log(2/\delta)}{N}} + \frac{d+\log(2/\delta)}{N}\right].
\end{equation}
Similarly, with probability at least $1-\delta/2$,
\begin{equation}
    ||\widehat{\mubs}-\mubs|| \lesssim \sqrt{1+\sigma^2} \sqrt{\frac{d+\log(2/\delta)}{N}}.
\end{equation}
Therefore, whenever $N \geq d + \log(2/\delta)$, a union bound gives
\begin{equation}
    \norm{\widehat{\Gbs}-\Gbs} \lesssim (1+\sigma^2) \sqrt{\frac{d+\log(2/\delta)}{N}}
\end{equation}
with probability at least $1-\delta$.

As for the denominator in the RHS of \eqref{eq:davis_kahan_bound},
\begin{equation}
    \pi_m - \pi_{m+1} \asymp m^{-\alpha} - (m+1)^{-\alpha} = \alpha \xi^{-(\alpha+1)}
\end{equation}
for some $\xi \in [m,m+1]$ by the Mean Value Theorem. Therefore, on the high-probability event,
\begin{equation}
    \eps_{\text{PCA}} \lesssim (1+\sigma^2)m^{\alpha+1}\sqrt{\frac{d+\log(2\delta)}{N}}.
\end{equation}
Therefore, a sufficient condition for \eqref{eq:eps_pca_condition} to hold with probability at least $1-\delta$ is
\begin{equation}
    N \gtrsim (1+\sigma^2)^2 (d + \log(2/\delta)) m^{2\alpha+2} \max\left\{m^{2\alpha}, \frac{1}{\sigma^2 (\log K)^2}\right\}.
\end{equation}

This proves Theorem \ref{thm:noisy_sketch}.

\section{Proofs for Online SGD}\label{appendix:proofs_online_sgd}
In this section, we discretize the continuous-time dynamics and work with i.i.d.\@ samples from the target mixture model \eqref{eq:target}, which we will refer to as $p^*$ in this appendix. We work in a general framework that will cover both the sketched and unsketched cases.

Let $\Sbs$ denote either the identity $\Ibs_d$ or a sketching matrix in $\R^{m \times d}$. In either case, $||\Sbs||_{\text{op}} \leq 1$. Write $\ubs = \Sbs \xbs$, and let each weight $\wbs_i$ have the same dimension as $\ubs$. To make the dependence on the weights $\Wbs$ explicit, we use the notation
\begin{equation}
    p_i(\Wbs; \ubs) = \frac{\exp(\inner{\wbs_i, \ubs})}{\sum_{k=1}^K \exp(\inner{\wbs_k, \ubs})}, \quad \ell(\Wbs; \ubs, y) = -\log p_y(\Wbs; \ubs)
\end{equation}
for the softmax probability of each class and the loss on a given instance $(\ubs,y)$, respectively. The population gradient flow trajectory, the population gradient descent iterates, and the online SGD iterates are denoted respectively by
\begin{align}
    &\dot{\Wbs}(t) = -\nabla \Rcal(\Wbs(t)), \quad &\Wbs(0) = \zerobs, \\
    &\widetilde{\Wbs}^{(n+1)} = \widetilde{\Wbs}^{(n)} - \eta \nabla \Rcal(\widetilde{\Wbs}^{(n)}), \quad &\widetilde{\Wbs}^{(0)} = \zerobs, \\
    &\Wbs^{(n+1)} = \Wbs^{(n)} - \eta \nabla \ell(\Wbs^{(n)}, \ubs^{(n+1)}, y^{(n+1)}), \quad &\Wbs^{(0)} = \zerobs,
\end{align}
where each $(\ubs^{(n)}, y^{(n)})$ arises from an i.i.d.\@ sample $(\xbs^{(n)}, y^{(n)})$ from $p^*$, with $\ubs^{(n)} = \Sbs \xbs^{(n)}$. To go between the continuous and discrete-time optimization, we use the effective optimization time
\begin{equation}
    t_n = \eta n.
\end{equation}
Further define the second-moment scales,
\begin{equation}\label{eq:L_M_def}
    L = \norm{\E[\ubs \ubs^{\top}]}_{\text{op}}, \quad M = \E||\ubs||^2 = \tr(\E[\ubs \ubs^{\top}]),
\end{equation}
where both expectations are taken conditional on $\Sbs$, which can be random but is assumed fixed. The constant $L$ controls smoothness and factors the discretization error, while $M$ controls the variance of the stochastic gradients.

\subsection{Risk and Gradient Descent Properties}\label{appendix:risk_and_gd_properties}

\begin{lemma}\label{lem:convex_smooth_risk}
    The risk $\Rcal$ is convex and $L$-smooth in the weights $\Wbs$. Moreover, 
    \begin{equation}\label{eq:risk_grad_at_zero}
        \norm{\nabla \Rcal(0)}_F^2 \leq \left(1 - \frac{1}{K}\right)\sum_{i=1}^K \pi_i^2.
    \end{equation}
\end{lemma}
\begin{proof}
We have
\begin{equation}\label{eq:risk_gradient_w}
    \nabla_{\wbs_i} \Rcal(\Wbs) = -\pi_i \Sbs \mubs_i + \E_{\ubs}[p_i(\ubs)\ubs]
\end{equation}
and
\begin{equation}
    \nabla_{\wbs_j}\big[\nabla_{\wbs_i} \Rcal(\Wbs)\big] = \E_{\ubs}\big[p_i(\ubs)\big(\1\{i=j\}-p_j(\ubs)\big)\ubs \ubs^{\top}\big].
\end{equation}
Therefore, for any $\Vbs$ with the same size as $\Wbs$, we have
\begin{equation}
    \nabla^2 \Rcal(\Wbs)[\Vbs,\Vbs] = \sum_{i=1}^K \sum_{j=1}^K \E_{\ubs}\big[p_i(\ubs)\big(\1\{i=j\}-p_j(\ubs)\big)\inner{\vbs_i,\ubs}\inner{\vbs_j,\ubs}\big].
\end{equation}
Now, if we let $\abs = \Vbs^{\top} \ubs \in \R^K$ and $\pbs = (p_1(\ubs),\dots,p_k(\ubs))$, we have
\begin{equation}
    \big|\nabla^2 \Rcal(\Wbs)[\Vbs,\Vbs]\big| = \E\big[|\abs^{\top}(\mathrm{diag}(\pbs) - \pbs\pbs^{\top})\abs|\big] \lesssim \E[||\abs||^2].
\end{equation}
Now, we can expand
\begin{equation}
    \E[\norm{\abs}^2] = \E_{\ubs}\left[\inner{\Vbs^{\top} \ubs, \Vbs^{\top} \ubs}\right] = \tr\left(\Vbs^{\top} \E[\ubs \ubs^{\top}] \Vbs\right) \leq ||\E[\ubs \ubs^{\top}|| \, ||\Vbs||_{F}^2 = L ||\Vbs||_F^2.
\end{equation}
Hence, $\Rcal$ is convex and $\nabla R(\Wbs)$ is Lipschitz with constant $L$.

Moreover, from \eqref{eq:risk_gradient_w}, we have
\begin{equation}
    \nabla_{\wbs_i} \Rcal(0) = \Sbs \left(-\pi_i \mubs_i + \frac{1}{K} \sum_{k=1}^K \pi_k \mubs_k\right).
\end{equation}
Then,
\begin{equation}
    \begin{split}
    ||\nabla \Rcal(0)||_F^2 &\leq \sum_{i=1}^K \norm*{-\pi_i \mubs_i + \frac{1}{K}\sum_{k=1}^K \pi_k \mubs_k}^2 \\
    &= \sum_{i=1}^K \pi_i^2 - \frac{2}{K} \sum_{i=1}^K \pi_i^2 + \frac{1}{K} \sum_{k=1}^K \pi_k^2 = \left(1-\frac{1}{K}\right)\sum_{i=1}^K \pi_i^2,
    \end{split}
\end{equation}
which in particular is at most one.
\end{proof}

\begin{lemma}[Non-Expansiveness of a Population GD Step]\label{lem:non_expansive_gd_step}
    For $0 < \eta \leq 2/L$, the map
    \begin{equation}
        U_{\eta}(\Wbs) = \Wbs - \eta \nabla \Rcal(\Wbs)
    \end{equation}
    is non-expansive with respect to the Frobenius norm, i.e.,
    \begin{equation}
        \norm{U_{\eta}(\Abs) - U_{\eta}(\Bbs)}_F \leq \norm{\Abs - \Bbs}_F.
    \end{equation}
    Moreover, the gradient norm is non-increasing along both population gradient flow and population gradient descent.
\end{lemma}
\begin{proof}
    We have
    \begin{equation}
        \begin{split}
            U_{\eta}(\Abs) - U_{\eta}(\Bbs) &= (\Abs - \eta \nabla \Rcal(\Abs)) - (\Bbs - \eta \nabla \Rcal(\Bbs)) \\
            &= \left(\Ibs - \eta \int_0^1 \nabla^2 \Rcal(\Bbs + s(\Abs - \Bbs)) \dee s\right)(\Abs - \Bbs),
        \end{split}
    \end{equation}
    where we have used, by the Fundamental Theorem of Calculus,
    \begin{equation}
        \nabla \Rcal(\Abs) - \nabla \Rcal(\Bbs) = \left(\int_0^1 \nabla^2 \Rcal(\Bbs + s(\Abs -\Bbs)) \dee s\right)(\Abs - \Bbs).
    \end{equation}
    By Lemma \ref{lem:convex_smooth_risk}, $I - \eta \int_0^1 \nabla^2 \Rcal(\cdot) \dee s$ has operator norm at most one when $\eta \leq 2/L$, which establishes non-expansiveness.

    Along population gradient flow,
    \begin{equation}
        \frac{\dee}{\dee t} \left[\frac{1}{2} \norm{\nabla \Rcal(\Wbs(t))}_F^2\right] = -\nabla^2 \Rcal(\Wbs(t))[\nabla \Rcal(\Wbs(t)), \nabla \Rcal(\Wbs(t))] \leq 0.
    \end{equation}
    For population gradient descent,
    \begin{equation}
        \nabla \Rcal(U_{\eta}(\Wbs)) = \left(\Ibs - \eta \int_0^1 \nabla^2 \Rcal(\Wbs - s\eta \nabla \Rcal(\Wbs) \dee s\right) \nabla \Rcal(\Wbs),
    \end{equation}
    where the operator norm of the matrix multiplying $\nabla \Rcal(\Wbs)$ is at most one due to the $\eta \leq 2/L$ condition.
\end{proof}

\subsection{From Gradient Flow to Gradient Descent}\label{appendix:gf_to_gd}

\begin{lemma}[Discretization Error]\label{lem:discretization_error}
    Assume $0 < \eta \leq 1/L$. Fix $T > 0$ and an integer $N$ with $\eta N \leq T$. Then,
    \begin{equation}\label{eq:weight_discretization_gap}
        \max_{0 \leq n \leq N} \norm{\widetilde{\Wbs}^{(n)} - \Wbs(t_n)}_F \leq \frac{1}{2}L\eta T.
    \end{equation}
    Consequently,
    \begin{equation}\label{eq:risk_discretization_gap}
        \max_{0 \leq n \leq N} \absolute{\Rcal(\widetilde{\Wbs}^{(n)})- \Rcal(\Wbs(t_n))} \leq \frac{1}{2} L\eta T + \frac{1}{8} L^3 \eta^2 T^2.
    \end{equation}
    In particular, the right-hand side is at most a universal constant times $\eta LT (1+LT)$.
\end{lemma}
\begin{proof}
    We may write
    \begin{equation}
        \Wbs(t_{n+1}) = U_{\eta}(\Wbs(t_n)) + \tau_n, \quad \text{where} \quad \tau_n = \int_{t_n}^{t_{n+1}} \nabla \Rcal(\Wbs(t_n)) - \nabla \Rcal(\Wbs(s)) \dee s.
    \end{equation}
    By $L$-smoothness (Lemma \ref{lem:convex_smooth_risk}) and monotonicity of the gradient norm along gradient flow (Lemma \ref{lem:non_expansive_gd_step}), we have
    \begin{equation}
        \begin{split}
            \norm{\nabla \Rcal(\Wbs(t_n)) - \nabla \Rcal(\Wbs(s))}_F &\leq L\norm{\Wbs(t_n) - \Wbs(s)}_F \\
            &\leq L\int_{t_n}^s \norm{\nabla \Rcal(\Wbs(r))}_F \dee r \\
            &\leq L(s-t_n) ||\nabla \Rcal(\Wbs(0))||_F \\
            &\leq L(s-t_n),
        \end{split}
    \end{equation}
    where we have used Lemma \ref{lem:convex_smooth_risk} again in the last inequality. Therefore,
    \begin{equation}
        ||\tau_n||_F \leq \int_{t_n}^{t_{n+1}} L(s-t_n) \dee s = \frac{1}{2}L\eta^2.
    \end{equation}
    Set $e_n = \norm{\widetilde{\Wbs}^{(n)} - \Wbs(t_n)}_F$. Then,
    \begin{equation}
        \begin{split}
            e_{n+1} &= \norm{\widetilde{\Wbs}^{(n+1)} - \Wbs(t_{n+1})}_F \\
            &= \norm*{U_{\eta}(\widetilde{\Wbs}^{(n)}) - \left(U_{\eta}(\Wbs(t_n)) + \tau_n\right)}_F \\
            &\leq ||\widetilde{\Wbs}^{(n)} - \Wbs(t_n)||_F + ||\tau_n||_F \\
            &\leq e_n + \frac{1}{2}L\eta^2, 
        \end{split}
    \end{equation}
    where the third line uses non-expansiveness. Since $e_0 = 0$, we obtain $e_n \leq n L\eta^2/2 \leq L\eta T/2$, proving \eqref{eq:weight_discretization_gap}.

    By convexity and $L$-smoothness of $\Rcal$ (Lemma \ref{lem:convex_smooth_risk}) and Cauchy-Schwarz, we have,
    \begin{equation}
        |\Rcal(\widetilde{\Wbs}^{(n)}) - \Rcal(\Wbs(t_n))| \leq ||\nabla \Rcal(\Wbs(t_n))||_F e_n + \frac{L}{2}e_n^2 \leq e_n + \frac{L}{2}e_n^2,
    \end{equation}
    where we recall that $||\nabla \Rcal(\Wbs(t_n))||_F \leq ||\nabla \Rcal(0)||_F \leq 1$. \eqref{eq:risk_discretization_gap} immediately follows.
\end{proof}

\subsection{From Population to Stochastic Gradient Descent}\label{appendix:gd_to_sgd}

We now compare online SGD to population GD. Let $\Fcal_n$ denote the $\sigma$-algebra generated by the first $n$ samples $(\ubs_i, y_i)_{i=1}^n$ used by online SGD. Define the stochastic gradient error
\begin{equation}
    \Dbs_{n+1} = \nabla \ell(\Wbs^{(n)}; \ubs^{(n+1)}, y^{(n+1)}) - \nabla \Rcal(\Wbs^{(n)}).
\end{equation}
Then, $\E[\Dbs_{n+1} \mid \Fcal_n] = 0$, and the online update becomes
\begin{equation}
    \Wbs^{(n+1)} = U_{\eta}(\Wbs^{(n)}) - \eta \Dbs_{n+1}.
\end{equation}

\begin{lemma}[Conditional Stochastic Gradient Variance]\label{lem:conditional_stoch_grad_variance}
    For every $n \geq 0$, 
    \begin{equation}
        \E\left[\norm{\Dbs_{n+1}}_F^2 \mid \Fcal_n\right] \leq 2M.
    \end{equation}
\end{lemma}
\begin{proof}
    Given a data sample $(\ubs,y)$, we have
    \begin{equation}
        \nabla \ell(\Wbs; \ubs, y) = \ubs \left(\pbs(\Wbs; \ubs) - \ebs_y\right)^{\top},
    \end{equation}
    where $\pbs(\Wbs; \ubs)$ is the vector of softmax probabilities and $\ebs_y$ is the $y$th standard basis vector in $\R^K$. Then,
    \begin{equation}
        \begin{split}
        ||\nabla \ell(\Wbs; \ubs, y)||_F^2 &= ||\ubs||^2 ||\pbs(\Wbs;\ubs) - \ebs_y||^2 \\
        &= ||\ubs||^2 \norm*{\sum_{i=1}^K p_i(\Wbs; \ubs)^2 - 2p_y(\Wbs;\ubs) + 1} \\
        &\leq 2||\ubs||^2. 
        \end{split}
    \end{equation}
    Consequently, since the conditional variance is at most the conditional second moment,
    \begin{equation}
        \E[||\Dbs_{n+1}||_F^2 \mid \Fcal_n] \leq \E[\norm{\nabla \ell(\Wbs^{(n)};\ubs,y)}_F^2 \mid \Fcal_n] \leq 2 \E||\ubs||_2^2 = 2M, 
    \end{equation}
    by definition of $M$ \eqref{eq:L_M_def}.
\end{proof}

Next, we prove a uniform-in-time high-probability bound.
\begin{lemma}[Uniform Online SGD Stability]\label{lem:sgd_to_gd_bound}
    Assume $0 < \eta < 1/L$. Fix $T > 0$, $\delta \in (0,1)$, and $N$ such that $\eta N \leq T$. Then, with probability at least $1-\delta$,
    \begin{equation}\label{eq:sgd_gd_weight_gap}
        \max_{0 \leq n \leq N} \norm{\Wbs^{(n)} - \widetilde{\Wbs}^{(n)}} \leq \sqrt{\frac{2\eta MT}{\delta}}.
    \end{equation}
    On the same event,
    \begin{equation}\label{eq:sgd_gd_risk_gap}
        \max_{0 \leq n \leq N} \absolute{\Rcal(\Wbs^{(n)}) - \Rcal(\widetilde{\Wbs}^{(n)})} \leq \sqrt{\frac{2\eta MT}{\delta}} + \frac{L\eta MT}{\delta}.
    \end{equation}
\end{lemma}
\begin{proof}
    Let $\Ebs_n = \Wbs^{(n)} - \widetilde{\Wbs}^{(n)}$. Then, 
    \begin{equation}
        E_{n+1} = U_{\eta}(\Wbs^{(n)}) - U_{\eta}(\widetilde{\Wbs}^{(n)}) - \eta \Dbs_{n+1}.
    \end{equation}
    We compute
    \begin{equation}
        \E[||\Ebs_{n+1}||_F^2 \mid \Fcal_n] = \norm{U_{\eta}(\Wbs^{(n)}) - U_{\eta}(\widetilde{\Wbs}^{(n)})}_F^2 + \eta^2 \E[||\Dbs_{n+1}||_F^2 \mid \Fcal_n] \leq ||\Ebs_n||_F^2 + 2\eta^2 M,
    \end{equation}
    by Lemmas \ref{lem:non_expansive_gd_step} and \ref{lem:conditional_stoch_grad_variance}.

    Fix $r > 0$ and define the stopping time
    \begin{equation}
        \tau := \inf\{n \geq 0: ||\Ebs_n||_F \geq r\},
    \end{equation}
    with the convention $\inf \empty = \infty$. We have
    \begin{equation}
        \begin{split}
            \E\norm*{\Ebs_{(n+1)\land \tau}}_F^2 &= \E\left[\1_{\tau \leq n} ||\Ebs_{\tau}||_F^2\right] + \E\left[\1_{\tau > n}||\Ebs_{n+1}||_F^2\right] \\
            &= \E\left[\1_{\tau \leq n} ||\Ebs_{\tau}||_F^2\right] + \E\left[\1_{\tau > n}\E[||\Ebs_{n+1}||_F^2 \mid \Fcal_n]\right] \\
            &\leq \E\left[\1_{\tau \leq n} ||\Ebs_{\tau}||_F^2\right] + \E\left[\1_{\tau > n} (||\Ebs_n||_F^2 + 2\eta^2 M)\right] \\
            &\leq \E\norm*{\Ebs_{n \land \tau}}_F^2 + 2\eta^2 M.
        \end{split}
    \end{equation}
    Since $\Ebs_0 = \zerobs$, we have
    \begin{equation}
        \E\norm*{\Ebs_{N \land \tau}}_F^2 \leq 2\eta^2 MN.
    \end{equation}
    If $\tau \leq N$, then $||\Ebs_{N \land \tau}||_F = ||\Ebs_{\tau}||_F \geq r$. By Markov's inequality,
    \begin{equation}
        \Prob\left(\max_{0 \leq n \leq N} ||\Ebs_n||_F \geq r\right) = \Prob(\tau \leq N) \leq \frac{\E\norm*{\Ebs_{N \land \tau}}_F^2}{r^2} \leq \frac{2\eta^2 M N}{r^2}.  
    \end{equation}
    Choose $r = \eta \sqrt{2MN/\delta}$. Then, $\eta N \leq T$ implies \eqref{eq:sgd_gd_weight_gap}. Moreover, by convexity and $L$-smoothness of $\Rcal$ (Lemma \ref{lem:convex_smooth_risk}), we have, on the event that \eqref{eq:sgd_gd_weight_gap} holds,
    \begin{equation}
        |\Rcal(\Wbs^{(n)}) - \Rcal(\widetilde{\Wbs}^{(n)})| \leq ||\nabla \Rcal(\widetilde{\Wbs}^{(n)})||_F ||\Ebs_n||_F + \frac{L}{2}||\Ebs_n||_F^2 \leq \sqrt{\frac{2\eta MT}{\delta}} + \frac{L\eta MT}{\delta},
    \end{equation}
    where the last inequality uses monotonicity of gradient norms along GD (Lemma \ref{lem:non_expansive_gd_step}) and \eqref{eq:risk_grad_at_zero}.
\end{proof}

\subsection{From Population Gradient Flow to Online SGD}\label{appendix:gf_to_sgd}

\begin{theorem}[Uniform Transfer from GF to SGD]\label{thm:gf_to_sgd}
    Assume $0 < \eta \leq 1/L$. Fix $T > 0$, $\delta \in (0,1)$, and $N$ such that $\eta N \leq T$. With probability at least $1-\delta$, for each $0 \leq n \leq N$,
    \begin{equation}
        \absolute*{\Rcal(\Wbs^{(n)}) - \Rcal(\Wbs(t_n))} \leq \frac{1}{2}L\eta T + \frac{1}{8} L^3 \eta^2 T^2 + \sqrt{\frac{2\eta MT}{\delta}} + \frac{L\eta MT}{\delta}.
    \end{equation}
\end{theorem}
\begin{proof}
    The result is immediate from the triangle inequality and Lemmas \ref{lem:discretization_error} and \ref{lem:sgd_to_gd_bound}.
\end{proof}

\begin{corollary}[Learning Rate Condition]\label{cor:general_lr_condition}
    Fix $T > 0$ and $\delta \in (0,1)$. There exists a universal constant $c > 0$ such that, if 
    \begin{equation}
        \eta \leq c \min\left\{\frac{1}{L}, \frac{\Rcal(\Wbs(T))}{LT(1+LT)}, \frac{\delta \Rcal(\Wbs(T))^2}{MT}, \frac{\delta \Rcal(\Wbs(T))}{LMT}\right\},
    \end{equation}
    then with probability at least $1-\delta$,
    \begin{equation}
        \frac{1}{2}\Rcal(\Wbs(t_n)) \leq \Rcal(\Wbs^{(n)}) \leq \frac{3}{2} \Rcal(\Wbs(t_n))
    \end{equation}
    for every $n$ satisfying $t_n \leq T$.
\end{corollary}
\begin{proof}
    By Theorem \ref{thm:gf_to_sgd} and the condition $\eta \leq 1/L$, we have, with probability at least $1-\delta$,
    \begin{equation}
        \absolute{\Rcal(\Wbs^{(n)}) - \Rcal(\Wbs(t_n))} \leq C \left(\eta LT (1+LT) + \sqrt{\frac{\eta MT}{\delta}} + \frac{L\eta MT}{\delta}\right),
    \end{equation}
    for every n with $t_n \leq T$, where $C > 0$ is a universal constant. In the condition on $\eta$, we can choose $c$ sufficiently small such that
    \begin{equation}
        \absolute{\Rcal(\Wbs^{(n)}) - \Rcal(\Wbs(t_n))} \leq \frac{1}{2}\Rcal(\Wbs(t)).
    \end{equation}
    The risk is non-increasing along population gradient flow because
    \begin{equation}
        \frac{\dee}{\dee t} \Rcal(\Wbs(t)) = \inner{\nabla \Rcal(\Wbs(t)), \dot{\Wbs}(t)}_F = -||\nabla \Rcal(\Wbs(t))||_F^2 \leq 0.
    \end{equation}
    So, whenever $t_n \leq T$, we have $\Rcal(\Wbs(T)) \leq \Rcal(\Wbs(t_n))$. Therefore,
    \begin{equation}
        \absolute{\Rcal(\Wbs^{(n)}) - \Rcal(\Wbs(t_n))} \leq \frac{1}{2}\Rcal(\Wbs(t_n)),
    \end{equation}
    which gives the result.
\end{proof}

\subsection{Unsketched Noisy Case}\label{appendix:online_sgd_unsketched_case}

When $\Sbs = \Ibs_d$, we have
\begin{equation}
    \E[\xbs \xbs^{\top}] = \sum_{i=1}^K \pi_i \left(\mubs_i \mubs_i^{\top} + \Sigmabs_i\right),
\end{equation}
where $\Sigmabs_i = \E[\xibs_i \xibs_i^{\top}] \preceq \sigma^2 \Ibs_d$, by the sub-Gaussianity assumption. Then,
\begin{equation}
    L \lesssim 1+ \sigma^2 \asymp 1, \quad M = 1 + \sum_{i=1}^K \pi_i \tr \Sigma_i \leq  1 + \sigma^2 d.
\end{equation}

We recall the notation from Theorem \ref{thm:mixture_model_scaling_law}: $T_{\sigma} = c_*/\sigma$ is the population gradient flow time horizon, and $t_i$, $i \in [K]$ are the class recovery times, with $t_i \asymp Z_{K,\alpha} i^{\alpha} \log K$. Moreover, let
\begin{equation}
    f_{K,\alpha}(t) = 
    \begin{dcases}
        (\log K)\left(1- \left(\frac{t}{Z_{K,\alpha} K^{\alpha} \log K}\right)^{1/\alpha-1}\right), & 0 < \alpha < 1, \\
        \log(K (\log K)^2 /t), & \alpha = 1\\
        (\log K)^{2-1/\alpha} \cdot t^{-(1-1/\alpha)}, & \alpha > 1.
    \end{dcases}
\end{equation}

We re-state and prove Theorem \ref{thm:sgd_gmm} from the main text. 

\begin{theorem}[Online SGD Inherits Unsketched Population Phases]
    Fix $\delta \in (0,1)$ and $T \leq T_{\sigma}$. There exists a constant $c > 0$ such that, if
    \begin{equation}
        \eta \lesssim \min\left\{1, \frac{\Rcal(\Wbs(T))}{T(1+T)}, \frac{\delta \min\{\Rcal(\Wbs(T)), \Rcal(\Wbs(T))^2\}}{(1+\sigma^2d)T}\right\},
    \end{equation}
    then, there exist $0 < c_1 < c_2$ such that, with probability at least $1-\delta$, the following holds for every $n$ satisfying $t_n \leq T$:
    \begin{equation}
        \Rcal(\Wbs^{(n)}) \asymp 
        \begin{cases}
            \log K, & 0 \leq t_n \lesssim Z_{K,\alpha} \log K, \\
            f_{K,\alpha}(t_n), & Z_{K,\alpha} \log K \lesssim t_n \leq c_1 Z_{K,\alpha} K^{\alpha} \log K, \\
            K/t_n, & c_2 Z_{K,\alpha} K^{\alpha} \log K \leq t \leq T.
        \end{cases}
    \end{equation}
    for every $n$ satisfying $t_n \leq T$.
\end{theorem}
\begin{proof}
    Immediate from Theorem \ref{thm:mixture_model_scaling_law} and Theorem \ref{thm:gf_to_sgd}. The learning rate condition is obtained from the general rule
    \begin{equation}
        \eta \leq c \min\left\{\frac{1}{L}, \frac{\Rcal(\Wbs(T))}{LT(1+LT)}, \frac{\delta \Rcal(\Wbs(T))^2}{MT}, \frac{\delta \Rcal(\Wbs(T))}{LMT}\right\}
    \end{equation}
    by noticing $L \lesssim 1$, $M \leq 1+\sigma^2 d$.
\end{proof}

The following corollary is immediate.

\begin{corollary}[Explicit Learning Rate Conditions for $\alpha > 1$]\label{cor:lr_condition_online_sgd_unsketched}
    Suppose $\alpha > 1$, $0 < \sigma \leq 1$, and $T = T_{\sigma} \asymp 1/\sigma$. Then, up to constants depending on $\delta$, the learning rate condition in the theorem above is as follows.
    \begin{enumerate}
        \item If $T \lesssim \log K$, it is sufficient to have
        \begin{equation}
            \eta \lesssim \min\left\{1, \sigma^2 \log K, \frac{\sigma \log K}{1+\sigma^2d}\right\}.
        \end{equation}
        \item If $\log K \lesssim T \leq c_1 K^{\alpha} \log K$, it is sufficient to have
        \begin{equation}
            \eta \lesssim \min\left\{1, (\log K)^{2-1/\alpha}\sigma^{3-1/\alpha}, \frac{(\log K)^{4-2/\alpha} \sigma^{3-2/\alpha}}{1+\sigma^2 d}, \frac{(\log K)^{2-1/\alpha} \sigma^{2-1/\alpha}}{1+\sigma^2 d}\right\}.
        \end{equation}
        \item If $c_2 K^{\alpha} \log K \leq T$, it is sufficient to have
        \begin{equation}
            \eta \lesssim \min\left\{1, K\sigma^3, \frac{K^2\sigma^3}{1+\sigma^2d}, \frac{K\sigma^2}{1+\sigma^2d}\right\}.
        \end{equation}
    \end{enumerate}
\end{corollary}

\subsection{Empirical Sketch Case}\label{appendix:online_sgd_empirical_sketch_case}

Let $\widehat{\Sbs} \in \R^{m \times d}$ be the sketch matrix obtained from empirical PCA on $N$ samples, as in Section \ref{section:compute_optimal_sl}. Then, conditional on $\hat{\Sbs}$,
\begin{equation}
    L = \norm*{\E_{\xbs}[\widehat{\Sbs}\xbs \xbs^{\top} \widehat{\Sbs}^{\top}]}_{\text{op}} \leq \norm*{\E_{\xbs}[\xbs \xbs^{\top}]} \leq 1 + \sigma^2, 
\end{equation}
\begin{equation}
    M = \E_{\xbs}\norm{\widehat{\Sbs} \xbs}^2 = \sum_{i=1}^K \pi_i \E_{\xibs_i}[\norm{\widehat{\Sbs}(\mubs_i + \xibs_i)}^2] = \sum_{i=1}^K \pi_i \norm{\widehat{\Sbs}\mubs_i}^2 + \tau^2 m \leq 1 + \sigma^2 m.
\end{equation}

We are ready to prove Theorem \ref{thm:sgd_sketched} from the main text.

\begin{theorem}[Online SGD after Empirical PCA]
    Assume $\alpha >1$ and $m \leq K/2$. Fix $\delta \in (0,1)$ and $T \leq T_{\sigma}$. Assume the sample used to construct $\hat{\Sbs}$ is independent of the samples used in online SGD. Suppose $\eps_{\text{PCA}} \leq c_0 \min\{\pi_m, \sigma \log K\}$ for a sufficiently small $c_0 > 0$, where $\eps_{\text{PCA}}$ is defined in \eqref{eq:eps_pca_def}. If
    \begin{equation}
        \eta \lesssim \min\left\{1, \frac{\Rcal_{\widehat{\Sbs}}(T)}{T(1+T)}, \frac{\delta \min\{\Rcal_{\widehat{\Sbs}}(T), \Rcal_{\widehat{\Sbs}}(T)^2\}}{(1+\sigma^2 m)T}\right\},
    \end{equation}
    then, with probability at least $1-\delta$,
    \begin{equation}
        \Rcal_{\widehat{\Sbs}}(\Wbs^{(n)}) \asymp 
        \begin{cases}
            \log K, & 0 \leq t_n \lesssim \log K \\
            m^{1-\alpha}\log K + (\log K)^{2-1/\alpha} \cdot t_n^{-(1-1/\alpha)}, & \log K \lesssim t_n \leq T
        \end{cases}
    \end{equation}
    uniformly for all $n$ satisfying $t_n \leq T$.
\end{theorem}
\begin{proof}
    Follows from Theorem \ref{thm:emp_pca_scaling_law} and Corollary \ref{cor:general_lr_condition}.
\end{proof}

\section{The Irreducible Risk}\label{appendix:irreducible risk}
While we lack the machinery to establish the exact behaviour of population gradient flow beyond the time horizon $T$, we can find tight bounds for the infimum of the population risk in the case where the intra-class noise in \eqref{eq:target} is isotropic Gaussian, i.e., $\xibs \sim \Ncal(\zerobs, \sigma^2 \Ibs_d)$ with $\sigma^2 > 0$. Under this Gaussian mixture model, define
\begin{equation}
    \Rcal_{K,\pibs}^*(\sigma) := \inf_{\Wbs \in \R^{d \times K}} \Rcal(\Wbs).
\end{equation}
As a preliminary, we confirm that population gradient flow converges to a global minimizer of the risk --- which is convex in $\Wbs$ --- and show that the weights remain within the subspace where their sum is zero. 
\begin{proposition}\label{prop:pop_gf_convergence}
    Let $\Wcal = \{\Wbs \in \R^{d \times K}: \sum_{i=1}^K \wbs_i = 0\}$. Then, over $\Wcal$, there exists a unique minimizer $\Wbs^*$ of $\Rcal$. Moreover, under population gradient flow initialized at zero, $\Wbs(t) \to \Wbs^*$ with respect to $||\cdot||_F$ and
    \begin{equation}\label{eq:pop_gf_conv_rate}
        \Rcal(\Wbs(t)) - \Rcal(\Wbs^*) \leq \frac{||\Wbs(0) - \Wbs^*||_F^2}{2t}
    \end{equation}
    for all $t > 0$.
\end{proposition}
\begin{proof}
    In Remark \ref{rem:risk_is_convex}, we argue that $\Rcal$ is convex in $\Wbs$. In fact, $\Rcal$ is strictly convex on $\Wcal$. Indeed, suppose that $\nabla^2 \Rcal(\Wbs)[\Vbs,\Vbs] = 0$ for some $\Vbs \in \Wbs$. Then, for almost every $\xbs \in \R^d$, we have $\inner{\vbs_i, \xbs} = \inner{\vbs_j,\xbs}$ for all $i,j \in [K]$. Since $\sigma^2 > 0$, it must be that $\vbs_i = \vbs_j$. But since $\Vbs \in \Wbs$, it must be that $\Vbs = 0$. 

    Next, we prove coercivity of $\Rcal$ on $\Wbs$ to establish the existence of a minimizer. For this, given $\Vbs \in \Wcal$, define
    \begin{equation}
        c(V) := \E_{\xbs,y} \left[\max_{j \in [K]} \inner{\vbs_j - \vbs_y, \xbs}\right].
    \end{equation}
    For all nonzero $\Vbs \in \Wcal$, we have $c(\Vbs) > 0$. Indeed, there exists $i,j$ such that $v_j - v_i \neq 0$. Then,
    \begin{equation}
        c(\Vbs) \geq \pi_i \E_{\xbs|y=i}\left[\left(\inner{\vbs_j -\vbs_i, \xbs}\right)_+\right] > 0.
    \end{equation}
    Furthermore, one can show that $\Vbs \mapsto c(\Vbs)$ is $2\E_{\xbs}[||\xbs||]$-Lipschitz with respect to $||\cdot||_F$. Hence, by continuity and compactness, we can define
    \begin{equation}
        c_0 := \inf_{\substack{V \in \Wbs \\ ||\Vbs||_F = 1}} c(\Vbs) > 0.
    \end{equation}
    With this in hand, notice that
    \begin{equation}
        \Rcal(t\Vbs) = \E_{\xbs,y} \left[\log\left(\sum_{j=1}^K \exp\left(t\inner{\vbs_j-\vbs_y,\xbs}\right)\right)\right] \geq t\E_{\xbs,y} \left[\max_j \inner{\vbs_j-\vbs_y,\xbs}\right] = tc(\Vbs).
    \end{equation}
    For nonzero $\Wbs \in \Wcal$, taking $t = ||\Wbs||_F$, $\Vbs = \Wbs/||\Wbs||_F$ yields
    \begin{equation}
        \Rcal(\Wbs) \geq c_0 ||\Wbs||_F.
    \end{equation}
    Hence, $\Rcal$ is coercive on $\Wcal$. This, combined with the continuity of $\Rcal$, ensures the existence of a minimizer $\Wbs^*$. To see that this is in fact a global minimizer for $\Rcal$ on the full domain $\R^{d \times K}$, we leverage the translation-invariance of the softmax. For any $\Wbs \in \R^{d \times K}$, let $\Wbs - \bar{\wbs}$ denote the matrix that results from subtracting the mean $\frac{1}{K}\sum_{i=1}^K \wbs_i$ from each column of $\Wbs$. Then, $\Wbs - \bar{\wbs} \in \Wcal$ and $\Rcal(\Wbs - \bar{\wbs}) = \Rcal(\Wbs)$.

    The iterates $(\Wbs(t))_{t \geq 0}$ of population gradient flow $\dot{\Wbs} = -\nabla \Rcal(\Wbs)$ with initialization $\Wbs(0) = 0$ lie in $\Wcal$ for all $t \geq 0$. This can be seen by summing both sides of \eqref{eq:risk_gradient} over all classes. Convexity gives
    \begin{equation}
        \Rcal(\Wbs(t)) - \Rcal(\Wbs^*) \leq \inner{\nabla \Rcal(\Wbs(t)), \Wbs(t) - \Wbs^*}.
    \end{equation}
    Since $\dot{\Wbs}(t) = -\nabla \Rcal(\Wbs(t))$, we have
    \begin{equation}
        \frac{\dee}{\dee t} \frac{1}{2} \norm{\Wbs(t) - \Wbs^*}_F^2 = -\inner{\nabla \Rcal(\Wbs(t)), \Wbs(t) - \Wbs^*}_F \leq -\left(\Rcal(\Wbs(t)) - \Rcal(\Wbs^*)\right).
    \end{equation}
    Integrating from $0$ to $t$ gives
    \begin{equation}
        \int_0^t \left(\Rcal(\Wbs(s)) - \Rcal(\Wbs^*)\right) \dee s \leq \frac{1}{2} \norm{\Wbs(0) - \Wbs^*}_F^2.
    \end{equation}
    The risk is non-increasing along gradient flow, so
    \begin{equation}
        t\left(\Rcal(\Wbs(t)) - \Rcal(\Wbs^*)\right) \leq \frac{1}{2}\norm{\Wbs(0) - \Wbs^*}_F^2.
    \end{equation}
    Hence, $\Rcal(\Wbs(t)) \to \Rcal(\Wbs^*)$. Moreover, since $\norm{\Wbs(t) - \Wbs^*}_F^2$ is non-increasing (and therefore bounded), it must converge. By strict convexity, it converges to zero.
\end{proof}

To find tight bounds on $\Rcal_{K,\pibs}^*(\sigma)$, the natural starting point is to consider the Bayes-optimal classifier. The Bayesian posterior of class $i$ given an input $\xbs$ and our prior $\pibs$ over classes is 
\begin{equation}
    p_i^*(\xbs) = \frac{\pi_i \exp(\inner{\xbs,\mubs_i}/\sigma^2)}{\sum_{k=1}^K \pi_k \exp(\inner{\xbs,\mubs_k}/\sigma^2)}.
\end{equation}
Recall our logistic regression model with class probabilities $p_i(\xbs)$, $i \in [K]$, parametrized by weights $\Wbs$. We have a natural lower bound for its cross-entropy risk via
\begin{equation}
    \begin{split}
    \Rcal(\Wbs) &= \E_{\xbs} \left[H(p^*(\xbs), p(\xbs))\right] \\
    &= \E_{\xbs}\left[H(p^*(\xbs)) + \mathrm{KL}\left(p^*(\xbs) \mid\mid p(\xbs)\right) \right] \\
    &\geq \E_{\xbs}\left[H(p^*(\xbs))\right] \\
    &= \Rcal^*_{\text{Bayes}}.
    \end{split}
\end{equation}

Note that the posterior probabilities $p_i^*$ depend only on the vector $\zbs \in \R^K$ such that $\zbs_i = \inner{\xbs,\mubs_i}$, $i \in [K]$. We abuse notation to write the true posterior as $p_i^*(\zbs)$. Then, the Bayes risk is
\begin{equation}\label{eq:bayes_risk}
    \Rcal_{\text{Bayes}}^* = -\sum_{i=1}^K \pi_i \E_i \left[\log p_i^*(\zbs)\right],
\end{equation}
where $\E_i$ denotes the expectation over $\zbs$ given the label $y = i$. There is no closed form expression for the above, so we settle for deriving tight bounds.

Denote the scalar softplus function by $s(u) := \log(1+e^u)$, and the log posterior ratio
\begin{equation}\label{eq:log_posterior_ratio}
    \ell_{ji}(\zbs) := \log \frac{p_j^*(\zbs)}{p_i^*(\zbs)} = \log \frac{\pi_j}{\pi_i} + \frac{z_j - z_i}{\sigma^2}.
\end{equation}
Then, notice that 
\begin{equation}
    p_i^*(\zbs) = \frac{\pi_i \exp(z_i/\sigma^2)}{\sum_{k=1}^K \pi_k \exp(z_k/\sigma^2)} = \frac{1}{1+\sum_{k \neq i} \frac{\pi_k}{\pi_i} \exp(\frac{z_k-z_i}{\sigma^2})} = \frac{1}{1 + \sum_{k\neq i} \exp(\ell_{ki})}.
\end{equation}
Hence,
\begin{equation}\label{eq:loss_term_log_posterior}
    -\log p_i^*(\zbs) = \log\left(1 + \sum_{k \neq i} \exp(\ell_{ki})\right),
\end{equation}
a log-sum-exp of a Gaussian vector with correlated entries. In the binary case with classes $i \neq j$, we have the equivalence $-\log p_i^*(\zbs) = s(\ell_{ji})$. This motivates the study of the pairwise Bayes conditional entropy contribution
\begin{equation}
    B_{ij}(\sigma) := \pi_i \E_i[s(\ell_{ji})] + \pi_j \E_j [s(\ell_{ij})].
\end{equation}
That is, we study the contribution of competitor class $j$ when the true label is $i$, and vice versa.
\begin{lemma}\label{lem:Bij_bound}
    Let $b_{ij} = \log \pi_i - \log \pi_j$ and $h(u) = e^{-u/2}s(u)$. Then,
    \begin{equation}\label{eq:exact_B_ij}
        B_{ij} = \frac{\sigma \sqrt{\pi_i \pi_j}}{\sqrt{\pi}} \exp\left(-\frac{1}{4\sigma^2} - \frac{\sigma^2 b_{ij}^2}{4}\right) \int_{\R} h(u) e^{-\sigma^2 u^2/4} \cosh\left(\frac{\sigma^2 b_{ij} u}{2}\right) \dee u.
    \end{equation}
    In particular, if $\sigma^2 |b_{ij}| \leq c$ for some fixed $c \in (0,1)$,
    \begin{equation}\label{eq:approximate_B_ij}
        B_{ij}(\sigma) = 2\sqrt{\pi} \sigma \sqrt{\pi_i \pi_j} e^{-\frac{1}{4\sigma^2}}e^{-\sigma^2 b_{ij}^2/4} (1+ O(\sigma^2(1+\sigma^2 b_{ij}^2)).
    \end{equation}
\end{lemma}
\begin{proof}
    By its definition \eqref{eq:log_posterior_ratio}, we have $\ell_{ji} \sim \Ncal(-b_{ij}-\frac{1}{\sigma^2}, \frac{2}{\sigma^2})$. Hence,
    \begin{equation}
        \pi_i \E_i[s(\ell_{ji})] = \frac{\sigma \pi_i}{2\sqrt{\pi}} e^{-1/(4\sigma^2)} \int_{\R} s(u) e^{-(u+b_{ij})/2} e^{-\sigma^2 (u+b_{ij})^2/4} \dee u
    \end{equation}
    and similarly,
    \begin{equation}
        \pi_j \E_j[s(\ell_{ij})] = \frac{\sigma \pi_j}{2\sqrt{\pi}} e^{-1/(4\sigma^2)} \int_{\R} s(u) e^{-(u-b_{ij})/2} e^{-\sigma^2 (u-b_{ij})^2/4} \dee u.
    \end{equation}
    The first part of the result follows from the identity $\pi_i e^{-b_{ij}/2} = \pi_j e^{b_{ij}/2} = \sqrt{\pi_i \pi_j}$, along with the definitions of $h$ and $\cosh$. 

    For the second part, let
    \begin{equation}
        I_{ij} := \int_{\R} h(u) e^{-\sigma^2 u^2/4} \cosh\left(\frac{\sigma^2 b_{ij}u}{2}\right) \dee u.
    \end{equation}
    Notice that, via change of variable and integration by parts,
    \begin{equation}
        \begin{split}
            \int_{\R} h(u) \dee u &= \int_{\R} e^{-u/2} \log(1+e^u) \dee u \\
            &= \int_0^{\infty} x^{-3/2} \log(1+x) \dee x \\
            &= 2\int_0^{\infty} \frac{x^{-1/2}}{1+x} \dee x \\
            &= 4 \int_0^{\infty} \frac{1}{1+t^2}  \\
            &= 4 \arctan t \big|_{0}^{\infty} \\
            &= 2\pi.
        \end{split}
    \end{equation}
    On the other hand, we decompose
    \begin{equation}
        e^{-\sigma^2u^2/4} \cosh(\sigma^2 b_{ij}u/2) - 1 = (e^{-\sigma^2u^2/4}-1) + e^{-\sigma^2u^2/4}\left(\cosh(\sigma^2 b_{ij}u/2)-1\right),
    \end{equation}
    so we can bound
    \begin{equation}
        \begin{split}
            \left|e^{-\sigma^2u^2/4}\cosh(\sigma^2 b_{ij}u/2)-1\right| &\leq |e^{-\sigma^2u^2/4}-1| + |\cosh(\sigma^2 b_{ij}u/2)-1| \\
            &\lesssim \sigma^2 u^2 + \sigma^4 b_{ij}^2 u^2 e^{\sigma^2 |b_{ij}||u|/2},
        \end{split}
    \end{equation}
    where the last line follows from the inequalities $1-e^{-x} \leq x$ and $\cosh x -1  \leq \frac{1}{2}x^2 e^{|x|}$ --- the latter of which can be seen via Taylor expansion.
    Then, the error associated with approximating $I_{ij}$ by $2\pi$ is
    \begin{equation}
        \left|\int_{\R} h(u) \left[e^{-\sigma^2 u^2/4} \cosh\left(\frac{\sigma^2 b_{ij}u}{2}\right)-1\right] \dee u\right| \lesssim \sigma^2 \int_{\R} h(u) u^2 \dee u + \sigma^4 b_{ij}^2 \int_{\R} h(u)u^2 e^{\sigma^2|b_{ij}||u|/2} \dee u.
    \end{equation}
    Note that, by definition,
    \begin{equation}
        h(u) \lesssim
        \begin{cases}
            ue^{-u/2}, & \text{as } u \to +\infty, \\
            e^{u/2}, & \text{as } u \to -\infty.
        \end{cases}
    \end{equation}
    Hence, both of the above integrals are finite since $\sigma^2 |b_{ij}| < 1$. We can conclude 
    \begin{equation}
        I_{ij} = 2\pi + O(\sigma^2 + \sigma^4 b_{ij}^2),
    \end{equation}
    which we can plug into \eqref{eq:exact_B_ij} to obtain \eqref{eq:approximate_B_ij}.
\end{proof}

Next, we aggregate the binary Bayes conditional entropies:
\begin{equation}
    P(\sigma) := \sum_{i < j}^K B_{ij}(\sigma),
\end{equation}
and relate this quantity to the Bayes risk.
\begin{proposition}\label{prop:bayes_risk_lower_bound}
    Let 
    \begin{equation}
        T(\sigma) := e^{-\frac{5}{16\sigma^2}} \left(\sum_{i=1}^K \pi_i^{1/2}\right)\left(\sum_{i=1}^K \pi_i^{1/4}\right)^2.
    \end{equation}
    Then, for all $\sigma > 0$,
    \begin{equation}
        \Rcal_{\text{Bayes}}^* \geq P(\sigma) - T(\sigma).
    \end{equation}
\end{proposition}
In particular, this gives us a lower bound on $\Rcal_{K,\pibs}^*(\sigma)$. To prove this result, we require the following technical lemma.

\begin{lemma}\label{lem:softplus_min_inequalities}
    For $u_1,\dots,u_m \in \R$, 
    \begin{equation}
        \sum_{j=1}^m s(u_j) - \log \left(1 + \sum_{j=1}^m e^{u_j}\right) \leq \sum_{j < k} s\left(\min\{u_j,u_k\}\right). 
    \end{equation}
    Moreover,
    \begin{equation}
        s(\min\{u,v\}) \leq 2e^{(u+v)/4}.
    \end{equation}
\end{lemma}
\begin{proof}
    For the first part, we proceed by induction. The base case $m = 1$ is trivial. Now, suppose $m > 1$ and the result holds up to $m-1$. Define
    \begin{equation}
        F_m := \sum_{j=1}^m s(u_j) - \log\left(1 + \sum_{j=1}^m e^{u_j}\right). 
    \end{equation}
    Let $S = \sum_{j=1}^{m-1} e^{u_j}$ and $x = e^{u_m}$. Then,
    \begin{equation}
        \begin{split}
            F_m &= \sum_{j=1}^{m-1} s(u_j) - \log\left(1 + \sum_{j=1}^{m-1} e^{u_j}\right) + s(u_m) - \log\left(\frac{1+\sum_{j=1}^m e^{u_j}}{1+\sum_{j=1}^{m-1} e^{u_j}}\right) \\
            &= F_{m-1} + \log(1+x) + \log(1+S) - \log(1+S+x) \\
            &= F_{m-1} + \log\left(\frac{(1+S)(1+x)}{1+S+x}\right) \\
            &= F_{m-1} + \log\left(1 + \frac{Sx}{1+S+x}\right) \\
            &\leq F_{m-1} + \log(1+\min\{S,x\}).
        \end{split}
    \end{equation}
    Subsequently, we can write
    \begin{equation}
        1 + \min\{S,x\} \leq \prod_{j=1}^{m-1} \left(1 + \min\{e^{u_j},x\}\right)
    \end{equation}
    since
    \begin{equation}
        \prod_{j=1}^{m-1} \left(1 + \min\{e^{u_j},x\}\right) \geq 1 + \sum_{j=1}^{m-1} \min\{e^{u_j},x\} \geq 1 + \min\{S,x\}.
    \end{equation}
    This proves the first part. 
    
    For the second part, we split into two cases. Let $m = \min\{u,v\}$. If $m \leq 0$, we have
    \begin{equation}
        s(m) = \log(1+e^m) \leq e^m \leq e^{(u+v)/4},
    \end{equation}
    using the elementary inequality $\log(1+x) \leq x$ for $x \geq -1$. On the other hand, if $m > 0$, we have
    \begin{equation}
        s(m) = \log(1+e^m) \leq 1 +m \leq 2e^{m/2} \leq 2e^{(u+v)/4}.
    \end{equation}
\end{proof}

\begin{proof}[Proof of Proposition \ref{prop:bayes_risk_lower_bound}]
    \eqref{eq:bayes_risk} and \eqref{eq:loss_term_log_posterior} imply that the Bayes risk can be expressed as
    \begin{equation}
        \Rcal^*_{\text{Bayes}} = \sum_{i=1}^K \pi_i \E_i \left[\log \left(1 + \sum_{k \neq i} \exp(\ell_{ki}) \right)\right]. 
    \end{equation}
    Focusing on each term individually, we have, by Lemma \ref{lem:softplus_min_inequalities},
    \begin{equation}\label{eq:risk_term_lower_bound}
        \log\left(1 + \sum_{k \neq i} \exp(\ell_{ki}) \right) \geq \sum_{k \neq i} s(\ell_{ki}) - \sum_{\substack{j < k \\ j,k \neq i}} s\left(\min\{\ell_{ji}, \ell_{ki}\}\right).
    \end{equation}
    Notice that
    \begin{equation}
        \sum_{i=1}^K \pi_i \sum_{k \neq i} \E_i [s(\ell_{ki})] = \sum_{i < j} B_{ij}(\sigma) = P(\sigma).
    \end{equation}
    It remains to handle the rightmost term in \eqref{eq:risk_term_lower_bound}. By Lemma \ref{lem:softplus_min_inequalities},
    \begin{equation}
       \pi_i \E_i \left[s(\min\{\ell_{ji}, \ell_{ki}\})\right] \leq 2 \pi_i \E_i \left[e^{(\ell_{ji} + \ell_{ki})/4}\right] = 2 \pi_i^{1/2} (\pi_j \pi_k)^{1/4} e^{-5/(16\sigma^2)}.
    \end{equation}
    Putting things together,
    \begin{equation}
        2 \sum_{i=1}^K \pi_i^{1/2} \sum_{\substack{j < k \\ j,k \neq i}} (\pi_j \pi_k)^{1/4} e^{-5/(16\sigma^2)} \leq e^{-5/(16\sigma^2)} \left(\sum_{i=1}^K \pi_i^{1/2}\right) \left(\sum_{j=1}^K \pi_j^{1/4}\right)^2 =: T(\sigma).
    \end{equation}
\end{proof}

To obtain an upper bound, we construct a classifier that fits in our model class (i.e., it has no biases) and approximately achieves the Bayes risk. Prior to stating it, we give some motivation. Recall that the Bayesian posterior satisfies
\begin{equation}
    \log \frac{p_j^*(\zbs)}{p_i^*(\zbs)} = \log \frac{\pi_j}{\pi_i} + \frac{z_j - z_i}{\sigma^2}.
\end{equation}
As a result, the Bayes-optimal decision boundary between classes $i$ and $j$ is 
\begin{equation}
    \inner{\xbs, \mubs_j - \mubs_i} = \sigma^2 \log \frac{\pi_i}{\pi_j}.
\end{equation}
This is achieved by taking a logistic regression model with weights $\wbs_i = \frac{1}{\sigma^2} \mubs_i$ and biases $\log \pi_i$. However, we do not have biases in our model class and thus cannot achieve the Bayes-optimal classifier unless the priors are equal. Instead, we settle for a modification of the Bayes-optimal weights $\wbs_i$ to mimic the behaviour of the Bayes classifier on the line segment between $\mubs_i$ and $\mubs_j$.

Suppose that we take weights $\wbs_i = r_i \mubs_i$, $i \in [K]$, for some positive scalars $r_i$. Then, the decision boundary between classes $i$ and $j$ takes the form
\begin{equation}
    \inner{\xbs,r_j \mubs_j - r_i \mubs_i} = 0.
\end{equation}
For $\lambda \in [0,1]$, we have
\begin{equation}
    \inner{(1-\lambda) \mubs_i + \lambda \mubs_j, r_j \mubs_j - r_i \mubs_i} = 0
\end{equation}
if and only if
\begin{equation}
    \lambda = \frac{r_i}{r_i + r_j}.
\end{equation}
So, decision boundary crosses the line segment between $\mubs_i$ and $\mubs_j$ at the point $\frac{1}{r_i + r_j}(r_j \mubs_i + r_i \mubs_j)$. Now, compare this to the Bayes rule: we have
\begin{equation}
    \inner{\xbs, \mubs_j - \mubs_i} = \inner{(1-\lambda) \mubs_i + \lambda \mubs_j, \mubs_j - \mubs_i} = \sigma^2 \log \frac{\pi_i}{\pi_j}
\end{equation}
if and only if
\begin{equation}
    \lambda = \frac{1}{2} + \frac{\sigma^2}{2} \log \frac{\pi_i}{\pi_j} =: \lambda_{\text{Bayes}}.
\end{equation}
Now, if for all $k \in [K]$ we take $r_k := \frac{1}{\sigma^2} + \Delta_k$ for some $\Delta_k \in \R$, a bivariate Taylor expansion about $(\Delta_i, \Delta_j) = (0,0)$ gives
\begin{equation}
    \frac{r_i}{r_i + r_j} = \frac{1}{2} + \frac{\sigma^2}{4}(\Delta_i - \Delta_j) + o(\sigma^2 |\Delta_i - \Delta_j|).
\end{equation}
As a result, choosing $\Delta_i = 2\log \pi_i$, $\Delta_j = 2 \log \pi_j$ gives an approximation to $\lambda_{\text{Bayes}}$ that is accurate up to an additive error of $o(\sigma^2)$. This leads us to our bias-free logistic regression model.

Define
\begin{equation}\label{eq:beta_r_defs}
    \beta_i := \log \pi_i - \frac{1}{2}(\log \pi_1 + \log \pi_K), \quad B_{\pibs} := \max_i |\beta_i| = \frac{1}{2} \log \frac{\pi_1}{\pi_K}, \quad r_i := \frac{1}{\sigma^2} + 2\beta_i,
\end{equation}
and take the model weights to be $\tilde{\wbs}_i = r_i \mubs_i$. Let
\begin{equation}\label{eq:J_C_tau_defs}
    J(m,\tau) := \E_{t \sim \Ncal(m,\tau^2)}\left[s(-t)\right], \quad C_{ij} := \pi_i J(r_i, \tau_{ij}) + \pi_j J(r_j, \tau_{ij}), \quad \tau_{ij} := \sigma \sqrt{r_i^2 + r_j^2}.
\end{equation}
Then, the risk of the proposed weight configuration is
\begin{equation}
    \begin{split}
        \Rcal(\tilde{\Wbs}) &= -\sum_{i=1}^K \pi_i \E_i\left[\log p_i(\zbs)\right] \\
        &= \sum_{i=1}^K \pi_i \E_i \left[\log\left(1 + \sum_{j \neq i} \exp(r_jz_j - r_iz_i)\right)\right] \\
        &\leq \sum_{i=1}^K \pi_i \sum_{j \neq i} \E_i \left[s(r_j z_j - r_i z_i)\right] \\
        &= \sum_{i=1}^K \pi_i \sum_{j \neq i} \E_{t \sim \Ncal(r_i, \sigma^2(r_i^2 + r_j^2))}[s(-t)] \\
        &= \sum_{i < j} C_{ij}.
    \end{split}
\end{equation}
Then, we can control the risk by precisely characterizing each of the $C_{ij}$.
\begin{lemma}\label{lem:C_ij_bound}
    There exists $c \in (0,1)$ such that if $0 < \sigma < c$ and $\sigma B_{\pibs} < c$, then, for all indices $i,j$ such that $i \neq j$, 
    \begin{equation}
        C_{ij} = 2 \sqrt{\pi} \sigma \sqrt{\pi_i \pi_j} e^{-\frac{1}{4\sigma^2}}(1+\eps_{ij}^U), \quad \text{with} \quad |\eps_{ij}^U| \lesssim \sigma^2 (1 + B_{\pibs} + B_{\pibs}^2).
    \end{equation}
\end{lemma}
\begin{proof}
    First, we compute
    \begin{equation}\label{eq:J_integral}
        J(m,\tau) := \E_{t \sim \Ncal(m,\tau^2)}[s(-t)] = \frac{e^{-m^2/(2\tau^2)}}{\sqrt{2\pi}\tau} \int_{\R} \exp\left(\frac{m}{\tau^2}t -\frac{t^2}{2\tau^2}\right)s(-t) \dee t.
    \end{equation}
    using the Gaussian density. Focusing on the integral on the RHS, we define
    \begin{equation}
        I(\theta,\rho) := \int_{\R} \exp\left(\theta t - \frac{\rho t^2}{2}\right)s(-t) \dee t.
    \end{equation}
    Consider the special case
    \begin{equation}
        I(\theta, 0) = \int_{\R} e^{\theta t} \log(1+e^{-t}) \dee t = \frac{1}{\theta}\int_{\R} \frac{e^{\theta t}}{1+e^t} \dee t = \frac{1}{\theta} \int_{0}^{\infty} \frac{x^{\theta-1}}{1+x} \dee x = \frac{1}{\theta} B(\theta, 1-\theta) = \frac{\pi}{\theta \sin(\pi \theta)},
    \end{equation}
    where the first step follows from integration by parts, the second uses a change of variable, and the remaining steps follow from properties of the Beta function. As a function of $\theta$, the above has bounded derivative on any compact sub-interval of $(0,1)$ and therefore we can expand up to first-order about $\theta = 1/2$:
    \begin{equation}
        I(\theta,0) = 2\pi + O(|\theta - 1/2|).
    \end{equation}
    Meanwhile,
    \begin{equation}
        I(\theta,0) - I(\theta,\rho) = \int_{\R} e^{\theta t} \left(1 - e^{-\rho t^2/2}\right)s(-t) \dee t \leq \frac{\rho}{2} \int_{\R} t^2 e^{\theta t} s(-t) \dee t.
    \end{equation}
    The integral resulting from the final inequality is finite, as the integrand decays exponentially both as $t \to -\infty$ and as $t \to +\infty$. Hence,
    \begin{equation}
        I(\theta,\rho) = 2\pi + O\left(|\theta-1/2| + \rho\right).
    \end{equation}
    For the purposes of calculating $C_{ij}$, we are interested in the case $\theta =\frac{r_i}{\tau_{ij}^2}$, $\rho = \frac{1}{\tau_{ij}^2}$. Recalling the definitions \eqref{eq:beta_r_defs} and \eqref{eq:J_C_tau_defs}, we can treat these quantities as bivariate functions of $(\sigma^2\beta_i,\sigma^2\beta_j)$. A Taylor expansion about $(0,0)$ gives
    \begin{equation}
        \frac{r_i}{\tau_{ij}^2} = \frac{1}{2} - \sigma^2 \beta_j + O(\sigma^4 \beta_i^2 + \sigma^4 \beta_j^2) 
    \end{equation}
    and
    \begin{equation}
        \frac{1}{\tau_{ij}} = \frac{\sigma}{\sqrt{2}} \left(1 - \sigma^2 (\beta_i + \beta_j) + O(\sigma^4 B_{\pibs}^2)\right).
    \end{equation}
    Hence, 
    \begin{equation}
        I\left(\frac{r_i}{\tau_{ij}^2}, \frac{1}{\tau_{ij}^2}\right) = 2\pi\left(1 + O(\sigma^2 + \sigma^2 B_{\pibs})\right).
    \end{equation}
    Plugging this into our expression of interest and recalling \eqref{eq:J_integral}, we have 
    \begin{equation}
        \pi_i J(r_i, \tau_{ij}) = \frac{\sqrt{2\pi}}{\tau_{ij}} \pi_i e^{-r_i^2/2\tau_{ij}^2} \left(1 + O(\sigma^2 + \sigma^2 B_{\pibs})\right).
    \end{equation}
    Another Taylor expansion gives
    \begin{equation}
        \begin{split}
            \log \pi_i - \frac{r_i^2}{2\tau_{ij}^2} &= \log \pi_i -\frac{1}{4\sigma^2} - \frac{1}{2}\beta_i + \frac{1}{2}\beta_j + O(\sigma^2\beta_i^2 + \sigma^2\beta_j^2) \\
            &= -\frac{1}{4\sigma^2} + \frac{1}{2}\left(\log \pi_i + \log \pi_j\right) + O(\sigma^2 B_{\pibs}^2).
        \end{split}
    \end{equation}
    Hence,
    \begin{equation}
        \pi_i J(r_i, \tau_{ij}) = \sigma \sqrt{\pi} \sqrt{\pi_i \pi_j} e^{-1/(4\sigma^2)} \left(1 + O(\sigma^2(1+B_{\pibs} + B_{\pibs}^2)\right).
    \end{equation}
    Note that $\pi_j J(r_j, \tau_{ij})$ has the same leading term and error order. This implies the desired result.
\end{proof}
We are now ready to state upper and lower bounds for the best achievable risk under our logistic regression model.
\begin{theorem}\label{thm:inf_risk_upper_lower_bounds}
    Let
    \begin{equation}
        S(\pibs) = \sum_{i < j} \sqrt{\pi_i \pi_j}, \quad P_0(\sigma, \pibs) = 2\sqrt{\pi} \sigma e^{-\frac{1}{4\sigma^2}}S(\pibs), \quad A(\pibs) = \frac{(\sum_{i=1}^K \pi_i^{1/2})(\sum_{i=1}^K \pi_i^{1/4})^2}{S(\pibs)}.
    \end{equation}
    There exist constants $c,C > 0$ such that if $0 < \sigma \leq c$ and $\sigma B_{\pibs} \leq c$, then
    \begin{equation}
        \Rcal_{K,\pibs}^*(\sigma) \geq P_0(\sigma, \pibs) \left(1 - C(\sigma^2 + (\sigma B_{\pibs})^2) - C \frac{e^{-\frac{1}{16\sigma^2}}}{\sigma}A(\pibs)\right),
    \end{equation}
    \begin{equation}
        \Rcal_{K,\pibs}^*(\sigma) \leq P_0(\sigma, \pibs) \left(1 + C\sigma^2 (1+ B_{\pibs} + B_{\pibs}^2)\right).
    \end{equation}
\end{theorem}
\begin{proof}
    For the lower bound, we have $\Rcal_{K,\pibs}^* \geq \Rcal^*_{\text{Bayes}} \geq P - T$ by Proposition \ref{prop:bayes_risk_lower_bound}. Then, by Lemma \ref{lem:Bij_bound},
    \begin{equation}
        P = \sum_{i < j} B_{ij} = 2\sqrt{\pi} \sigma e^{-1/(4\sigma^2)} \left(\sum_{i < j} (\pi_i \pi_j)^{1/2} e^{-\sigma^2 b_{ij}^2/4} (1+\eps_{ij}^B)\right),
    \end{equation}
    with
    \begin{equation}
        |\eps_{ij}^B| \leq C(\sigma^2 + \sigma^4 b_{ij}^2) \leq C(\sigma^2 + 2\sigma^4 B_{\pibs}^2).
    \end{equation}
    We can write 
    \begin{equation}
        \exp(-\sigma^2 b_{ij}^2/4) \geq 1 - \frac{\sigma^2}{4} b_{ij}^2 \geq 1 - \sigma^2 B_{\pibs}^2.
    \end{equation}
    Hence, up to re-defining the constant $C$, we have
    \begin{equation}
        P(\sigma) \geq P_0(\sigma, \pibs) \left(1 - C(\sigma^2 + (\sigma B_{\pibs})^2\right).
    \end{equation}
    With an appropriate choice of $C$, we also have
    \begin{equation}
        \frac{T(\sigma)}{P_0(\sigma,\pibs)} \leq C \frac{e^{-1/(16\sigma^2)}}{\sigma} A(\pibs),
    \end{equation}
    which gives the lower bound result.

    As for the upper bound, our earlier finding $\Rcal(\tilde{\Wbs}) \leq \sum_{i < j} C_{ij}$, along with Lemma \ref{lem:C_ij_bound}, imply
    \begin{equation}
        \Rcal_{K,\pibs}^*(\sigma) \leq P_0(\sigma, \pi)\left[1 + C\sigma^2 (1 + B_{\pibs} + B_{\pibs}^2)\right].
    \end{equation}
\end{proof}


Given our power law prior,
\begin{equation}\label{eq:exact_power_law_prior}
    \pi_i = \frac{i^{-\alpha}}{Z_{K,\alpha}}, \quad Z_{K,\alpha} = \sum_{i=1}^K i^{-\alpha}.
\end{equation}
Then,
\begin{equation}\label{eq:B_S_power_law_prior}
    B_{\pibs} = \frac{\alpha}{2}\log K, \quad S(\pibs) = \frac{1}{2}\left[\frac{Z_{K,\alpha/2}^2}{Z_{K,\alpha}} - 1\right]. 
\end{equation}
For every fixed $\alpha > 0$, standard generalized harmonic sum estimates give
\begin{equation}\label{eq:S_estimate}
    S(\pibs) \asymp
    \begin{cases}
        K, & 0 < \alpha < 1, \\
        K/\log K, & \alpha = 1, \\
        K^{2-\alpha}, & 1 < \alpha < 2, \\
        (\log K)^2, & \alpha = 2, \\
        1, & \alpha > 2.
    \end{cases}
\end{equation}
\begin{corollary}\label{cor:power_law_optimal_risk}
    If $0 < \sigma = o(1/\log K)$ and the prior $\pibs$ satisfies \eqref{eq:exact_power_law_prior} for some $\alpha > 0$, then
    \begin{equation}
        \Rcal_{K,\pibs}^*(\sigma) = \sqrt{\pi} \sigma e^{-\frac{1}{4\sigma^2}}\left(\frac{Z_{K,\alpha/2}^2}{Z_{K,\alpha}}-1\right) (1+o(1)).
    \end{equation}
\end{corollary}
\begin{proof}
    The condition $\sigma = o(1/\log K)$ implies $\sigma B_{\pibs} = o(1)$ and $\sigma^2 B_{\pibs} = o(1)$. Moreover, by H\"older's inequality,
    \begin{equation}
        \sum_{i=1}^K \pi_i^{1/2} \leq \left(K \sum_{i=1}^K \pi_i\right)^{1/2} = K^{1/2}, \quad \sum_{i=1}^K \pi_i^{1/4} \leq K^{3/4} \left(\sum_{i=1}^K \pi_i\right)^{1/4} \leq K^{3/4}.
    \end{equation}
    Hence, $A(\pibs) \lesssim K^{1/2} (K^{3/4})^2 = K^2/S(\pibs) \lesssim K^2$. We can use this to argue 
    \begin{equation}
        \frac{e^{-1/(16\sigma^2)}}{\sigma} A(\pi) = o(1).
    \end{equation}
    Indeed, taking the logarithm of the left-hand side gives at most
    \begin{equation}
        \log \frac{1}{\sigma} - \frac{1}{16\sigma^2} + 2\log K.
    \end{equation}
    Under the assumption $\sigma = o(1/\log K)$, the $O(1/\sigma^2)$ term dominates, which implies the above tends to $-\infty$ as $K \to +\infty$. 

    Hence, by Theorem \ref{thm:inf_risk_upper_lower_bounds}, we can conclude $\sigma = o(1/\log K)$ implies
    \begin{equation}
        \Rcal_{K,\pibs}^*(\sigma) = P_0(\sigma, \pibs)(1+o(1)).
    \end{equation}
    The result then follows from the definition $P_0(\sigma, \pibs) = 2\sqrt{\pi} \sigma e^{-1/(4\sigma^2)} S(\pibs)$ and \eqref{eq:B_S_power_law_prior}.
\end{proof}
Hence, if $0 < \sigma = o(1/\log K)$, we have
\begin{equation}
    \Rcal_{K,\pibs}^*(\sigma) \asymp \sigma e^{-\frac{1}{4\sigma^2}}
    \begin{cases}
        K, & 0 < \alpha < 1, \\
        K/\log K, & \alpha = 1, \\
        K^{2-\alpha}, & 1< \alpha < 2, \\
        (\log K)^2, & \alpha = 2, \\
        1, & \alpha > 2.
    \end{cases}
\end{equation}
We note that the irreducible risk $\Rcal_{K,\pibs}^*(\sigma)$ is of strictly lower order than the risk of population gradient flow stated in Theorem \ref{thm:mixture_model_scaling_law} for $t \leq T$ under the usual $\sigma = o(1/\log K)$ assumption. To see this when $\alpha > 1$, notice that
\begin{equation}
    \exp\left(-\frac{1}{4\sigma^2}\right) = \exp\left(-\frac{1}{4o(\frac{1}{(\log K)^2})}\right) = \exp\left(-\omega((\log K)^2)\right) = K^{-\omega(\log K)}.
\end{equation}
This is of strictly lower order than $\log K$ (the risk during Phase 1) and $K\sigma$ (the risk during Phase 3 for $t \asymp T = c_*/\sigma$). As for Phase 2, we have
\begin{equation}
    \frac{\sigma e^{-1/(4\sigma^2)}K^{2-\alpha}}{(\log K)^{2-1/\alpha} \sigma^{1-1/\alpha}} = \frac{\sigma^{1/\alpha} K^{2-\alpha}}{(\log K)^{2-1/\alpha}} K^{-\omega(\log K)} \to 0.
\end{equation}

\section{Squared Loss in the Idealized Case}\label{appendix:proofs_squared_loss}
In this appendix, we compare our findings in the setting of classification with the cross-entropy loss to training with the squared loss. We consider two instantiations of the latter: casting the classification problem as multi-output regression (Appendix \ref{appendix:multi_output_regression}) and applying the squared loss on top of the softmax activation (Appendix \ref{appendix:squared_loss_softmax}). We highlight the differences in scaling behaviour that result from the change in loss function. For this discussion, we focus entirely on the case $\sigma^2 = 0$ and $\alpha > 1$.

\subsection{Multi-Output Linear Regression}\label{appendix:multi_output_regression}

In this section, we convert the classification problem in \eqref{eq:target} to a multi-output regression problem. Given an input $\xbs$ with classification label $y$, the objective is to predict the one-hot vector $\ebs_y \in \R^K$. We have the loss
\begin{equation}
    \ell_{\mathrm{sq}}(\Wbs;\xbs,\ybs) = \sum_{k=1}^K \big(\1\{y=k\} - \inner{\xbs, \wbs_k}\big)^2,
\end{equation}
where $\Wbs \in \R^{d \times K}$ with columns $\wbs_1,\dots,\wbs_K$. When $\sigma^2 = 0$, the risk is
\begin{equation}
    \Rcal_{\mathrm{sq}}(\Wbs) = \sum_{i=1}^K \pi_i \bigg(1 - 2\theta_{ii} + \sum_{k=1}^K \theta_{ki}^2\bigg).
\end{equation}
The gradient with respect to a parameter vector $\wbs_i \in \R^d$ is then
\begin{equation}
    \nabla_{\wbs_i} \Rcal_{\mathrm{sq}} = -2\pi_i \mubs_i + 2\sum_{j=1}^K \pi_j \theta_{ij} \mubs_j,
\end{equation}
leading to an ODE system for the summary statistics:
\begin{equation}
    \dot{\theta}_{ij} = 2\pi_j \left(\1\{i=j\} -\theta_{ij}\right), \quad i,j \in [K].
\end{equation}
Now, to reduce this to a scalar ODE as we did in Appendix \ref{appendix:noiseless_case}, notice that if we initialize $\Wbs(0) = \zerobs$, we have $\dot{\theta}_{ij}(t) = 0$ for all times $t \geq 0$ whenever $i \neq j$. Hence, it suffices to focus on the (linear) scalar ODE
\begin{equation}
    \dot{\theta}_{ii} = 2\pi_i (1-\theta_{ii})
\end{equation}
and a simplified expression for the risk as a function of only the ``diagonal'' summary statistics:
\begin{equation}\label{eq:squared_risk_logits}
    \Rcal(t) = \sum_{i=1}^K \pi_i (1-\theta_{ii}(t))^2.
\end{equation}
Solving the linear ODE and plugging in the initial condition yields
\begin{equation}
    1 - \theta_{ii}(t) = e^{-2\pi_i t}
\end{equation}
and therefore
\begin{equation}
    \Rcal_{\mathrm{sq}}(t) = \sum_{i=1}^K \pi_i e^{-4\pi_i t}.
\end{equation}
We now proceed to deriving the scaling law. For any constant $c \in (1/K,1)$, if we define $t_i(c) := \inf\{t \geq 0: \theta_{ii}(t) \geq c\}$, we have
\begin{equation}
    t_i(c) = \frac{\log(1/(1-c))}{2\pi_i} \asymp \frac{1}{\pi_i}.
\end{equation}
Compare this to the analogous definition of $t_i$ in Appendix \ref{appendix:noiseless_case}, where we find $t_i \asymp (\log K)/\pi_i$. The difference here is that, in addition to the dynamics of each class evolving separately, the loss and gradient operate on the logits directly, rather than softmax probabilities. With the softmax activation, the sum-to-one constraint means that the small initialization $p_{i|i}(0) = 1/K$ must be overcome.

Now, if we define  $k^*(t) := |\{i \in [K]: t_i(1/2) \leq t\}|$, then 
\begin{equation}
    k^*(t) = \left\lfloor\left(\frac{2t}{\log 2}\right)^{1/\alpha} \right\rfloor \asymp t^{1/\alpha}.
\end{equation}

\begin{proposition}
     Under \eqref{eq:target}, suppose $\sigma^2 = 0$ and $\alpha > 1$. Then, for the population gradient flow on the squared loss \eqref{eq:squared_risk_logits}, there exists a constant $c_1 \in (0,1)$ such that
    \begin{equation}
        \Rcal_{\mathrm{sq}}(t) \asymp
        \begin{cases}
            1, & 0 \leq t \leq t_1 \\
            t^{-(1-1/\alpha)}, & t_1 \leq t \leq c_1t_K.
        \end{cases}
    \end{equation}
    Moreover, $\pi_K e^{-4\pi_K t} \leq \Rcal_{\mathrm{sq}}(t) \leq e^{-4\pi_K t}$ for all $t \geq 0$.
\end{proposition}

\begin{proof}
\underline{Early phase, $t \leq t_1$:} Every summand in $\Rcal_{\mathrm{sq}}(t)$ is non-negative and decreasing and therefore
\begin{equation}
    \Rcal_{\mathrm{sq}}(t) \leq \Rcal_{\mathrm{sq}}(0) = \sum_{i=1}^K \pi_i = 1.
\end{equation}
When $t \leq t_1$, we have
\begin{equation}
    e^{-4\pi_i t} \geq e^{-4\pi_i t_1} = \exp\left(-2 \log 2\right) = 1/4
\end{equation}
Hence, 
\begin{equation}
    \Rcal_{\mathrm{sq}}(t) \geq \frac{1}{4} \sum_{i=1}^K \pi_i = \frac{1}{4}.
\end{equation}
Consequently, $\Rcal_{\mathrm{sq}}(t) \asymp 1$ in the early phase.

\underline{Intermediate phase, $t_1 < t \leq c_1 t_K$:} For every $i > k^*(t)$, we have $\theta_{ii} < 1/2$ and therefore $(1-\theta_{ii}(t))^2 \geq 1/4$. Thus,
\begin{equation}
    \Rcal_{\mathrm{sq}}(t) \geq \frac{1}{4} \sum_{i=k^*(t)+1}^K \pi_i.
\end{equation}
The same integral approximation as in Appendix \ref{appendix:noiseless_case}, which is valid up to $t = c_1t_K$ for constant $c_1 \in (0,1)$, gives
\begin{equation}
    \sum_{i=k^*(t)+1}^K i^{-\alpha} \asymp (k^*(t))^{1-\alpha} \asymp t^{-(1-1/\alpha)}.
\end{equation}
Meanwhile,
\begin{equation}
    \sum_{i=1}^{k^*(t)} \pi_i e^{-4\pi_i t} = \sum_{i=1}^{k^*(t)} (\pi_it)e^{-4\pi_i t} \leq \sum_{i=1}^{k^*(t)} \frac{1}{4et} \asymp \frac{k^*(t)}{t} \asymp t^{-(1-1/\alpha)},
\end{equation}
where we have used the inequality $xe^{-4x} \leq 1/(4e)$ for all $x \geq 0$. Hence, 
\begin{equation}
    \Rcal_{\mathrm{sq}}(t) \asymp t^{-(1-1/\alpha)}
\end{equation}
in this phase.

\underline{Convergence phase, $t \geq t_K$:} Since $\pi_i \geq \pi_K$ for every $i$, we can write
\begin{equation}
    \Rcal(t) \leq e^{-4\pi_K t} \sum_{i=1}^K \pi_i = e^{-4\pi_K t}.
\end{equation}
Moreover, focusing only on the last class, which converges at the slowest rate, we have
\begin{equation}
    \Rcal(t) \geq \pi_K e^{-4\pi_K t}.
\end{equation}
Hence,
\begin{equation}
    \pi_K e^{-4\pi_K t} \leq \Rcal(t) \leq e^{-4\pi_K t}.
\end{equation}
\end{proof}

Multi-output linear regression therefore reproduces the intermediate exponent $1-1/\alpha$, but through a different mechanism from multiclass logistic regression. Each regression coordinate converges exponentially and independently, and the aggregate power law results from summing a power-law distribution of exponential relaxation rates. The initial $\log K$ plateau is replaced by a constant order plateau (for a shorter window of time), and the terminal $K/t$ phase is replaced by exponential convergence.



\subsection{Squared Loss with Softmax Activation}\label{appendix:squared_loss_softmax}

The latter subsection does not constitute a fair comparison to cross-entropy training, since no conditional probability distribution over the classes is parametrized. The model trained by squared loss on the logits cannot be evaluated using the cross-entropy loss, precluding a direct quantitative comparison to Appendix \ref{appendix:noiseless_case}. In this subsection, we consider the squared loss applied on top of the softmax activation:
\begin{equation}\label{eq:squared_loss}
    \ell_{\mathrm{MSE}}(\Wbs; \xbs, y) = \frac{1}{2} \sum_{k=1}^K \left(\1\{y=k\} - p_k(\xbs)\right)^2.
\end{equation}
When $\sigma^2 = 0$, we have the training loss
\begin{equation}\label{eq:square_risk}
    \Rcal_{\mathrm{MSE}}(\Wbs) = \frac{1}{2} \sum_{i=1}^K \pi_i \sum_{j=1}^K \left(\1\{j=i\} - p_{j|i}\right)^2.
\end{equation}
We proceed the same way we did in Section \ref{appendix:noiseless_case}. The gradients with respect to weights take the form
\begin{equation}
    \nabla_{\wbs_k} \Rcal_{\mathrm{MSE}} = -\sum_{i=1}^K \pi_i \sum_{j=1}^K p_{j|i} \left(\1\{j=i\} - p_{j|i}\right)\left(\1\{j=k\} - p_{k|i} \right)\mubs_i.
\end{equation}
As we did in the analyses of cross-entropy training dynamics, we track the summary statistic dynamics
\begin{equation}
    \begin{split}
        \dot{\theta}_{ki} &= \pi_i \sum_{j=1}^K p_{j|i} \left(\1\{j=i\} - p_{j|i}\right)\left(\1\{j=k\} - p_{k|i} \right) \\
        &= \pi_i p_{k|i}\left(\1\{i=k\} - p_{i|i} - p_{k|i} + \sum_{j=1}^K p_{j|i}^2\right)
    \end{split}
\end{equation}
It is straightforward to see that $\sum_{k=1}^K \dot{\theta}_{ki} = 0$ by the softmax sum-to-one constraint. Moreover, the dynamics of each $\theta_{ki}$, for $k \neq i$, depend on $k$ only through $p_{k|i}$. Since the $p_{k|i}$ all start from the same initialization at $1/K$, we identify the same symmetry as we did in the idealized case for cross-entropy training (Lemma \ref{lem:conservation_law}). Namely, for all $k \neq i$,
\begin{equation}
    \theta_{ki} = -\frac{1}{K-1} \theta_{ii}, \quad p_{k|i} = \frac{1}{K-1}(1-p_{i|i}).
\end{equation}
Under this symmetry, the squared loss simplifies to
\begin{equation}
    \Rcal_{\mathrm{MSE}}(t) = \frac{1}{2} \frac{K}{K-1} \sum_{i=1}^K \pi_i (1-p_{i|i})^2.
\end{equation}
Plugging into our usual $p_{i|i}$ dynamics gives
\begin{equation}\label{eq:squared_softmax_ode}
    \begin{split}
        \dot{p}_{i|i} &= p_{i|i}\left[\dot{\theta}_{ii} - \sum_{k=1}^K p_{k|i} \dot{\theta}_{ki} \right] \\
        &= \frac{K}{K-1} p_{i|i} (1-p_{i|i}) \dot{\theta}_{ii} \\
        &= \frac{K}{K-1} \pi_i p_{i|i}^2 (1-p_{i|i})\left[1-2p_{i|i} + \left(p_{i|i}^2 + \frac{1}{K-1}(1-p_{i|i})^2\right)\right] \\
        &= \left(\frac{K}{K-1}\right)^2 \pi_i p_{i|i}^2 (1-p_{i|i})^3.
    \end{split}
\end{equation}
Notice that the right-hand side of the ODE carries an extra factor $\frac{K}{K-1}p_{i|i}(1-p_{i|i})$ compared to the cross-entropy case \eqref{eq:idealized_p_ii_ode}, which suggests learning under squared loss proceeds more slowly. Indeed, we obtain the following sequential learning result (c.f., Lemma \ref{lem:noiseless_sequential_learning}).
\begin{lemma}[Sequential Learning under Squared Loss]\label{lem:square_sequential_learning}
    Under the target \eqref{eq:target}, suppose that $\sigma^2 = 0$. Fix a constant $c \in (1/K,1)$ and define the hitting time $t_i := \inf \{t: p_{i|i}(t) \geq c\}$ under population gradient flow with the squared loss \eqref{eq:square_risk}. Then
    \begin{equation}
        t_i \asymp \frac{K}{\pi_i}.
    \end{equation}
\end{lemma}
This is substantially worse than the $\Theta((\log K)/\pi_i)$ recovery time under cross-entropy training. 

\begin{proof}
We can implicitly solve the ODE \eqref{eq:squared_softmax_ode} via separation of variables and a partial fraction decomposition. Indeed,
\begin{equation}
    \frac{1}{p^2(1-p)^3} = \frac{1}{p^2} + \frac{3}{p} + \frac{3}{1-p} + \frac{2}{(1-p)^2} + \frac{1}{(1-p)^3}.
\end{equation}
Then, omitting the additive constant, the indefinite integral is
\begin{equation}\label{eq:F_p_def}
    F(p) := \int \frac{\dee p}{p^2(1-p)^3} = 3 \log \frac{p}{1-p} - \frac{1}{p} + \frac{2}{1-p} + \frac{1}{2(1-p)^2}.
\end{equation}
Then,
\begin{equation}
    F(p_{i|i}(t)) = \int \left(\frac{K}{K-1}\right)^2 \pi_i \dee t = \left(\frac{K}{K-1}\right)^2 \pi_i t + C.
\end{equation}
Plugging in the initial condition $p = 1/K$ gives
\begin{equation}\label{eq:F_1/K}
    C = F(1/K) = -K - 3 \log (K-1) + \frac{2K}{K-1} + \frac{1}{2}\left(\frac{K}{K-1}\right)^2.
\end{equation}
Therefore, the implicit ODE solution is
\begin{equation}
    F(p_{i|i}(t)) = -K - 3 \log (K-1) + \frac{2K}{K-1} + \frac{1}{2}\left(\frac{K}{K-1}\right)^2 + \left(\frac{K}{K-1}\right)^2 \pi_i t.
\end{equation}
Then, the recovery time for a threshold $c \geq 1/K$ is
\begin{equation}
    t_i(c) = \left(\frac{K-1}{K}\right)^2 \frac{F(c) - F(1/K)}{\pi_i}.
\end{equation}
By \eqref{eq:F_1/K}, we have $F(1/K) = -K + O(\log K)$, while $F(c) = O(1)$. Hence,
\begin{equation}
    t_i(c) \asymp \frac{K}{\pi_i}.
\end{equation}
\end{proof}

As in other sections of the paper, let $t_i$ denote $t_i(1/2)$.

Although optimization is performed with the squared loss $\Rcal_{\mathrm{MSE}}$, we evaluate using the cross-entropy risk $\Rcal_{\mathrm{CE}}$ in order to compare directly against Section \ref{section:noiseless_case}. The same proof idea of decomposing the risk into head and tail class contributions carries over, yielding the following result.

\begin{proposition}[Noiseless Scaling Law for Squared Loss, Softmax Activation]\label{prop:square_scaling_law}
    Under the target \eqref{eq:target}, suppose $\sigma^2 = 0$ and $\alpha > 1$. Then, for the population gradient flow on the squared loss with softmax activation \eqref{eq:square_risk}, there exist constants $c_1, c_2 \in (0,1)$ such that
    \begin{equation}
        \Rcal_{\mathrm{CE}}(t) \asymp
        \begin{cases}
            \log K, & 0 \leq t \leq t_1 \\
            K^{1-1/\alpha} \log K \cdot t^{-(1-1/\alpha)} + t^{-1/2} g_{\alpha}((t/K)^{1/\alpha}), & t_1 \leq t \leq c_1t_K, \\
            t^{-1/2} g_{\alpha}(K), & (1+c_2) t_K \leq t,
        \end{cases}
    \end{equation}
    where
    \begin{equation}
        g_{\alpha}(m) =
        \begin{cases}
            m^{1-\alpha/2}, & 1 < \alpha < 2 \\
            \log (1+m), & \alpha = 2 \\
            1, & \alpha > 2.
        \end{cases}
    \end{equation}
\end{proposition}

\begin{proof}
We work with the same head-tail decomposition as in Appendix \ref{appendix:noiseless_case}. For the head classes, i.e., the indices $i \in [K]$ for which $t > t_i$, let $q_i = 1-p_{i|i}$. We have the dynamics
\begin{equation}
    \dot{q}_i = -\left(\frac{K}{K-1}\right)^2 \pi_i (1-q_i)^2 q_i^3.
\end{equation}
For $t \geq t_i$, we have $1/4 \leq (1-q_i(t))^2 \leq 1$. Hence, $\dot{q}_i \asymp -\pi_i q_i^3$. By separation of variables, and plugging in the initial condition $q_i(t_i) = 1/2$, we have 
\begin{equation}
    q_i(t) \asymp \sqrt{\frac{1}{1+\pi_i(t-t_i)}}.
\end{equation}
Now, for a given time $t \geq 0$, we define a cut-off $k^*(t)$ between head and tail classes:
\begin{equation}
    k^*(t) := \left|\{i \in [K]: t_i \leq t\}\right|.
\end{equation}
Using Lemma \ref{lem:square_sequential_learning} and the assumption $\alpha > 1$, we can solve
\begin{equation}
    k^*(t) \asymp \left(\frac{t}{K}\right)^{1/\alpha}.
\end{equation}
We evaluate the model trained by squared loss with the cross-entropy risk so that we can have a proper comparison with Proposition \ref{prop:noiseless_scaling_law}. We begin with the tail risk
\begin{equation}
    \Rcal_{\mathrm{CE,tail}}(t) = -\sum_{i=k^*(t)+1}^K \pi_i \log p_{i|i}.
\end{equation}
We can upper bound this with the same integral approximation as in Appendix \ref{appendix:idealized_tail_risk}:
\begin{equation}
    \begin{split}
    \Rcal_{\mathrm{CE,tail}}(t) \lesssim \sum_{i=k^*(t)+1}^K i^{-\alpha} \log K \leq (\log K)\int_{k^*(t)}^K x^{-\alpha} \dee x &= \frac{\log K}{1-\alpha} (K^{1-\alpha} - (k^*(t))^{1-\alpha}),
    \end{split}
\end{equation}
which is of order $(k^*(t))^{1-\alpha} \asymp K^{1-1/\alpha} \log K \cdot t^{-(1-1/\alpha)}$ when $t \leq c_1 t_K$ for constant $c_1 \in (0,1)$.

For the lower bound, we show that $t_i(1/\sqrt{K})$ is on the same order as $t_i(1/2)$. Indeed, recalling the definition of $F$ in \eqref{eq:F_p_def}, we have
\begin{equation}
    F(1/\sqrt{K}) = -\sqrt{K} + O(\log K).
\end{equation}
Then,
\begin{equation}
    t_i(1/\sqrt{K}) \asymp \frac{F(1/\sqrt{K}) - F(1/K)}{\pi_i} = \frac{K-\sqrt{K}+O(\log K)}{\pi_i} \asymp \frac{K}{\pi_i}
\end{equation}
and
\begin{equation}
    \frac{t_i(1/\sqrt{K})}{t_i(1/2)} = \frac{K-\sqrt{K} + O(\log K)}{K + O(\log K)} = 1 -\frac{1}{\sqrt{K}} + O\left(\frac{\log K}{K}\right) \asymp 1
\end{equation}
So, defining $k'(t) := |\{i \in [K]: t_i(1/\sqrt{K}) \leq t\}|$, we have $(k'(t) + 1) \asymp (k^*(t)+1)$. Thus,
\begin{equation}
    \Rcal_{\mathrm{CE,tail}}(t) \geq \frac{1}{2}\sum_{i=k'(t)+1}^K \pi_i \log K \asymp K^{1-1/\alpha} \log K \cdot t^{-(1-1/\alpha)},
\end{equation}
using the same integral approximation and assuming $t \leq c_1 t_K$ for some constant $c_1 \in (0,1)$.

This leaves the head risk,
\begin{equation}
    \Rcal_{\mathrm{CE,head}}(t) = -\sum_{i=1}^{k^*(t)} \pi_i \log p_{i|i} \asymp \sum_{i=1}^{k^*(t)} \pi_i \sqrt{\frac{1}{1+\pi_i(t-t_i)}}
\end{equation}
Recall that when $p_{i|i} \geq 1/2$, we have $1-p_{i|i}(t) \leq -\log p_{i|i}(t) \leq 2(1-p_{i|i}(t))$. Hence, we lower-bound the head error by
\begin{equation}
    \Rcal_{\mathrm{CE,head}}(t) \gtrsim \sum_{i=1}^{k^*(t)} \pi_i \sqrt{\frac{1}{1+\pi_i(t-t_i)}} \asymp \sum_{i=1}^{k^*(t)} \sqrt{\pi_i} \sqrt{\frac{1}{t-t_i + \pi_i^{-1}}}.
\end{equation}
Using the relation $\pi_i^{-1} \asymp t_i / K$, we have $t - t_i + \pi_i^{-1} \leq t + \frac{Ct_i}{K} \lesssim t$ for some constant $C > 0$, and therefore
\begin{equation}
    \Rcal_{\mathrm{CE,head}} \gtrsim t^{-1/2} \sum_{i=1}^{k^*(t)} i^{-\alpha/2}
\end{equation}
We split the sum into three cases, which each follow from an integral approximation:
\begin{equation}
    \sum_{i=1}^{k^*(t)} i^{-\alpha/2} \asymp
    \begin{cases}
        (k^*(t))^{1-\alpha/2}, & 1 < \alpha < 2 \\
        \log (1+k^*(t)), & \alpha = 2 \\
        1, & \alpha > 2.
    \end{cases}
\end{equation}
For the upper bound, we split the head risk into
\begin{equation}
    \Rcal_{\mathrm{CE,head}}(t) = \sum_{i=1}^{\floor{k^*(t)/2}} -\pi_i \log p_{i|i}(t) + \sum_{i=\floor{k^*(t)/2} + 1}^{k^*(t)} -\pi_i \log p_{i|i}(t).
\end{equation}
For $i \leq k^*(t)/2$, we have 
\begin{equation}
    \frac{t_i}{t} \leq \frac{t_i}{t_{k^*(t)}} = \left(\frac{i}{k^*(t)}\right)^{\alpha} \leq 2^{-\alpha}. 
\end{equation}
Therefore, $t-t_i \asymp t$. As a result, we can proceed as in the lower bound and write
\begin{equation}
    \sum_{i=1}^{\floor{k^*(t)/2}} -\pi_i \log p_{i|i}(t) \lesssim t^{-1/2} \sum_{i=1}^{k^*(t)} i^{-\alpha/2}.
\end{equation}
For $k^*(t)/2 < i \leq k^*(t)$, we have, since $-\log p_{i|i}(t) \leq \log 2$,
\begin{equation}
    \sum_{i=\floor{k^*(t)/2}+1}^{k^*(t)} -\pi_i \log p_{i|i}(t) \leq (\log 2) \sum_{i=\floor{k^*(t)/2}+1}^{k^*(t)} \pi_i \lesssim \sum_{i=\floor{k^*(t)/2}+1}^K i^{-\alpha} \lesssim (k^*(t))^{1-\alpha},
\end{equation}
which is dominated by $\Rcal_{\mathrm{CE,tail}}(t)$ up to $c_1t_K$. Therefore,
\begin{equation}
    \Rcal_{\mathrm{CE,head}}(t) \lesssim t^{-1/2}\sum_{i=1}^{k^*(t)} i^{-\alpha/2} + (k^*(t))^{1-\alpha}.
\end{equation}
In the early phase $0 \leq t \leq t_1$, the tail lower bound gives
\begin{equation}
    \Rcal_{\mathrm{CE}}(t) \gtrsim (\log K) \sum_{i=k'(t)+1}^K \pi_i \asymp \log K
\end{equation}
since $(k'(t) + 1) \asymp (k^*(t)+1)$. The matching upper bound follows from $p_{i|i}(t) \geq 1/K$. 

The convergence phase rate follows from the fact that $t-t_i \asymp t$ when $t \geq c_2 t_K$ for constant $c_2 > 1$. Hence,
\begin{equation}
    \Rcal_{\mathrm{CE}}(t) = \Rcal_{\mathrm{CE,head}}(t) \asymp \sum_{i=1}^K \sqrt{\frac{\pi_i}{t}}.
\end{equation}
\end{proof}




\begin{remark}
    This scaling law from training with squared loss is worse than the one from training with cross-entropy loss \eqref{eq:idealized_risk_phases}. Ignoring log factors, both the early and intermediate phases last longer by a factor of $K$. This is because the key ODE \eqref{eq:squared_softmax_ode} depends on $p_{i|i}^2$ under squared loss, as opposed to the linear dependence in \eqref{eq:idealized_p_ii_ode}. As for the convergence phase, i.e., when $t \gtrsim K^{1+\alpha}$, the $\Theta(t^{-1/2})$ rate under squared loss (when $\alpha > 2$) is worse than the $\Theta(K/t)$ rate in Proposition \ref{prop:noiseless_scaling_law}. This worse rate appears due to the cubic dependence on $(1-p_{i|i})$ in \eqref{eq:squared_softmax_ode}, compared to the quadratic dependence in \eqref{eq:idealized_p_ii_ode}.
\end{remark}

\section{Experiments}\label{appendix:experiments}
As a simple experiment, we devise an imbalanced classification problem using the ImageNet dataset \cite{russakovsky2015imagenet}. We choose image classification over language modelling because, in the latter case, contexts can have many valid completions (thus giving substantial overlaps between classes), whereas each ImageNet example has only one correct answer.

We perform linear classification (i.e., logistic regression) over the $K = 1000$ ImageNet classes, given frozen pre-trained image representations from a ResNet-18 model \cite{he2016deep} implemented by PyTorch \cite{paszke2019pytorch}. These representations have dimension $d = 512$. We use all 1,281,167 training and 50,000 validation images. In order to test different values of the prior power law exponent $\alpha$, we randomly order the $K$ classes and assign the prior $\pi_i = i^{-\alpha}/Z_{K,\alpha}$ to them. We resample the training images according to this prior to create our desired imbalanced dataset. Validation losses are computed classwise and then weighted by the prior. For the classification head, weights are initialized to zero and biases are turned off. We train using SGD with a batch size of 256, a learning rate of 1 (which we did not tune), and 500,000 weight updates.

We measure the persistent recovery time of a class $i$ as the last time $\widehat{P}_{i|i} < c$ for a constant threshold $c \in (0,1)$, where $\widehat{P}_{i|i}$ is the average probability assigned to class $i$ over all test examples with true label $i$. We use several values of $c$: $\{0.05,0.1,0.25,0.5\}$. In our experiments, the persistent recovery time is consistently close to the first recovery time $t_i$ defined in our theory.

We report the following findings: for all values of $\alpha$ tested $\{1.25, 1.5, 1.75, 2\}$, classes with higher frequency are consistently recovered earlier (among those recovered during the training horizon) and their recovery times closely follow the theoretical prediction $(\log K)/\pi_i$. This is robust to across all choices of $c$ --- though less classes are recovered within the horizon as $c$ increases. We quantify this with a power law fit of the form $t_i = \beta[(\log K)/\pi_i]^{\gamma}$, where $t_i$ is the observed recovery time. Table \ref{tab:imagenet_results} and Figure \ref{fig:imagenet_results} summarize the results for $c = 0.1$.

\begin{table}
    \centering
    \begin{tabular}{lcccc}
        \toprule
        $\alpha$        & 1.25  & 1.50  & 1.75  & 2.00  \\
        \midrule
        Fitted $\gamma$ & 1.158 & 1.181 & 1.166 & 1.166 \\
        $R^2$           & 0.943 & 0.955 & 0.969 & 0.975 \\
        \bottomrule
    \end{tabular}
    \vspace{2mm}
    \caption{Fitted power-law exponent $\gamma$ in the relation $t_i = \beta[(\log K)/\pi_i]^{\gamma}$ given $K = 1000$ observations of persistent recovery times $\sup \{t: \widehat{P}_{i|i}(t) < 0.1\}$ for online SGD on ImageNet with an imbalanced class prior $\pi_i = i^{-\alpha}/Z_{K,\alpha}$.}
    \label{tab:imagenet_results}
\end{table}

\begin{figure}[htbp]
    \centering

    \begin{subfigure}[b]{0.48\linewidth}
        \centering
        \includegraphics[width=\linewidth]{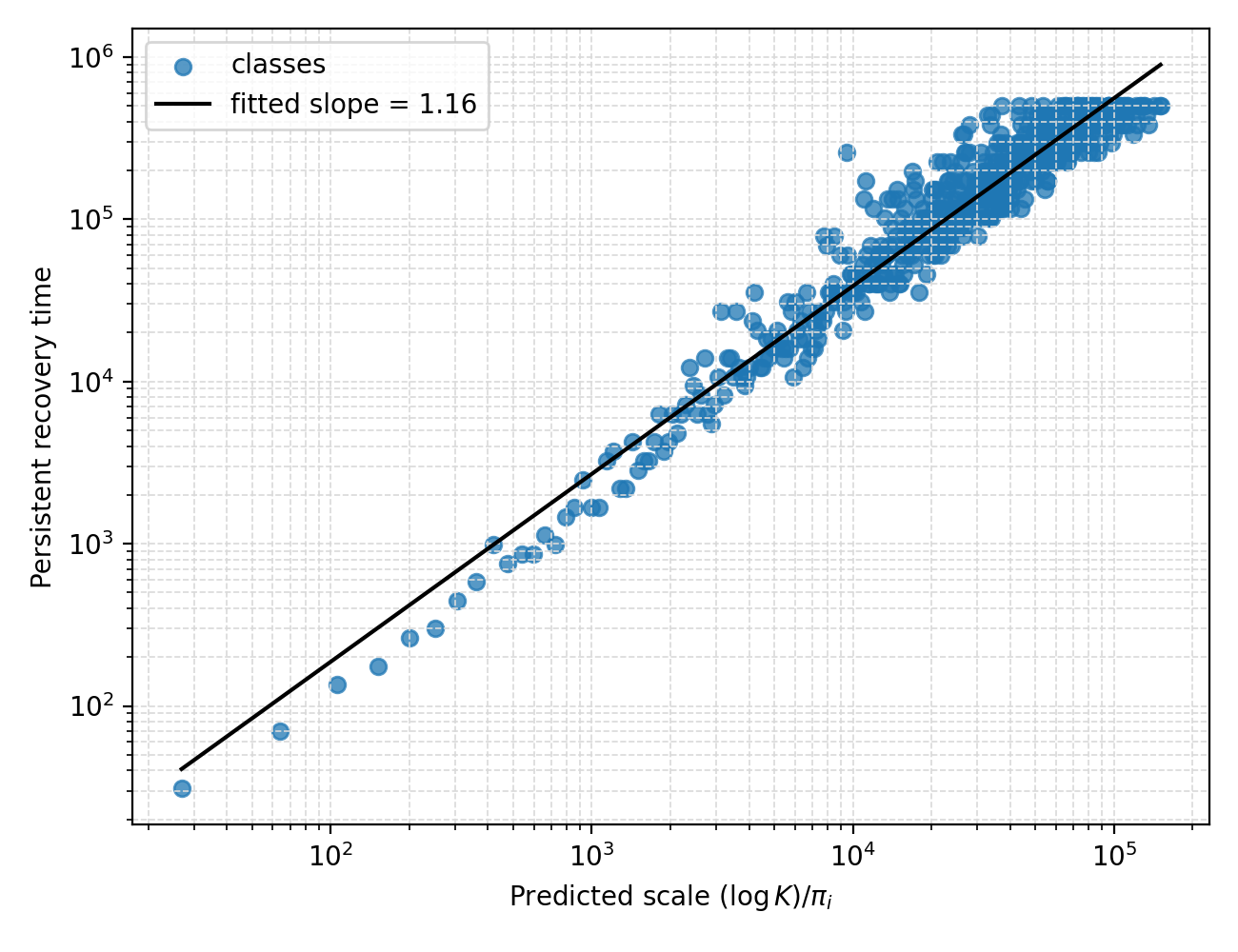}
        \caption{$\alpha = 1.25$}
    \end{subfigure}
    \hfill
    \begin{subfigure}[b]{0.48\linewidth}
        \centering
        \includegraphics[width=\linewidth]{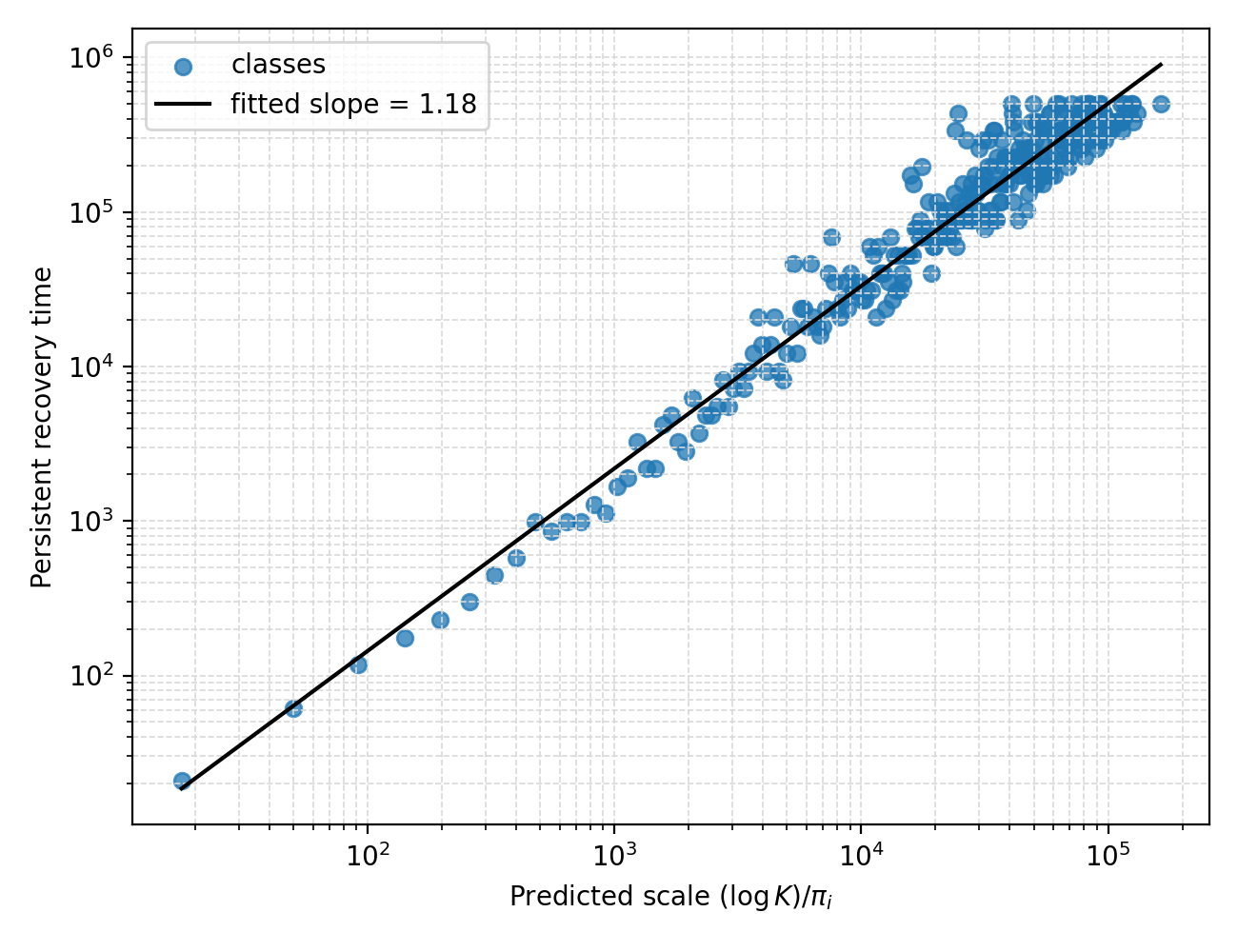}
        \caption{$\alpha = 1.5$}
    \end{subfigure}

    \vspace{0.3cm}

    \begin{subfigure}[b]{0.48\linewidth}
        \centering
        \includegraphics[width=\linewidth]{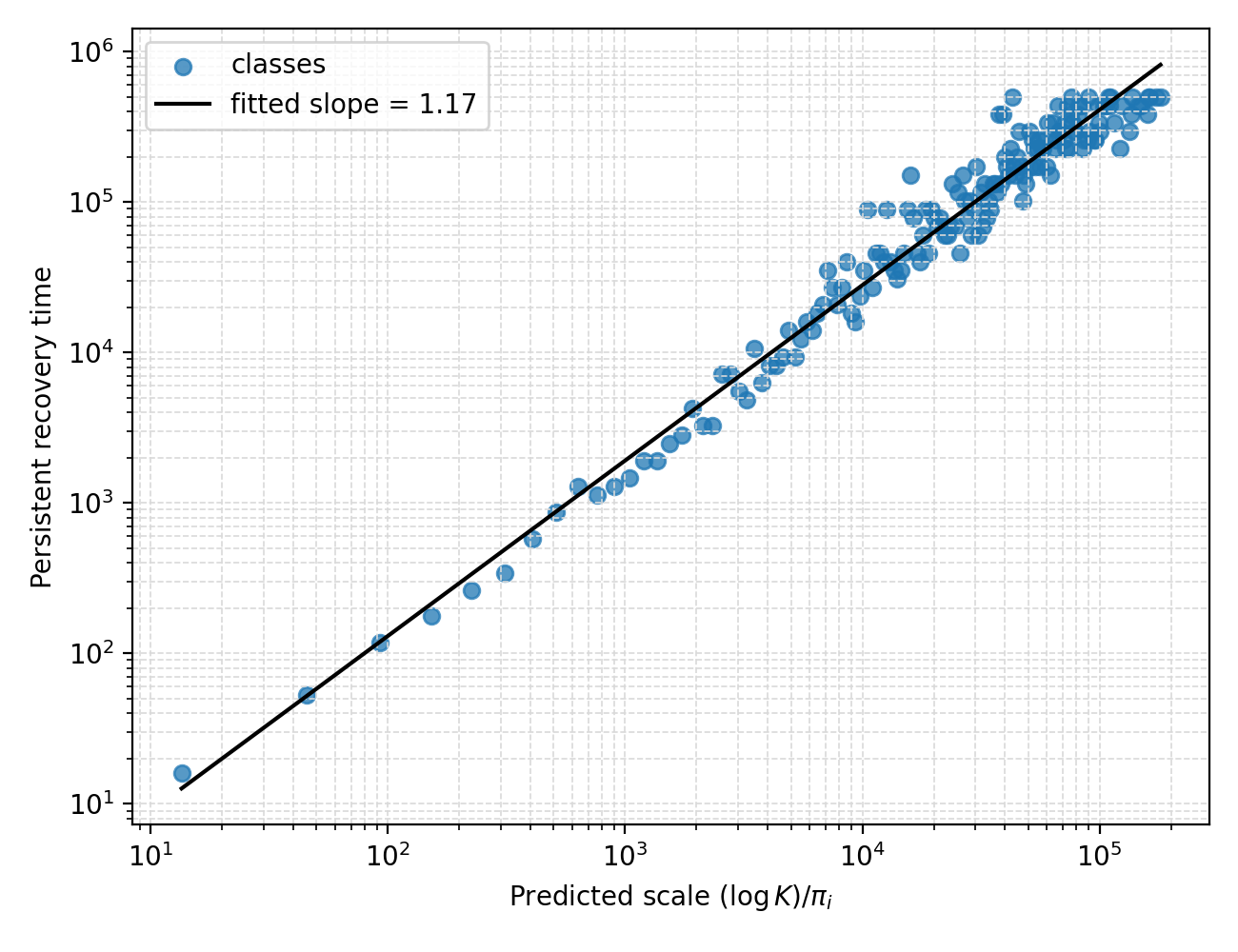}
        \caption{$\alpha = 1.75$}
    \end{subfigure}
    \hfill
    \begin{subfigure}[b]{0.48\linewidth}
        \centering
        \includegraphics[width=\linewidth]{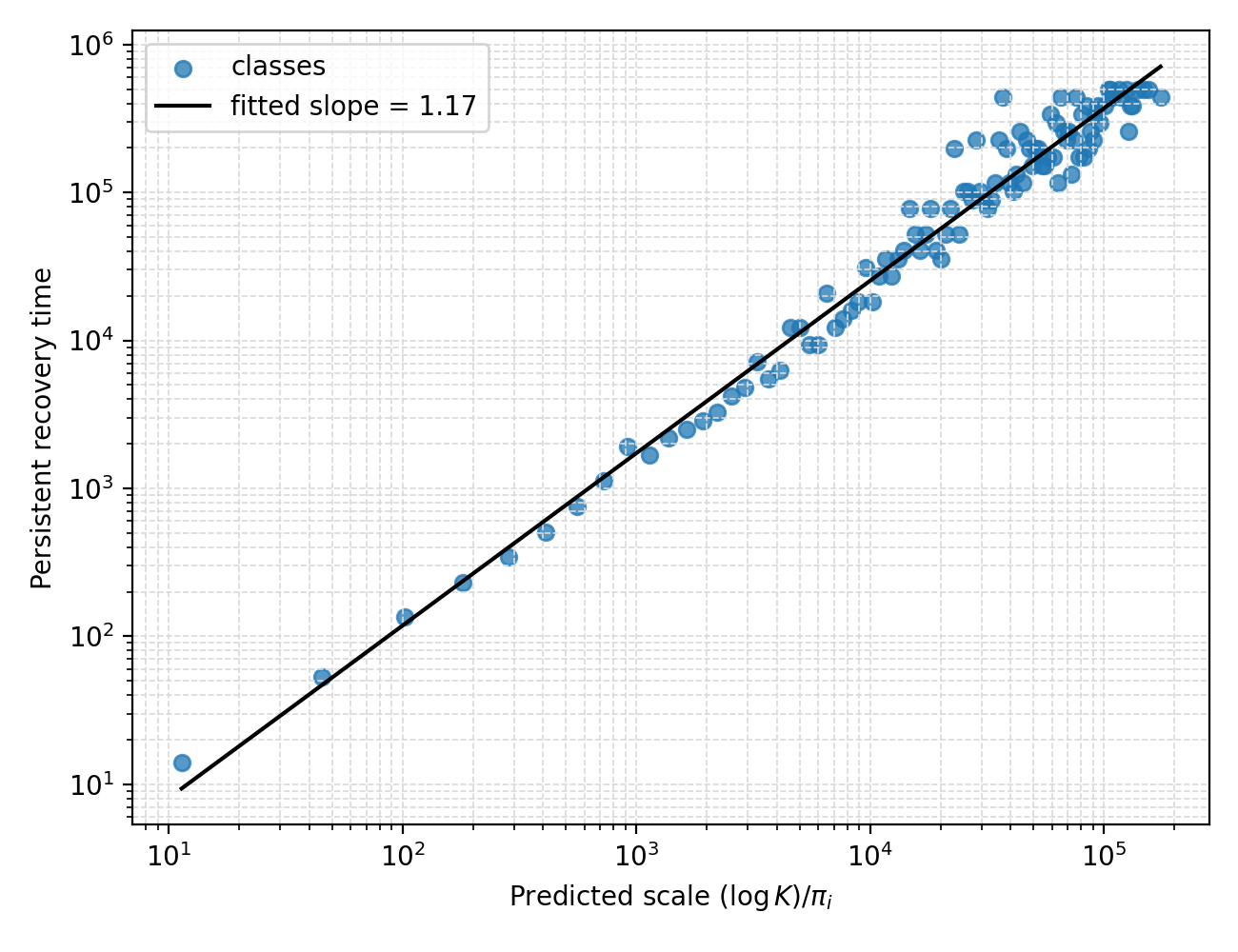}
        \caption{$\alpha = 2$}
    \end{subfigure}

    \caption{Plot of class persistent recovery times, i.e., $\sup \{t: \widehat{P}_{i|i}(t) < 0.1\}$ compared to the theoretical prediction $t_i \asymp (\log K)/\pi_i$ for imbalanced classification on ImageNet with power-law prior $\pi_i = i^{-\alpha}/Z_{K,\alpha}$ for different values of $\alpha$.}
    \label{fig:imagenet_results}
\end{figure}

\clearpage

\section*{NeurIPS Paper Checklist}

\begin{enumerate}

\item {\bf Claims}
    \item[] Question: Do the main claims made in the abstract and introduction accurately reflect the paper's contributions and scope?
    \item[] Answer: \answerYes{} 
    \item[] Justification: Every main claim made in the abstract and introduction is backed up by a rigorous proposition or theorem in the main body of the paper.

\item {\bf Limitations}
    \item[] Question: Does the paper discuss the limitations of the work performed by the authors?
    \item[] Answer: \answerYes{} 
    \item[] Justification: Limitations are addressed in the paper's conclusion and in discussions of theorem assumptions (e.g., the time horizon for the $\sigma^2 > 0$ case).

\item {\bf Theory assumptions and proofs}
    \item[] Question: For each theoretical result, does the paper provide the full set of assumptions and a complete (and correct) proof?
    \item[] Answer: \answerYes{} 
    \item[] Justification: All theoretical results in the main body and appendix have their assumptions explicitly stated. All proofs can be found in the appendices.

    \item {\bf Experimental result reproducibility}
    \item[] Question: Does the paper fully disclose all the information needed to reproduce the main experimental results of the paper to the extent that it affects the main claims and/or conclusions of the paper (regardless of whether the code and data are provided or not)?
    \item[] Answer: \answerYes{} 
    \item[] Justification: See captions under figures.

\item {\bf Open access to data and code}
    \item[] Question: Does the paper provide open access to the data and code, with sufficient instructions to faithfully reproduce the main experimental results, as described in supplemental material?
    \item[] Answer: \answerNo{} 
    \item[] Justification: Experiments are designed to illustrate the theory rather than be results in their own right.

\item {\bf Experimental setting/details}
    \item[] Question: Does the paper specify all the training and test details (e.g., data splits, hyperparameters, how they were chosen, type of optimizer) necessary to understand the results?
    \item[] Answer: \answerYes{} 
    \item[] Justification: See captions of figures detailing the simulation settings.

\item {\bf Experiment statistical significance}
    \item[] Question: Does the paper report error bars suitably and correctly defined or other appropriate information about the statistical significance of the experiments?
    \item[] Answer: \answerNo{} 
    \item[] Justification: Experiments are designed to illustrate the theory rather than be results in their own right.

\item {\bf Experiments compute resources}
    \item[] Question: For each experiment, does the paper provide sufficient information on the computer resources (type of compute workers, memory, time of execution) needed to reproduce the experiments?
    \item[] Answer: \answerNo{} 
    \item[] Justification: Experiments are designed to illustrate the theory rather than be results in their own right.
    
\item {\bf Code of ethics}
    \item[] Question: Does the research conducted in the paper conform, in every respect, with the NeurIPS Code of Ethics \url{https://neurips.cc/public/EthicsGuidelines}?
    \item[] Answer: \answerYes{} 
    \item[] Justification: We have reviewed the Code of Ethics. This being a theory paper, we do not anticipate any harms originating from this research.

\item {\bf Broader impacts}
    \item[] Question: Does the paper discuss both potential positive societal impacts and negative societal impacts of the work performed?
    \item[] Answer: \answerNA{} 
    \item[] Justification: This being a theory paper, we do not anticipate any negative societal impact originating from this research.
    
\item {\bf Safeguards}
    \item[] Question: Does the paper describe safeguards that have been put in place for responsible release of data or models that have a high risk for misuse (e.g., pre-trained language models, image generators, or scraped datasets)?
    \item[] Answer: \answerNA{} 
    \item[] Justification: This being a theory paper, we do not anticipate any negative societal impact originating from this research.

\item {\bf Licenses for existing assets}
    \item[] Question: Are the creators or original owners of assets (e.g., code, data, models), used in the paper, properly credited and are the license and terms of use explicitly mentioned and properly respected?
    \item[] Answer: \answerYes{} 
    \item[] Justification: We properly cite ImageNet and PyTorch in Appendix F for our experiment.

\item {\bf New assets}
    \item[] Question: Are new assets introduced in the paper well documented and is the documentation provided alongside the assets?
    \item[] Answer: \answerNA{} 
    \item[] Justification: We do not introduce new assets.

\item {\bf Crowdsourcing and research with human subjects}
    \item[] Question: For crowdsourcing experiments and research with human subjects, does the paper include the full text of instructions given to participants and screenshots, if applicable, as well as details about compensation (if any)? 
    \item[] Answer: \answerNA{} 
    \item[] Justification: Our work contains no such experiments.

\item {\bf Institutional review board (IRB) approvals or equivalent for research with human subjects}
    \item[] Question: Does the paper describe potential risks incurred by study participants, whether such risks were disclosed to the subjects, and whether Institutional Review Board (IRB) approvals (or an equivalent approval/review based on the requirements of your country or institution) were obtained?
    \item[] Answer: \answerNA{} 
    \item[] Justification: Our paper does not involve crowdsourcing nor research with human subjects.

\item {\bf Declaration of LLM usage}
    \item[] Question: Does the paper describe the usage of LLMs if it is an important, original, or non-standard component of the core methods in this research? Note that if the LLM is used only for writing, editing, or formatting purposes and does \emph{not} impact the core methodology, scientific rigor, or originality of the research, declaration is not required.
    \item[] Answer: \answerNA{} 
    \item[] Justification: LLMs were not a core component of this research.

\end{enumerate}

\end{document}